\documentclass{article}

  \usepackage[preprint]{neurips_2026}

\usepackage[utf8]{inputenc} 
\usepackage[T1]{fontenc}    
\usepackage{hyperref}       
\usepackage{url}            
\usepackage{booktabs}       
\usepackage{amsfonts}       
\usepackage{nicefrac}       
\usepackage{microtype}      
\usepackage{xcolor}         
\usepackage{comment}
\usepackage{stmaryrd}
\usepackage{fontawesome}
\usepackage{algorithmic}
\usepackage{algorithm2e}
\usepackage{subcaption}
\usepackage{pgfplots}
\usepgfplotslibrary{groupplots}
\usepackage{amsthm}

\usepackage{xcolor}
\newcommand{\revision}[1]{#1}

\usepackage{todonotes}
\newcommand{\todoe}[1]{\todo[color=green!40, inline]{EK: \scriptsize #1}}

\usepackage{amsmath,amsfonts,bm}

\def\eqref#1{equation~\ref{#1}}

\def\1{\bm{1}}

\DeclareMathAlphabet{\mathsfit}{\encodingdefault}{\sfdefault}{m}{sl}
\SetMathAlphabet{\mathsfit}{bold}{\encodingdefault}{\sfdefault}{bx}{n}
\newcommand{\tens}[1]{\bm{\mathsfit{#1}}}

\def\tQ{{\tens{Q}}}

\def\gA{{\mathcal{A}}}

\def\gC{{\mathcal{C}}}

\def\gE{{\mathcal{E}}}

\def\gM{{\mathcal{M}}}

\def\gO{{\mathcal{O}}}
\def\gP{{\mathcal{P}}}
\def\gQ{{\mathcal{Q}}}
\def\gR{{\mathcal{R}}}
\def\gS{{\mathcal{S}}}

\def\gX{{\mathcal{X}}}
\def\gY{{\mathcal{Y}}}

\def\sI{{\mathbb{I}}}

\def\sP{{\mathbb{P}}}
\def\sQ{{\mathbb{Q}}}
\def\sR{{\mathbb{R}}}

\newcommand{\etens}[1]{\mathsfit{#1}}

\def\etP{{\etens{P}}}
\def\etQ{{\etens{Q}}}

\newcommand{\R}{\mathbb{R}}

\newtheorem{theorem}{Theorem}[section]
\newtheorem{lemma}[theorem]{Lemma}

\newtheorem{assumption}[theorem]{Assumption}

\newtheorem{proposition}[theorem]{Proposition}

\def\ival#1{\llbracket #1 \rrbracket}
\def\normone#1{{\left\Vert #1 \right\Vert}_1}
\def\DR#1#2#3{\mathrm{#1}\left(#2\middle\Vert#3\right)}
\def\eqref#1{(\ref{#1})}

\title{Sample Efficient Hierarchical Reinforcement Learning via Best Policy Identification}

\author{%
  Anders Jonsson\\
  Department of Engineering\\
  Universitat Pompeu Fabra\\
  \texttt{anders.jonsson@upf.edu} \\
  \And
  Emilie Kaufmann\\
  Univ. Lille, CNRS, Inria\\
  Centrale Lille, UMR 9189-CRIStAL\\
  \texttt{emilie.kaufmann@univ-lille.fr} \\
  \And
  Gianmarco Tedeschi\\
  Dept.~Electronics, Information, and Bioengineering\\
  Politecnico di Milano\\
  \texttt{gianmarco.tedeschi@polimi.it} \\
  \And
  Lorenzo Steccanella\\
  Department of Engineering\\
  Universitat Pompeu Fabra\\
  \texttt{lorenzo.steccanella.w@gmail.com} \\
}

\begin{document}

\maketitle

\begin{abstract}
We present HBPI-UCRL, a model-based algorithm for hierarchical reinforcement learning (HRL) that learns high-level and low-level policies in parallel.
HBPI-UCRL exploits the fact that a high-level transition corresponds to a multi-step transition at the low level.
We introduce two conditions on the low-level dynamics that are sufficient to make parallel HRL learnable.
When these conditions hold, we prove that HBPI-UCRL has a polynomial sample complexity in the problem parameters.
In the sparse-reward, goal-directed setting, our sample complexity upper bound for HBPI-UCRL is strictly lower than that of its non-hierarchical counterpart, providing theoretical justification for the empirical success of HRL.
\end{abstract}

\allowdisplaybreaks[4]
\setlength{\abovedisplayskip}{4pt}
\setlength{\belowdisplayskip}{4pt}
\setlength{\textfloatsep}{10pt}
\setlength{\floatsep}{10pt}

\section{Introduction}

In hierarchical reinforcement learning (HRL), a branch of reinforcement learning (RL), the learning problem is decomposed into a hierarchy of subproblems, each a Markov decision process (MDP).
At the top of the hierarchy is a semi-Markov decision process (SMDP) that selects between subproblems.
When selected, a subproblem executes a local policy for multiple time steps before returning control to the SMDP.
The aim is to approximate the optimal hierarchical policy, which involves maximizing a local value function of each subproblem as well as the value function of the SMDP.

Many authors have demonstrated empirically that HRL carries several practical benefits, such as efficient exploration~\citep{bellemare2020,konidaris2009skills,machado2022}, improved sample efficiency~\citep{levy2019hrl,nachum2018hrl,rafati2019hrl}, better credit assignment~\citep{vezhnevets2017hrl}, faster long-term planning~\citep{gopalan2017amdp} and generalization and transfer~\citep{ahn2022crl,matthews2022transfer}.
However, in spite of the long history and large body of work on HRL, there are surprisingly few works that prove theoretical guarantees for HRL algorithms.
This is especially true when the SMDP and subproblem policies are learned in parallel.
Hence an important research question is to establish theoretical foundations for HRL that explain its benefits.

We present a novel model-based algorithm for episodic HRL called HBPI-UCRL, for hierarchical best policy identification.
The algorithm learns the SMDP policy and subproblem policies in parallel by collecting data in the form of episodes.
We assume that the subproblem MDPs share the transition dynamics but differ in the local reward functions.
HBPI-UCRL performs best policy identification by computing empirical composite distributions and maintaining high-probability confidence bounds on the SMDP dynamics.
Since the SMDP dynamics depend on the subproblem policies, the SMDP learning problem is non-stationary when the subproblem policies are learned in parallel.
A key contribution of our work is to identify two conditions that relate the quality of the subproblem solutions to the accuracy of the SMDP dynamics.
These conditions are sufficient for HRL to be PAC-learnable, and can also be of practical importance. 

When our two conditions hold, we prove that HBPI-UCRL is $(\varepsilon,\delta)$-PAC with respect to the optimal hierarchical value function.
Its sample complexity is of order $\widetilde\gO(SA\bm H^4H^6/\varepsilon^2)$\footnote{The notation $\widetilde\gO$ hides logarithmic factors in $S$, $A$, $\bm H$, $H$, $1/\varepsilon$ and $1/\delta$, and states the result for the regime of small $\delta$.}, where $SA$ is the size of the subproblem state-action space, $H$ is the subproblem horizon and $\bm H$ is the SMDP horizon.
Notably, the sample complexity is {\em independent} of the number of high-level states $\bm S$ and the number of subproblems.
In sparse-reward, goal-directed HRL, a version of HBPI-UCRL achieves a sample complexity bound of order $\widetilde\gO(SA\bm H^2H^2/\varepsilon^2)$,
which is a factor $\bm S$ smaller than that of the non-hierarchical algorithm BPI-UCRL~\citep{kaufmann2021rf}, and only a factor $\bm H^2$ larger than that of BPI-UCRL when applied to a single sparse-reward subproblem.

To the best of our knowledge, ours is the first explicit sample complexity result for the parallel HRL setting.
The only comparable work for parallel HRL is that of~\citet{drappo2025hrl}, who analyze the regret specifically for the goal-oriented setting.
It is possible to use a regret-to-PAC conversion to obtain a sample complexity guarantee for their algorithm, but unlike HBPI-UCRL, their algorithm does not have a data-dependent stopping rule and cannot solve a more general class of HRL problems.
We further elaborate on the differences in Section~\ref{sec:related}.

The paper is organized as follows.
Section~\ref{sec:rl} introduces two families of composite distributions and describes best policy identification (BPI) for episodic RL.
Section~\ref{sec:hrl} describes our formalism for episodic HRL and introduces our two key conditions for learnability. HBPI-UCRL is described in Section~\ref{sec:algorithm}, in which we analyze its sample complexity, including the special case of sparse-reward, goal-directed HRL. Our sample complexity bounds are put in context of existing work in Section~\ref{sec:related}. In Section~\ref{sec:exp} we present some numerical experiments, and we conclude with a discussion.

\paragraph{Notation} Given a finite set $\gX$, $\Delta(\gX)$ denotes the probability simplex on $\gX$.
Given an integer $n>0$, $\ival{n}$ denotes the set $\{1,\ldots,n\}$.
We use $\sI(\gE)$ to denote the indicator function for an event $\gE$, $a\wedge b$ to denote the minimum of $a$ and $b$, and $a\vee b$ to denote the maximum. 
Given integers $i$ and $j$ such that $1\leq i\leq j$, let $\gamma_{i:j}$ denote a partial sequence $\gamma_i,\ldots,\gamma_j$ of elements from a given alphabet $\Gamma$.

\section{Reinforcement Learning}\label{sec:rl}


An episodic Markov decision process (MDP) is a tuple $\gM=\langle \gS, \gA, \gP, \gR, H \rangle$, where $\gS$ is the finite state space with cardinality $S=|\gS|$, $\gA$ is the finite action space with cardinality $A=|\gA|$, $\gP:\gS\times\gA\to\Delta(\gS)$\footnote{Several authors consider time-inhomogeneous transition kernels $\mathcal{P}_h$, but this is uncommon in the HRL literature.} is a transition kernel, $\gR:\gS\times\gA\to[0,1]$ is a reward function, and $H>1$ is an integer horizon.
An episode $s_1,a_1,\cdots,s_H,a_H,s_{H+1}$ is {\em valid} if $\gP(s_{h+1}|s_h,a_h)>0$ for each $h\in\ival{H}$.
We say that $\gM$ is {\em sparse-reward} if $\sum_{h=1}^H\gR(s_h,a_h)\leq 1$ for each valid episode.

A deterministic policy $\pi=\{\pi_h\}_{h\in\ival{H}}$ defines a mapping $\pi_h:\gS\to\gA$ for each $h\in\ival{H}$.
The value function $V^\pi:\ival{H+1}\times\gS\to\R$ of $\pi$ measures the expected reward sum under $\pi$ when starting from state $s$ at step $h$.
The value function is recursively defined for each $s\in\gS$ as $V_{H+1}^\pi(s;\gR)=0$ and
\[
V_h^\pi(s;\gR) = \gR(s,\pi_h(s)) + \sum_{s'} \gP(s'|s,\pi_h(s)) V_{h+1}^\pi(s';\gR) \quad \forall h\in\ival{H},
\]
where the notation emphasizes the dependence on $\gR$.
We assume known $\gR$ but unknown $\gP$, and below we assume a unique initial state $s_1$, though all results generalize to fixed initial state distributions.
We denote by $\pi^*$ an optimal policy, whose value function satisfies $V^*_1(s_1 ;\gR) = \max_{\pi} V^{\pi}_1(s_1 ; \gR)$. 

In online RL, in each episode $t$ an agent chooses a policy $\pi^{t}$ and collects a state sequence $s_{1:H+1}^t$ by taking in each step $h$ the action $a_h^t = \pi_{h}^t(s_{h}^{t})$ and transitioning to the next state $s_{h+1}^{t} \sim \mathcal{P}(\cdot |s_h^{t},a_h^{t})$. The policy $\pi^{t}$ can be selected based on data from previous episodes. 

Our paper focuses on a particular PAC RL problem, called best policy identification (BPI). In this setting, after each episode the agent can decide to stop data collection and output a guess $\widehat{\pi}$ for the optimal policy.  The resulting algorithm is called $(\varepsilon,\delta)$-PAC if $|V^*_1(s_1;\gR) - V_1^{\widehat{\pi}}(s_1 ; \mathcal{R})| \leq \varepsilon$ with probability larger than $1-\delta$. The \emph{sample complexity} of the algorithm is the number of episodes $\tau$ it requires before stopping and making an $(\varepsilon,\delta)$-PAC guess.

$(\varepsilon,\delta)$-PAC algorithms with a minimax optimal sample complexity are based on sophisticated Bernstein bonuses \citep[e.g.][]{dann2019policy}.
In this work we instead focus our attention on a simpler algorithm called BPI-UCRL \citep{kaufmann2021rf}, that couples an optimistic algorithm based on Hoeffding bonuses with an appropriate stopping rule.
In this section, we present an alternative stopping rule for BPI-UCRL rooted in new bounds on the estimation error of composite probability distributions.
As we shall see, this variant of BPI-UCRL can be naturally extended to an algorithm for HRL.

\textbf{Estimating transitions.} We now introduce notation for model-based episodic RL algorithms. After $t$ episodes, we compute counts $N^t$ and an empirical estimate $\widehat\gP^t$ of the transition kernel as follows:
\begin{align*}
N^t(s,a,s') &= \sum_{\ell=1}^t \sum_{h=1}^H \sI((s,a,s')=(s_h^\ell,\pi_h^\ell(s_h^\ell),s_{h+1}^\ell)), \quad \widehat\gP^t(s'|s,a) = \frac {N^t(s,a,s')} {\sum_{s'} N^t(s,a,s')},
\end{align*}
with $\widehat\gP^t(s'|s,a) \equiv 1 / S$ if $N^t(s,a)\equiv\sum_{s'} N^t(s,a,s')=0$.
We remark that the episode $s_1^t,a_1^t,\cdots, s_H^t,a_H^t,s_{H+1}^t$ is fully determined by $s_{1:H+1}^t$ and $\pi^t$.
We use $\gP_h^\pi(s'|s)=\gP(s'|s,\pi_h(s))$ and $\widehat\gP_h^{t,\pi}(s'|s)=\widehat\gP^t(s'|s,\pi_h(s))$ to denote the true and estimated transition kernels under the policy $\pi$. 
Using time-uniform concentration, we establish in Appendix~\ref{app:concentration} that the event 
\begin{align}
\gE=\big\{ &\forall t\in\mathbb{N}, \forall (s,a)\in\gS\times\gA, \big\lVert \widehat\gP^t(\cdot|s,a) - \gP(\cdot|s,a) \big\rVert_1 \leq B^t(s,a) \big\}.\label{eq:event}
\end{align}
satisfies $\mathbb{P}(\gE)\geq 1-\delta/2$, where $B^{t}(s,a) = \sqrt{2\beta(N^{t}(s,a),\delta)/N^{t}(s,a)} \wedge 2$ and $\beta$ is defined as $\beta(n,\delta) = \log(2SA/\delta) + (S-1)\log(e(1+n/(S-1)))$.


\textbf{Composite distributions.}
We consider two families of composite probability distributions induced by the transition kernels $\{\gP_h^\pi\}_{h\in\ival{H}}$ of $\pi$.
Given integers $i$ and $j$ such that $1\leq i<j\leq H+1$, let $\etP_{i:j}^\pi(\cdot|s_i)$ be the probability of sequences $s_{i+1:j}$ conditional on a state $s_i$, recursively defined as
\begin{equation*}
\etP_{i:j}^\pi(s_{i+1:j}|s_i) = \left\{
\hspace*{-5pt}
\begin{array}{l}
\gP_h^\pi(s_{i+1}|s_i), \hspace*{2.93cm} \text{if } i+1=j,\\
\gP_h^\pi(s_{i+1}|s_i)\, \etP_{i+1:j}^\pi(s_{i+2:j}|s_{i+1}), \; \text{else}.
\end{array}
\right.
\end{equation*}
Given $i$ and $j$ such that $1\leq i\leq j\leq H+1$, let $\sP_{i:j}^\pi(\cdot|s_i)$ be the occupancy of states $s_j$ at step $j$ conditional on a state $s_i$, recursively defined as
\begin{equation*}
\sP_{i:j}^\pi(s_j|s_i) = \left\{
\hspace*{-3pt}
\begin{array}{l}
\sI(s_i=s_j), \hspace*{3.66cm} \text{if } i=j,\\
\sum_{s_{i+1}} \gP_h^\pi(s_{i+1}|s_i)\, \sP_{i+1:j}^\pi(s_j|s_{i+1}), \; \text{else}.
\end{array}
\right.
\end{equation*}
Replacing $\gP_h^\pi$ in the above definitions with another transition kernel for $\pi$ induces different composite distributions, e.g.~the empirical estimates $\widehat\gP_h^{t,\pi}$ have associated composite distributions $\widehat\etP{}_{i:j}^{t,\pi}$ and $\widehat\sP{}_{i:j}^{t,\pi}$.

\textbf{BPI-UCRL Revisited.}
We present and analyze a variant of BPI-UCRL~\citep{kaufmann2021rf} for best policy identification.
In each episode $t$, BPI-UCRL computes an upper confidence bound on the optimal value function $V^*$.
For each state-action $(s,a)$, the algorithm first defines a confidence set
\begin{equation}\label{eq:confbounds}
\gC^t(s,a) = \big\{ p\in\Delta(\gS) \; \big\vert \; \big\lVert p - \widehat\gP^t(\cdot|s,a) \big\rVert_1 \leq B^t(s,a) \big\}.
\end{equation}
An upper confidence bound $\overline V{}^t$ on $V^*$ can now be recursively defined for $s$ as $\overline V{}_{\hspace*{-2pt}H+1}^t(s;\gR)=0$ and
\begin{align}
\overline V{}_{\hspace*{-2pt}h}^t(s;\gR) \hspace*{-1pt} &= \hspace*{-1pt} \max_a \Big[ \gR(s,a) \hspace*{-1pt} + \hspace*{-9pt} \max_{p\in\gC^t(s,a)} \sum_{s'} p(s') \overline V{}_{\hspace*{-2pt}h+1}^t(s';\gR) \Big] \quad \forall h\in\ival{H}.\label{eq:upperV}
\end{align}
The policy $\pi^{t+1}$ in episode $t+1$ is defined as the greedy policy for $\overline V{}^t$.

Our analysis relies on an error function $\widehat L{}^t$ recursively defined for each $s\in\gS$ as $\widehat L{}_{H+1}^t(s)=0$ and
\begin{align}
\widehat L{}_{h}^t(s) \hspace*{-1pt} &= \hspace*{-1pt} \min \Big[ 2, B^t(s,\pi_h^{t+1}(s)) + \sum_{s'} \widehat\gP{}_h^{t,\pi^{t+1}} \hspace*{-3pt}(s'|s) \widehat L{}_{h+1}^t(s') \Big] \quad \forall h\in\ival{H}.\label{eq:L}
\end{align}
Intuitively, $\widehat L{}^t$ compounds the approximation error of kernels in the confidence set $\gC^t$ across an entire episode. In Appendix~\ref{app:rf} we show that this is enough to upper bound the difference between both types of composite probability distributions.
We modify BPI-UCRL by redefining the stopping criterion as $\tau = \inf\{ t \in \mathbb{N} : \widehat L{}_1^t(s_1)\leq \varepsilon/(2H)\}$.
Upon stopping, BPI-UCRL outputs the policy $\widehat{\pi}=\pi^{\tau+1}$.

In Appendix~\ref{app:rf} we prove the following theorem that upper bounds the sample complexity of our variant of BPI-UCRL.
The upper bound is a factor $H$ larger than that reported by \citet{kaufmann2021rf} for stationary transition kernels in their Appendix E. However, we believe that their proof is incorrect, and that the sample complexity in the stationary setting is of the same order as ours. 

\begin{theorem}\label{thm:rf}
With probability larger than $1-\delta$, BPI-UCRL with the modified stopping rule returns a policy $\widehat\pi$ that satisfies $V_1^*(s_1;\gR)-V_1^{\widehat\pi}(s_1;\gR)\leq\varepsilon$ using a number of episodes
\begin{align*}
\tau = \gO \Big( \frac {SAH^4} {\varepsilon^2} \Big( &\log \frac {SAH} \delta + S \log \Big( \frac {SAH^5} {\varepsilon^2} \log \frac {SAH} \delta \Big) \Big) \Big).
\end{align*}
\end{theorem}

Our claims throughout focus on the regime of small $\delta$~\citep{kaufmann2021rf} in which the first term dominates, \revision{but any improvement in the sample complexity also carries over to the second term}.
In Appendix~\ref{app:rf} we show that if $\gM$ is sparse-reward, we can set the stopping criterion of BPI-UCRL to $\widehat L{}_1^t(s_1)\leq \varepsilon/2$ and achieve a sample complexity of order $\widetilde\gO( SAH^2/\varepsilon^2)$ in the regime of small $\delta$.


\section{Hierarchical Reinforcement Learning}\label{sec:hrl}

In episodic HRL, the aim is to identify a good hierarchical policy (defined below) in an episodic MDP $\gM=\langle\Sigma,\gA,\gX,\gY,N\rangle$, usually called the {\em flat} (i.e.~non-hierarchical) MDP in the HRL literature.

We assume that $\gM$ admits a hierarchical structure.
Concretely, the state space $\Sigma$ is characterized by a high-level state space $\bm\gS$ with cardinality $\bm S=|\bm\gS|$ and a subproblem state space $\gS$, and three known functions $f:\bm\gS\times\gS\to\Sigma$, $\bm g:\Sigma\to\bm\gS$ and $g:\Sigma\to\gS$ map pairs $(\bm s,s)\in\bm\gS\times\gS$ to states $\sigma\in\Sigma$ and vice versa.
In addition, there are $K$ subproblems in the form of episodic MDPs $\gM^k=\langle \gS, \gA, \gP, \gR^k, H \rangle$, $k\in\ival{K}$, that share the transition kernel $\gP$ but differ in the local reward function $\gR^k$.
The transition kernel $\gX$ of $\gM$ can be expressed as follows for each transition $(\sigma,a,\sigma')$:
\[
\gX(\sigma'|\sigma,a) = \sum_{s'} \gP(s' | g(\sigma),a) \cdot \sI(\sigma'=f(\bm g(\sigma),s')).
\]
Hence the transitions of $\gM$ are fully characterized by the kernel $\gP$ and the functions $f$, $\bm g$ and $g$.

Given a set of subproblem policies $\pi=\{\pi^k\}_{k\in\ival{K}}$, an episodic semi-Markov decision process (SMDP) is a tuple $\bm\gM^\pi=\langle\Sigma,\ival{K},\bm\gP^\pi,\bm\gR^\pi,\bm H\rangle$, where boldface indicates SMDP components.
At each step $\bm h\in\ival{\bm H}$, the learner observes a state $\sigma_{\bm h}\in\Sigma$, selects a subproblem $k\in\ival{K}$ and executes the subproblem policy $\pi^k$ for $H$ steps to generate a state sequence $s_{1:H+1}$ starting from $s_1=g(\sigma_{\bm h})$.
The state becomes $\sigma_{\bm h+1}=f(\bm g(\sigma_{\bm h}),s_{H+1})$, and the process repeats $\bm H$ times, implying $\bm HH=N$.

The SMDP transitions $\bm\gP^\pi$ and reward $\bm\gR^\pi$ depend on the subproblem policies $\pi=\{\pi^k\}_{k\in\ival{K}}$.
Given subproblem policies $\pi$, state pair $(\sigma,\sigma')$ and subproblem $k$, the SMDP dynamics are defined as
\begin{align*}
\bm\gP^\pi(\sigma'|\sigma,k) \hspace*{-2pt} &= \hspace*{-4pt} \sum_{s_{H+1}} \hspace*{-1pt} \sP_{1:H+1}^{\pi^k}(s_{H+1}|s_1\hspace*{-1pt}=\hspace*{-1pt}g(\sigma)) \sI(\sigma'\hspace*{-2pt}=\hspace*{-2pt}f(\bm g(\sigma),s_{H+1})), \; \bm\gR^\pi(\sigma,k) \hspace*{-2pt} = \hspace*{-2pt} V_1^{\pi^k}(g(\sigma);\gY_{\bm g(\sigma)}),
\end{align*}
where $\gY_{\bm s}(s,a)=\gY(f(\bm s,s),a)$, $\bm s\in\bm\gS$.
Hence SMDP transitions associated with subproblem $k$ are defined by the $H$-step transition kernel $\sP_{1:H+1}^{\pi^k}$ of $\pi^k$, and SMDP rewards are defined by the value of $\pi^k$ under $\gY$.
We use $\widehat{\bm\gP}{}^{t,\pi}$ and $\widehat{\bm\gR}{}^{t,\pi}$ to denote the empirical estimates induced by $\widehat\sP_{1:H+1}^{t,\pi^k}$ and $\widehat V_1^{t,\pi^k}$.

A (deterministic) SMDP policy $\bm\pi = \{\bm \pi_{\bm h}\}_{\bm h \in \ival{\bm H}}$ is a collection of mappings $\bm\pi_{\bm h} : \Sigma \rightarrow \ival{K}$ selecting the subproblem to be solved in step $\bm h$.
A \emph{hierarchical policy} $(\bm\pi,\pi)$ consists of an SMDP policy $\bm\pi$ and a set of subproblem policies $\pi=\{\pi^k\}_{k\in\ival{K}}$.
The value function $\bm V^{\bm\pi,\pi}:[\bm H+1]\times\Sigma\to\sR$ of the hierarchical policy $(\bm\pi,\pi)$ is defined for each $\sigma\in\Sigma$ as $\bm V_{\bm H+1}^{\bm\pi,\pi}(\sigma)=0$ and
\begin{equation*}\label{eq:highV}
\bm V_{\bm h}^{\bm\pi,\pi}(\sigma) = \bm\gR^\pi(\sigma,\bm\pi_{\bm h}(\sigma)) + \sum_{\sigma'} \bm\gP^\pi(\sigma'|\sigma,\bm\pi_{\bm h}(\sigma)) \bm V_{\bm h+1}^{\bm\pi,\pi}(\sigma') \quad \forall\bm h\in\ival{\bm H}.
\end{equation*}
$\bm V_{\bm h}^{\bm\pi,\pi}(\sigma)$ is the value of the hierarchical policy $(\bm\pi,\pi)$ in the flat MDP $\mathcal{M}$, i.e.~the expected cumulative sum of rewards under the original reward function $\mathcal{Y}$ when using a sequence of subproblem policies $(\pi^{k_1},\dots,\pi^{k_{\bm H}})$, with each subproblem $k_{\bm h} = \bm \pi_{\bm h}(\sigma_{\bm h})$ chosen by the SMDP policy $\bm\pi$.

Let $\bm V_{\bm h}^*(\sigma)\equiv\max_{\bm\pi} \bm V_{\bm h}^{\bm\pi,\pi^*}(\sigma)$ be the optimal value function of the SMDP $\bm\gM^{\pi^*}$ induced by the optimal subproblem policies $\pi^*=\{\pi^{*,k}\}_{k\in\ival{K}}$, where $\pi^{*,k}$ is the optimal subproblem policy under $\mathcal{R}^{k}$.
The aim of the learner is to return a
hierarchical policy $(\widehat{\bm\pi},\widehat\pi)$
such that $\vert \bm V_1^*(\sigma_1) - \bm V_1^{\widehat{\bm\pi},\widehat\pi}(\sigma_1) \vert \leq \varepsilon$.
Note that $\bm V^*$ is optimal with respect to the hierarchical structure, but usually different from the optimal value function of $\gM$.
This is a commonly accepted tradeoff in the HRL community, since HRL algorithms can typically solve the SMDP problem more efficiently in practice.

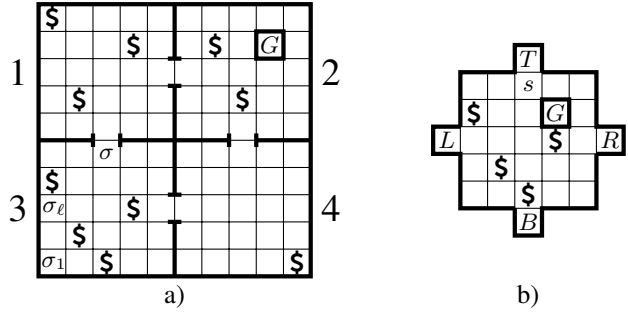
\begin{figure}[!t]
  \begin{center}
      \begin{tikzpicture}[scale=0.9]
        \draw[step=0.4,thin,shift={(0.2,0.2)}] (0.8,0.8) grid (4.8,4.8);
        \draw[ultra thick] (1,1) rectangle (5,5);
        \draw[ultra thick] (3,1) -- (3,1.8);
        \draw[ultra thick] (3,2.2) -- (3,3.8);
        \draw[ultra thick] (3,4.2) -- (3,5);
        \draw[ultra thick] (1,3) -- (1.8,3);
        \draw[ultra thick] (2.2,3) -- (3.8,3);
        \draw[ultra thick] (4.2,3) -- (5,3);
        
        \draw[ultra thick] (2.9,4.2) -- (3.1,4.2);
        \draw[ultra thick] (2.9,3.8) -- (3.1,3.8);
        \draw[ultra thick] (2.9,2.2) -- (3.1,2.2);
        \draw[ultra thick] (2.9,1.8) -- (3.1,1.8);
        
        \draw[ultra thick] (1.8,2.9) -- (1.8,3.1);
        \draw[ultra thick] (2.2,2.9) -- (2.2,3.1);
        \draw[ultra thick] (3.8,2.9) -- (3.8,3.1);
        \draw[ultra thick] (4.2,2.9) -- (4.2,3.1);

        \draw[ultra thick] (4.2,4.2) rectangle (4.6,4.6);

        \node at (4.4,4.4) {\small $G$};
        \node at (1.2,1.2) {\small $\,\sigma_1$};
        \node at (1.2,2.0) {\small $\,\sigma_\ell$};
        \node at (2.0,2.8) {\small $\sigma$};
        \node at (1.6,1.6) {\small \bf \faUsd};
        \node at (2.4,2.0) {\small \bf \faUsd};
        \node at (1.2,2.4) {\small \bf \faUsd};
        \node at (2.0,1.2) {\small \bf \faUsd};
        \node at (1.6,3.6) {\small \bf \faUsd};
        \node at (2.4,4.4) {\small \bf \faUsd};
        \node at (1.2,4.8) {\small \bf \faUsd};
        \node at (3.6,4.4) {\small \bf \faUsd};
        \node at (4.0,3.6) {\small \bf \faUsd};
        \node at (4.8,1.2) {\small \bf \faUsd};
        \node at (0.7,4.0) {\Large 1};
        \node at (5.3,4.0) {\Large 2};
        \node at (0.7,2.0) {\Large 3};
        \node at (5.3,2.0) {\Large 4};

        \draw[step=0.4,thin] (7.199,1.999) grid (9.2,4);
        \draw[ultra thick] (7.2,3.2) -- (6.8,3.2) -- (6.8,2.8) -- (7.2,2.8) -- (7.2,2) -- (8,2);
        \draw[ultra thick] (7.2,3.2) -- (7.2,4) -- (8,4) -- (8,4.4) -- (8.4,4.4) -- (8.4,4);
        \draw[ultra thick] (8.4,4) -- (9.2,4) -- (9.2,3.2) -- (9.6,3.2) -- (9.6,2.8) -- (9.2,2.8);
        \draw[ultra thick] (8,2) -- (8,1.6) -- (8.4,1.6) -- (8.4,2) -- (9.2,2) -- (9.2,2.8);
        \draw[ultra thick] (8.4,3.2) rectangle (8.8,3.6);

        \node at (8.6,3.4) {\small $G$};
        \node at (7,3)    {\small $L$};
        \node at (9.4,3)   {\small $R$};
        \node at (8.2,1.8)   {\small $B$};
        \node at (8.2,4.2)   {\small $T$};
        \node at (8.2,3.8)   {\small $s$};
        \node at (7.8,2.6) {\small \bf \faUsd};
        \node at (8.6,3.0) {\small \bf \faUsd};
        \node at (7.4,3.4) {\small \bf \faUsd};
        \node at (8.2,2.2) {\small \bf \faUsd};

        \node at (3,0.7) {a)};
        \node at (8.2,0.7) {b)};
      \end{tikzpicture}
  \end{center}
  {\color{red} \caption{a) An MDP representing a gridworld with 4 rooms and treasures; b) five subproblems corresponding to the terminal states $G, L, R, T, B$.}
  \label{fig:ex}
  }
\end{figure}

\textbf{An SMDP example.} To illustrate our definition of episodic SMDPs,
Figure~\ref{fig:ex}a) shows an episodic MDP $\gM=\langle\Sigma,\gA,\gX,\gY,N\rangle$ with initial state $\sigma_1$.
The aim is to reach the goal state $G$ while collecting treasures.
The reward $\gY$ is $1$ for collecting a treasure ({\small \bf \faUsd}) and for reaching $G$, and $0$ otherwise.
A treasure disappears when the agent steps on it, making it impossible to collect the reward twice.


We decompose $\gM$ by defining high-level states $\bm\gS=\{1,2,3,4\}$ for each room, and subproblems for reaching a terminal state among $G, L, R, T, B$ inside a room.
The subproblems include the location of treasures and have the same dynamics as the interior of each room.
Figure~\ref{fig:ex}b) shows the subproblem state corresponding to room 3.
The local reward $\gR^k$ is $1$ for collecting a treasure and for reaching the correct terminal state, and $0$ otherwise.
The function $f$ is not injective: both $(1,B)$ and $(3,s)$ map to state $\sigma$.
Since $\sigma$ is in room 3, $\bm g(\sigma)=3$ and $g(\sigma)=s$.
We can use $f$ to model that some terminal states are not available in a room, e.g.~in room $3$, reaching terminal state $L$ causes the agent to move to the state $\sigma_\ell=f(3,L)$.
The goal state $G$ is absorbing, but the agent can only achieve a reward of $1$ the first time it observes $G$.
The same is true of the terminal states of the subproblems.

\paragraph{Sufficient conditions for PAC-Learnability.} Since $\bm\gM^\pi$ depends on the subproblem policies $\pi$, the SMDP learning problem is non-stationary~\citep{nachum2018hrl}.
Two ingredients are necessary for the SMDP problem to be PAC-learnable.
First, the optimal SMDP value function $\bm V^*$ has to be unique, else the difference $\vert \bm V_1^*(\sigma_1) - \bm V_1^{\widehat{\bm\pi},\widehat\pi}(\sigma_1) \vert$ is ambiguous.
Hence the SMDP dynamics $\bm\gP^{\pi^*}$ and $\bm\gR^{\pi^*}$ of the optimal subproblem policies $\pi^*$ have to be unique.
Second, making progress towards $\pi^*$ should also make progress towards estimating $\bm\gP^{\pi^*}$ and $\bm\gR^{\pi^*}$,
else a subproblem policy $\pi$ may be arbitrarily close to solving the subproblems, but $\bm V^{\bm\pi,\pi}$ can still be very different from $\bm V^*$ for any $\bm\pi$.

We introduce two conditions that are sufficient to ensure that the SMDP problem is PAC-learnable.

\begin{assumption}\label{ass:rew_alt}
For each subproblem $k$, each optimal policy $\pi^{*,k}$ for $k$, each other policy $\pi^k$ and each state pair $(\bm s,s_1)\in\bm\gS\times\gS$ it holds that
\begin{align*}
\big\lVert \sP_{1:H+1}^{\pi^k}(\cdot|s_1) - \sP_{1:H+1}^{\pi^{*,k}}(\cdot|s_1) \big\rVert_1 &\leq 2 \big\vert V_1^{\pi^k}(s_1;\gR^k) - V_1^{\pi^{*,k}}(s_1;\gR^k) \big\vert,\\
\big\vert V_1^{\pi^k}(s_1;\gY_{\bm s}) - V_1^{\pi^{*,k}}(s_1;\gY_{\bm s}) \big\vert &\leq \big\vert V_1^{\pi^k}(s_1;\gR^k) - V_1^{\pi^{*,k}}(s_1;\gR^k) \big\vert.
\end{align*}
\end{assumption}

These conditions achieve both criteria.
If there are two optimal subproblem policies $\pi_1^*$ and $\pi_2^*$, the conditions ensure that $\bm\gP^{\pi_1^*}=\bm\gP^{\pi_2^*}$ and $\bm\gR^{\pi_1^*}=\bm\gR^{\pi_2^*}$.
If a subproblem policy $\pi$ is close to solving the subproblems, the right-hand side is small, implying that $\bm\gP^\pi$ and $\bm\gR^\pi$ tend towards $\bm\gP^{\pi^*}$ and $\bm\gR^{\pi^*}$.

In Appendix~\ref{app:ex} we prove the following proposition about the SMDP example in Figure~\ref{fig:ex}.

\begin{proposition}\label{prop:m}
Assume that the subproblem horizon $H$ is large enough for each optimal subproblem policy to collect all treasures in a room and reach the correct terminal state with probability $1$. Then the SMDP example in Figure~\ref{fig:ex} satisfies Assumption~\ref{ass:rew_alt}.
\end{proposition}

To illustrate our conditions, we provide several examples of SMDPs that do not satisfy them.
Assume that a subproblem $\gM^k$ has two terminal states $s_1^*$ and $s_2^*$ with associated rewards $1$ and $0.99$.
A subproblem policy $\pi^k$ that reaches $s_2^*$ achieves almost the same value as an optimal policy $\pi^{*,k}$ that reaches $s_1^*$, but $\sP_{1:H+1}^{\pi^k}$ is completely different from $\sP_{1:H+1}^{\pi^{*,k}}$.
Next assume that a subproblem $\gM^k$ has a single terminal state $s^*$ which is reached with probability $p$ from $s_1$ by an optimal policy $\pi^{*,k}$.
The difference $\Vert \sP_{1:H+1}^{\pi^k}(\cdot|s_1) - \sP_{1:H+1}^{\pi^{*,k}}(\cdot|s_1) \Vert_1$ of a policy $\pi^k$ that reaches $s^*$ with probability $q$ from $s_1$ can be as large as $2(1-q)$, which can be arbitrarily larger than $\big\vert V_1^{\pi^k}(s_1;\gR^k) - V_1^{\pi^{*,k}}(s_1;\gR^k) \big\vert=p-q$.
A concrete example of such a policy is one that ``gives up'' once it cannot reach the terminal state and instead wanders in a different direction.
A third example is one where the MDP rewards $\gY$ and subproblem rewards $\gR^k$ are non-zero for different non-terminal states of a subproblem $\gM^k$.
In this case, two policies can both be optimal for subproblem $\gM^k$ but collect different amounts of reward under $\gY$.
To the best of our knowledge, these issues have been ignored in previous work on parallel HRL, and we believe that our conditions consistute a major contribution in this regard, since without such conditions {\em the parallel HRL learning problem does not have a well-defined solution}.
We remark that our conditions are strictly weaker than those of~\cite{drappo2025hrl} who assume a single terminal state reached with probability one and zero reward for non-terminal states.



\section{The HBPI-UCRL Algorithm} \label{sec:algorithm}

We now present HBPI-UCRL, a novel model-based algorithm for HRL in episodic SMDPs.
Intuitively, the algorithm performs best policy identification on both levels of the hierarchy, using carefully tailored confidence bounds on the empirical estimates of SMDP transitions.
In each episode $t$, the learner uses a hierarchical policy $(\bm\pi^t,\pi^t)$ to generate a state sequence $\sigma_{1:\bm HH+1}^t$.
Let $\pi^{t,\bm h} \equiv \pi^{t,\bm\pi_{\bm h}^t(\sigma_{H(\bm h-1)+1})}$ be the subproblem policy selected by $\bm\pi^t$ at each $\bm h\in\ival{\bm H}$, and let $s_{1:H+1}^{t,\bm h}$
be the resulting subproblem state sequence.
The counts and empirical transition probabilities are given by
\begin{align*}
N^t(s,a,s') &= \sum_{\ell=1}^t \sum_{\bm h=1}^{\bm H} \sum_{h=1}^H \sI((s,a,s')=(s_h^{\ell,\bm h},a_h^{\ell,\bm h},s_{h+1}^{\ell,\bm h})), \quad \widehat\gP^t(s'|s,a) = \frac {N^t(s,a,s')} {\sum_{s'} N^t(s,a,s')},
\end{align*}
where $a_h^{\ell,\bm h}=\pi_h^{t,\bm h}(s_h^{\ell,\bm h})$.
The event $\gE$ in~\eqref{eq:event} still holds with probability $\mathbb{P}(\gE)\geq 1-\delta/2$.
HBPI-UCRL exploits that all uncertainty can be expressed in terms of the counts $N^t$ and the estimate $\widehat\gP^t$ of $\gP$.

Algorithm~\ref{alg:hbpi} provides pseudo-code for HBPI-UCRL.
The input is the parameters $\varepsilon,\delta$, the MDP $\gM$, the subproblems $\gM_k$, $k\in\ival{K}$, and the hierarchy in the form of SMDP states $\bm\gS$, functions $f$, $\bm g$ and $g$ and horizons $\bm H,H$.
In each episode $t$, HBPI-UCRL performs best policy identification for each subproblem $k$.
Concretely, the algorithm computes the confidence bounds $\gC^t$ in~\eqref{eq:confbounds} and the upper confidence bound $\overline V{}^{t,k}$ in~\eqref{eq:upperV} for each $k$, replacing $\gR$ with the local reward function $\gR^k$.
The policy $\pi^{t+1,k}$ is selected as the greedy policy for $\overline V{}^{t,k}$.
Given $\pi^{t+1}$, the algorithm also computes the function $\widehat L{}_h^{t,k}$ in~\eqref{eq:L} for each $k$ and the empirical SMDP dynamics $\widehat{\bm\gP}{}^{t,\pi^{t+1}}$ and $\widehat{\bm\gR}{}^{t,\pi^{t+1}}$ for $\pi^{t+1}$.

\begin{algorithm}[tb]
  \caption{HBPI-UCRL.}
  \label{alg:hbpi}
  \begin{algorithmic}
    \STATE {\bfseries Input:} Parameters $\varepsilon$, $\delta$, MDPs $\gM$, $\{\gM^k\}_{k\in\ival{K}}$, states $\bm\gS$, functions $f$, $\bm g$, $g$, horizons $\bm H,H$
    \STATE Initialize $t\gets -1$ and $N^0(s,a,s') \gets 0$ for each $s,a,s'\in\gS\times\gA\times\gS$
    \REPEAT
    \STATE $t\gets t+1$
    \STATE $\pi^{t+1,k},\widehat L^{t,k} \gets$ apply $\textsc{BPI}(N^t,H,\gR^k,\delta)$ for each $k\in\ival{K}$
    \STATE $\widehat{\bm\gP}{}^{t,\pi^{t+1}},\widehat{\bm\gR}{}^{t,\pi^{t+1}} \gets $ compute the empirical SMDP dynamics of the subproblem policies $\pi^{t+1}$
    \STATE $\bm\pi^{t+1}, \overline{\bm V}{}^t, \overline{\bm L}{}^t \gets $ compute the value function upper bound in~\eqref{eq:upperhighV} and the error function in~\eqref{eq:LL}
	\STATE use $\bm\pi^{t+1},\pi^{t+1}$ to collect an episode and update $N^{t+1}$
    \UNTIL{$\overline{\bm L}{}_1^t(\sigma_1) \leq \varepsilon/(6\bm HH)$}
    \STATE return $\bm\pi^{t+1},\pi^{t+1}$
  \end{algorithmic}
\end{algorithm}


To select the SMDP policy $\bm\pi^{t+1}$, HBPI-UCRL computes an SMDP value function $\overline{\bm V}{}^t$.
In Appendix~\ref{app:upper} we show that $\overline{\bm V}{}^t$ upper bounds $\bm V^*$ under event $\mathcal{E}$ and Assumption~\ref{ass:rew_alt}.
We define confidence sets as
\begin{align*}
\bm\gC^t(\sigma,k) &= \Big\{p\in\Delta(\Sigma) \; \big\vert \; \big\lVert p - \widehat{\bm\gP}{}^{t,\pi^{t+1}}(\cdot|\sigma,k) \big\rVert_1 \leq 5H\widehat L{}_1^{t,k}(g(\sigma)) \Big\} \quad \forall(\sigma,k).
\end{align*}
An upper confidence bound $\overline{\bm V}{}^t$ can now be recursively defined for each $\sigma\in\Sigma$ as $\overline{\bm V}{}_{\hspace*{-2pt}\bm H+1}^t(\sigma)=0$ and
\begin{align}
&\overline{\bm V}{}_{\hspace*{-2pt}\bm h}^t(\sigma) \hspace*{-1pt} = \hspace*{-1pt} \max_k \hspace*{-2pt} \Big[ \hspace*{-2pt} \min \hspace*{-2pt} \Big[ \bm HH, \widehat{\bm\gR}{}^{t,\pi^{t+1}}(\sigma,k) + 3H\widehat L_1^{t,k}(g(\sigma)) \hspace*{-1pt} + \hspace*{-9pt} \max_{p\in\bm\gC^t(\sigma,k)} \hspace*{-1pt} \sum_{\sigma'} p(\sigma') \overline{\bm V}{}_{\hspace*{-2pt}\bm h+1}^t(\sigma') \Big] \Big].\label{eq:upperhighV}
\end{align}
The policy $\bm\pi^{t+1}$ is selected as the greedy policy for $\overline{\bm V}{}^t$.
For each $\bm h$ and $\sigma$ we define the transition kernel for $\bm\pi^{t+1}$ that achieves the maximum over $\bm\gC^t$ in~\eqref{eq:upperhighV} as
\[
\overline{\bm\gP}{}_{\hspace*{-2pt}\bm h}^t(\cdot\,|\sigma)=\arg\max_{p\in\bm\gC^t(\sigma,\bm\pi_{\bm h}^{t+1}(\sigma))} \sum_{\sigma'} p(\sigma') \overline{\bm V}{}_{\hspace*{-2pt}\bm h+1}^t(\sigma').
\]
We next define an SMDP error function $\overline{\bm L}{}^t$ recursively for each $\sigma$ as $\overline{\bm L}{}_{\bm H+1}^t(\sigma)=0$ and
\begin{align}
\overline{\bm L}{}_{\bm h}^t(\sigma) &= \min \Big[ 2, \bm B_{\bm h}^t(\sigma) + \sum_{\sigma'} \overline{\bm\gP}{}_{\hspace*{-2pt}\bm h}^t(\sigma'|\sigma) \overline{\bm L}{}_{\bm h+1}^t(\sigma') \Big] \quad \forall\bm h\in\ival{\bm H},\label{eq:LL}
\end{align}
where $\bm B_{\bm h}^t(\sigma) \equiv 10H\widehat L_1^{t,\bm\pi_{\bm h}^{t+1}(\sigma)}(g(\sigma))$.
The stopping criterion of HBPI-UCRL is defined as 
$\tau = \inf \left\{t \in \mathbb{N} : \overline{\bm L}{}_1^t(\sigma_1)\leq\varepsilon/(6\bm HH)\right\}$.
Upon stopping, HBPI-UCRL returns $(\bm\pi^{\tau+1},\pi^{\tau+1})$.

\textbf{General analysis of HBPI-UCRL.} The following theorem provides a sample complexity upper bound for the novel algorithm HBPI-UCRL.

\begin{theorem}\label{thm:hrf}
Under Assumption~\ref{ass:rew_alt}, with probability larger than $1-\delta$ HBPI-UCRL returns a hierarchical policy $(\widehat{\bm\pi},\widehat\pi)$ such that $\lVert \bm V_1^*(\sigma_1)-\bm V_1^{\widehat{\bm\pi},\widehat\pi}(\sigma_1) \rVert_1 \leq \varepsilon$ using a number of episodes
\begin{align*}
\tau = \gO\Big( &\frac {SA\bm H^4H^6} {\varepsilon^2} \Big( \log \frac {SA\bm{H}H} \delta + S \log \Big( \frac {SA \bm H^5H^7} {\varepsilon^2} \log \frac {SA\bm{H}H} \delta \Big) \Big) \Big).
\end{align*}
\end{theorem}

\noindent
The proof appears in Appendix~\ref{app:hrf}.
If we apply BPI-UCRL to solve the flat MDP $\gM$, we obtain a sample complexity bound of order $\widetilde\gO(|\Sigma| AN^4/\varepsilon^2)=\widetilde\gO(\bm SSA\bm H^4H^4/\varepsilon^2)$ in the regime of small $\delta$.
Hence HBPI-UCRL is more sample efficient if $H^2<\bm S$, though we recall that the
optimal policy for $\gM$ is in general different from the optimal hierarchical policy $(\bm\pi^*,\pi^*)$.

\textbf{HBPI-UCRL for Sparse-Reward SMDPs.}
We next show that in sparse-reward SMDPs, a modified version of HBPI-UCRL achieves a lower sample complexity.
An SMDP $\bm\gM^\pi$ is sparse-reward if both the flat MDP $\gM$ and the subproblems $\gM^k$, $k\in\ival{K}$, are sparse-reward.
In Appendix~\ref{app:ex} we prove that the SMDP example in Figure~\ref{fig:ex} is sparse-reward if we remove all treasures.

To adapt HBPI-UCRL to sparse-reward SMDPs, we modify the algorithm in several ways.
We first redefine the upper confidence bound of the value function as $\overline{\bm V}{}_{\hspace*{-2pt}\bm H+1}^t(\sigma)=0$ and
\begin{align*}
\overline{\bm V}{}_{\hspace*{-2pt}\bm h}^t(\sigma) &= \max_k \Big[ \min \Big[ 1, \widehat{\bm\gR}{}^{t,\pi^{t+1}}(\sigma,k) + 3\widehat L_1^{t,k}(g(\sigma)) + \max_{p\in\bm\gC^t(\sigma,k)} \sum_{\sigma'} p(\sigma') \overline{\bm V}{}_{\hspace*{-2pt}\bm h+1}^t(\sigma') \Big] \Big],\\
\bm\gC^t(\sigma,k) &= \Big\{p\in\Delta(\Sigma) \;\Big\vert \; \big\lVert p - \widehat{\bm\gP}{}^{t,\pi^{t+1}}(\cdot|\sigma,k) \big\rVert_1 \leq 5\widehat L{}_1^{t,k}(g(\sigma)) \Big\} \quad \forall(\sigma,k).
\end{align*}
Compared to~\eqref{eq:upperhighV}, the confidence bound $\bm\gC^t$ is tighter, and the value function upper bound takes the minimum with $1$ instead of the minimum with $\bm HH$.
In Appendix~\ref{app:better} we prove that the new definition of $\overline{\bm V}{}^t$ is an upper bound on $\bm V^*$ for sparse-reward SMDPs under event $\gE$ and Assumption~\ref{ass:rew_alt}.
We redefine $\bm B_{\bm h}^t(\sigma) \equiv 10\widehat L_1^{t,\bm\pi_{\bm h}^{t+1}(\sigma)}(g(\sigma))$ and we modify the stopping criterion to be $\overline{\bm L}{}_1^t(\sigma_1) \leq \varepsilon/6$.

In Appendix~\ref{app:better} we prove the following theorem.

\begin{theorem}\label{thm:hrf2}
For sparse-reward SMDPs, under Assumption~\ref{ass:rew_alt}, w.p.~at least $1-\delta$ the modified HBPI-UCRL returns $(\widehat{\bm\pi},\widehat\pi)$ such that $\lVert \bm V_1^*(\sigma_1)-\bm V_1^{\widehat{\bm\pi},\widehat\pi}(\sigma_1) \rVert_1 \leq \varepsilon$ using a number of episodes
\begin{align*}
\tau = \gO\Big( & \frac {SA\bm H^2H^2} {\varepsilon^2} \Big( \log \frac {SA\bm{H}H} \delta + S \log \Big( \frac {SA \bm H^3H^3} {\varepsilon^2} \log \frac {SA\bm{H}H} \delta \Big) \Big) \Big).
\end{align*}
\end{theorem}

In comparison, the sample complexity of BPI-UCRL is $\widetilde\gO(\bm SSA\bm H^2H^2/\varepsilon^2)$ for sparse-reward $\gM$ in the small $\delta$ regime.
Hence the sample complexity of HBPI-UCRL is a factor $\bm S$ smaller, and only a factor $\bm H^2$ larger than that of BPI-UCRL when applied to a single sparse-reward subproblem $\gM^k$.

\section{Related Work}\label{sec:related}

Several authors have studied the sample complexity~\citep{brunskill2014options} or regret~\citep{drappo2023hrl,fruit2017hrl} of HRL when the optimal subproblem policies are given as prior knowledge.
\citet{nachum2019hrl} study representations that bound the suboptimality gap between the optimal flat policy and the optimal SMDP policy.
\citet{kuric2024hrl} present an algorithm that provably converges to a hierarchical policy equivalent to the optimal flat policy.

Recent works analyze the setting in which the SMDP and subproblem policies are learned in parallel.
The focus of prior work is often on regret, and we believe that ours is the first sample complexity upper bound for parallel HRL policy learning. 
While we propose general conditions for efficient HRL, most existing works define SMDPs specifically for goal-oriented episodic tasks \citep{wen2020hrl,robert2023hrl,drappo2025hrl}.
\citet{drappo2024hrl} analyze the regret in a more general HRL framework, but the regret depends on the concentrability coefficients at the SMDP and subproblem levels, and the authors do not make any attempt to bound these coefficients.
\citet{manenti2025hrl} prove that Q-learning converges in the parallel HRL setting, but do not show that the optimal SMDP value function is unique.

\citet{wen2020hrl} use posterior sampling to exploit a hierarchical structure and prove a bound on the Bayesian regret of the resulting algorithm in which the size of the state space is replaced by the size of the subproblem state space. Due to the Bayesian nature of this result, it is not possible to derive an $(\varepsilon,\delta)$-PAC algorithm from their regret in our frequentist setting. Moreover, the algorithm that achieves their regret result does not learn a hierarchical policy.

The work of \citet{drappo2025hrl} is more similar to ours in spirit, despite their focus on regret. Indeed, they also use an optimistic approach to learn the subproblem policies, but they build a quite different upper bound on the value of their high-level policy tailored to the goal-oriented setting.
They assume that the next high-level state after solving a subproblem is known, which allows to obtain an efficient algorithm by storing high-level states only.
Note that this approach cannot solve the example in Figure~\ref{fig:ex}. Using the regret-to-PAC conversion~\citep{Jin18OptQL},  
which runs their algorithm for a large enough $T$ and outputs the hierarchical policy used in a random episode in $\{1,\dots,T\}$, one can prove that choosing $T = \widetilde\gO(SA K \bm HH/(\varepsilon^2\delta^2))$ yields an $(\varepsilon,\delta)$-PAC policy. Compared to this bound, our bound shaves off a factor $K$, which comes from the fact that their setting considers different possible state spaces for each subproblem (e.g. different room shapes but with the transitions all inherited from the underlying flat MDP). Our worse scaling in the horizons can be explained by two things. The improvement in $H$ comes from the fact that their regret minizer is UCBVI with Bernstein bonuses \citep{Azar17UCBVI}, analyzed in the special case of sparse reward. Regarding our extra $\bm H$ factor, it is mitigated by the fact that their setting is slightly different, as their algorithm terminates an episode early when a subgoal is not reached by the current subproblem policy. Hence some of their episodes could contain only one low-level episode, which implies that the number of episodes has to be multiplied by $\bm{H}$.
In contrast, HBPI-UCRL never terminates an episode early.
Besides the fact that our algorithm is not specific to the goal-oriented setting, the main improvement in our bound is the dependency in $\delta$, which comes from our data-dependent stopping rule (a non-trivial component of HBPI-UCRL) instead of a deterministic sample size $T$ fixed in advance.

The work of \citet{robert2023hrl}  considers a slightly different sample complexity: following \citet{Kakade03PhD}, their sample complexity is defined as the number of episodes collected using a policy that is not $\varepsilon$-optimal.
Still, their lower bound also applies to our notion of sample complexity.
Their main result states that any HRL algorithm needs a number of samples of order $\Omega(\max(SAKH^2/\varepsilon^2,\bm SK\bm H^2/\varepsilon^2))$, which is inconsistent with our sample complexity upper bounds in Theorems~\ref{thm:hrf} and~\ref{thm:hrf2}.
In the first term, the additional factor $K$ is due to the fact that unlike us, they do not assume that subproblems share the transition kernel. However, we believe that the second term only applies to a slightly different setting, in which the HRL algorithm has to sample high-level transitions and rewards.
In contrast, we exploit the fact that high-level transitions and rewards are fully determined by the subproblem policies and the subproblem dynamics.

\section{Numerical Experiments}\label{sec:exp}


\begin{figure}[t!]
    \subfloat[Stopping time $\tau$ as a function of $\bm S$.]{
        \resizebox{0.485\linewidth}{0.45\linewidth}{\includegraphics[height=2cm]{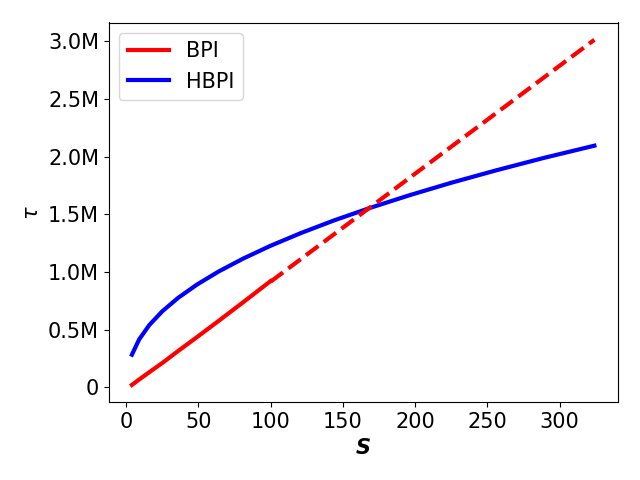}}%
        \label{fig:stopdet}%
    }\hfill
    \subfloat[Unverifiable stopping time as a function of $\bm S$.]{
        \resizebox{0.485\linewidth}{0.45\linewidth}{\includegraphics[height=2cm]{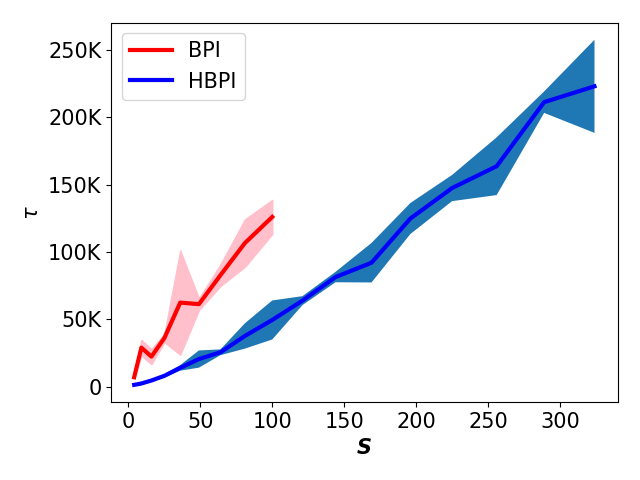}}%
        \label{fig:stopemp}%
    }
    \caption{Stopping time of BPI-UCRL and HBPI-UCRL as functions of the number of high-level states $\bm S$: (a) using the stopping criterion; (b) unverifiable by the algorithms but observed in practice.}
    \label{fig:regret_comparison}
\end{figure}

In this section we present results from numerical experiments with episodic (S)MDPs.
The aim is to show that HBPI-UCRL is more sample efficient in practice than BPI-UCRL~\citep{kaufmann2021rf}.
We carry out experiments on a deterministic version of the SMDP in Figure~\ref{fig:ex} with treasures removed, varying the size of each room and the number of rooms. 
Hence the SMDPs are sparse-reward and aim to reach the goal $G$.
Implementation details and additional experimental results in domains with treasures and stochastic transitions can be found in Appendix~\ref{app:exp}.
\revision{In particular, some randomness is introduced by the use of a random initial state in each episode, and we perform experiments (that are very costly) with three different seeds, with the shaded area indicating the standard deviation.}




Figure~\ref{fig:stopdet} shows the stopping time $\tau$ of BPI-UCRL and HBPI-UCRL as functions of the number of high-level states $\bm S$ for rooms of size $3\times 3$, varying the number of rooms from $2\times 2$ to $18\times 18$ and setting the SMDP horizon to $\bm H=2\sqrt{\bm S}$.
BPI becomes too slow for $11\times 11$ rooms, and the broken line from $10\times 10$ to $18\times 18$ is an interpolation.
The stopping time of HBPI-UCRL is clearly proportional to $\sqrt{\bm S}$, which validates our theory and indicates that a sample complexity of order $\widetilde\gO( SA\bm HH/\varepsilon^2 )$ is possible for sparse SMDPs with a tighter analysis (i.e.~independent of $\bm S$ but linear in $\bm H\propto\sqrt{\bm S}$).
The stopping time of BPI-UCRL is clearly not sublinear in $\bm S$.
We carry out several ablations in Appendix~\ref{app:exp} to separately test the impact of $\bm H$ and $\bm S$ on the sample complexity.

\revision{Figure~\ref{fig:stopemp} reports the last episode $t$ for which $\vert \bm V_1^*(\sigma_1) - \bm V_1^{\bm\pi^{t+1},\pi^{t+1}}(\sigma_1) \vert > \varepsilon$, for both BPI-UCRL and HBPI-UCRL. From that point on the algorithms keep using an $\varepsilon$ optimal policy and have therefore identified one, but without being able to certify it. This quantity is related to the notion of unverifiable sample complexity introduced by \cite{Katz-Samuels20Unverifiable} for pure exploration in bandits. We remark an interesting phenomenon: 
Although the theoretical stopping time of HBPI-UCRL is larger than that of BPI-UCRL for small values of $\bm S$, in practice it needs fewer episodes to converge.
This experiment also reveals  that the stopping time based on Hoeffding bounds is not tight.}

\section{Discussion}\label{sec:discuss}

There are several possible extensions of our work, and here we mention a few.
One such extension is to define the SMDP policy and value function on high-level states rather than full states, as is common in the HRL literature.
The challenge in this setting is that the initial state of each subproblem is not fully determined, which makes it harder to formalize and analyze the SMDP dynamics.

A related question is how to adapt Assumption~\ref{ass:rew_alt} to make it more general.
We have not fully characterized the class of HRL problems that satisfy the assumptions, but all known examples involve optimal subproblem policies that terminate in a single state with probability $1$.
Either HRL is sample efficient only in the case of goal-directed subproblems (which would be a contribution in its own right), or researchers have to propose other conditions that cause HRL to be PAC-learnable.


In our work, all subproblems share the transition dynamics, but this is not always the case.
An open research question is how to extend HBPI-UCRL to SMDPs whose subproblems have different transition dynamics, or subproblems that belong to one of several equivalence classes.
Finally, an interesting direction for future work is to extend HBPI-UCRL to the infinite-horizon setting.

\newpage

\bibliographystyle{plainnat}
\bibliography{neurips26}

\newpage

\onecolumn

\renewcommand{\contentsname}{Supplementary Materials}

\tableofcontents

\newpage

\appendix

\section{Equivalent Definitions of Recursive Functions}\label{app:tech}

In this appendix we prove the following lemma, which states that our composite probability distributions allow us to express equivalent definitions of recursive functions.
\begin{lemma}\label{lemma:alt}
Given a policy $\pi$, let $\{f_h\}_{h\in\ival{H}}$ and $\{g_h^\pi\}_{h\in\ival{H+1}}$ be sets of functions $f_h:\gS\times\gA\to\sR$ and $g_h^\pi:\gS\to\sR$. The following three definitions of $g_h^\pi$ are equivalent for each $s_h$ if we assume $g_{H+1}^\pi(s_h)=0$:
\begin{align*}
g_h^\pi(s_h) &= f_h(s_h,\pi_h(s_h)) + \sum_{s'} \gP_h^\pi(s'|s_h) g_{h+1}^\pi(s'),\\
g_h^\pi(s_h) &= \hspace*{-8pt} \sum_{s_{h+1:H+1}} \hspace*{-8pt} \etP_{h:H+1}^\pi(s_{h+1:H+1}|s_h) \sum_{i=h}^H f_i(s_i,\pi_i(s_i)),\\
g_h^\pi(s_h) &= \sum_{i=h}^H \sum_{s_i} \sP_{h:i}^\pi(s_i|s_h) f_i(s_i,\pi_i(s_i)).
\end{align*}
\end{lemma}
The proof is by induction, with the base case given by $h=H$. In this case, each of the three definitions can be written as
\begin{align*}
g_H^\pi(s_H) &= f_H(s_H,\pi_H(s_H)) + 0 = f_H(s_H,\pi_H(s_H)),\\
g_H^\pi(s_H) &= \hspace*{-9pt} \sum_{s_{H+1:H+1}} \hspace*{-9pt} \etP_{H:H+1}^\pi(s_{H+1:H+1}|s_H) \sum_{i=H}^H f_i(s_i,\pi_i(s_i))\\
 &= f_H(s_H,\pi_H(s_H)) \sum_{s_{H+1}} \gP_H^\pi(s_{H+1}|s_H) = f_H(s_H,\pi_H(s_H)),\\
g_H^\pi(s_H) &= \sum_{i=H}^H \sum_{s_i} \sP_{H:i}^\pi(s_i|s_H) f_i(s_i,\pi_i(s_i)) = \sum_{s_i} \sI(s_i=s_H) f_i(s_i,\pi_i(s_i)) = f_H(s_H,\pi_H(s_H)),
\end{align*}
where we have used the definitions of $\etP_{H:H+1}^\pi$ and $\sP_{H:H}^\pi$.

The inductive case is given by $h\in\ival{H-1}$. We show that the second and third definitions are equivalent to the first. The second definition can be written
\begin{align*}
&g_h^\pi(s_h) = \sum_{s_{h+1:H+1}} \etP_{h:H+1}^\pi(s_{h+1:H+1}|s_h) \sum_{i=h}^H f_i(s_i,\pi_i(s_i))\\
 &= f_h(s_h,\pi_h(s_h)) \sum_{s_{h+1:H+1}} \etP_{h:H+1}^\pi(s_{h+1:H+1}|s_h)\\
 & \hspace*{10pt} + \hspace*{-8pt} \sum_{s_{h+1:H+1}} \hspace*{-8pt} \gP_h^\pi(s_{h+1}|s_h)\, \etP_{h+1:H+1}^\pi(s_{h+2:H+1}|s_{h+1}) \sum_{i=h+1}^H f_i(s_i,\pi_i(s_i))\\
 &= f_h(s_h,\pi_h(s_h)) + \sum_{s_{h+1}} \gP_h^\pi(s_{h+1}|s_h) \sum_{s_{h+2:H+1}} \etP_{h+1:H+1}^\pi(s_{h+2:H+1}|s_{h+1}) \sum_{i=h+1}^H f_i(s_i,\pi_i(s_i))\\
 &= f_h(s_h,\pi_h(s_h)) + \sum_{s_{h+1}} \gP_h^\pi(s_{h+1}|s_h) g_{h+1}^\pi(s_{h+1}),
\end{align*}
where we have used the equality $\sum_{s_{h+1:H+1}} \etP_{h:H+1}^\pi(s_{h+1:H+1}|s_h)=1$, the recursive definition of $\etP_{h:H+1}^\pi$ and the inductive hypothesis. The third definition can be written
\begin{align*}
g_h^\pi(s_h) &= \sum_{i=h}^H \sum_{s_i} \sP_{h:i}^\pi(s_i|s_h) f_i(s_i,\pi_i(s_i))\\
 &= \sum_{s_i} \sI(s_i=s_h) f_i(s_i,\pi_i(s_i)) + \sum_{i=h+1}^H \sum_{s_i} \sum_{s_{h+1}} \gP_h^\pi(s_{h+1}|s_h) \sP_{h+1:i}^\pi(s_i|s_{h+1}) f_i(s_i,\pi_i(s_i))\\
 &= f_h(s_h,\pi_h(s_h)) + \sum_{s_{h+1}} \gP_h^\pi(s_{h+1}|s_h) \sum_{i=h+1}^H \sum_{s_i} \sP_{h+1:i}^\pi(s_i|s_{h+1}) f_i(s_i,\pi_i(s_i))\\
 &= f_h(s_h,\pi_h(s_h)) + \sum_{s_{h+1}} \gP_h^\pi(s_{h+1}|s_h) g_{h+1}^\pi(s_{h+1}),
\end{align*}
where we have used the definition of $\sP_{h:h}^\pi$, the recursive definition of $\sP_{h:i}^\pi$ and the inductive hypothesis. This proves that the three definitions are equivalent.

\section{Properties of the SMDP Example}\label{app:ex}

In this section we prove Propositions~\ref{prop:m} and~\ref{prop:1} given below, which state that the SMDP example in Figure~\ref{fig:ex} satisfies various assumptions.

\begin{proposition}\label{prop:1}
Assume that the subproblem horizon $H$ is large enough for each optimal subproblem policy to reach the correct terminal state with probability $1$. Then the SMDP example in Figure~\ref{fig:ex} with all treasures removed is sparse-reward and satisfies Assumption~\ref{ass:rew_alt}.
\end{proposition}

We first show that if the subproblem horizon $H$ is large enough for each optimal subproblem policy to collect all treasures and reach the correct terminal state $s^*$ with probability $1$, then the first condition of the assumption holds.
If a subproblem has $m$ treasures, then the value of the optimal subproblem policy $\pi^{*,k}$ in any non-terminal state $s$ is $V_1^{\pi^{*,k}}(s;\gR^k)=m+1$.
Any subproblem policy $\pi^k$ whose value is at most $m$ satisfies the first condition since $2|V_1^{\pi^k}(s;\gR^k)-V_1^{\pi^{*,k}}(s;\gR^k)|\geq 2\geq \lVert \sP_{1:H+1}^{\pi^k}(\cdot|s) - \sP_{1:H+1}^{\pi^{*,k}}(\cdot|s) \rVert_1$.
If $\pi^k$ has value $V_1^{\pi^k}(s;\gR^k)>m$, it reaches $s^*$ with probability at least $V_1^{\pi^k}(s;\gR^k)-m$, implying $\lVert \sP_{1:H+1}^{\pi^k}(\cdot|s) - \sP_{1:H+1}^{\pi^{*,k}}(\cdot|s) \rVert_1 \leq 2|V_1^{\pi^k}(s;\gR^k)-V_1^{\pi^{*,k}}(s;\gR^k)|$.

We next show that under the same condition, the SMDP example in Figure~\ref{fig:ex} satisfies the second condition of the assumption.
Recall that the reward functions $\gY_{\bm s}$ and $\gR^k$ are the same for all non-terminal states of a subproblem $k$ and all high-level states $\bm s$.
If a subproblem has $m$ treasures, then the optimal subproblem policy $\pi^{*,k}$ satisfies $V_1^{\pi^{*,k}}(s;\gY_{\bm s})=m$ and $V_1^{\pi^{*,k}}(s;\gR^k)=m+1$.
If the value of a subproblem policy $\pi^k$ under $\gY_{\bm s}$ is $V_1^{\pi^k}(s;\gY_{\bm s})=\ell$, then the value under $\gR^k$ can be at most $V_1^{\pi^k}(s;\gR^k)=\ell+1$, implying $\vert V_1^{\pi^{*,k}}(s;\gR^k)-V_1^{\pi^k}(s;\gR^k) \vert \geq m-\ell$.
Hence we have $\vert V_1^{\pi^k}(s_1;\gY_{\bm s}) - V_1^{\pi^{*,k}}(s_1;\gY_{\bm s}) \vert =m-\ell\leq \vert V_1^{\pi^{*,k}}(s;\gR^k)-V_1^{\pi^k}(s;\gR^k) \vert$, which satisfies the second condition.
Taken together, these two results are enough to prove Proposition~\ref{prop:m}.

To prove Proposition~\ref{prop:1}, note that even if we remove all treasures, the reasoning above still holds.
Specifically, this implies that Assumption~\ref{ass:rew_alt} is satisfied.
If we remove all treasures, the only available reward is $1$ for reaching $G$ the first time in the case of $\gM$, and $1$ for reaching the correct terminal state the first time in the case of $\gM^k$.
Hence the flat MDP $\gM$ and the subproblems $\gM^k$, $k\in\ival{K}$, are all sparse-reward, which is precisely the definition of a sparse-reward SMDP.

\section{Proof of Theorem~\ref{thm:rf}}\label{app:rf}

In this appendix we prove Theorem~\ref{thm:rf}, which states an upper bound on the sample complexity of BPI-UCRL.
We first introduce an error function $L^t$ analogous to $\widehat L{}^t$ but associated with the true transition kernel $\gP^{\pi^{t+1}}$ of $\pi^{t+1}$.
The function $L^t$ is recursively defined as $L_{H+1}^t(s)=0$ and
\begin{align*}
L_h^t(s) &= B^t(s,\pi_h^{t+1}(s)) + \sum_{s'} \gP_h^{\pi^{t+1}}(s'|s) L_{h+1}^t(s') \quad \forall h\in\ival{H}.
\end{align*}
We also define $B_h^{t,\pi}(s)\equiv B^t(s,\pi_h(s))$ for each policy $\pi$ and step $h$.

The following lemma bounds the L1-norm of the difference between the composite distributions induced by $\widehat\gP{}^{t,\pi^{t+1}}$ and those induced by any transition kernel for $\pi^{t+1}$ in the confidence sets $\gC^t$.
\begin{lemma}\label{lemma:tsbound}
Let $\{\gQ_h^{\pi^{t+1}}\}_{h\in\ival{H}}$ be a transition kernel for $\pi^{t+1}$ s.t.~$\gQ_h^{\pi^{t+1}}(\cdot|s_h)\in\gC^t(s_h,\pi_h^{t+1}(s_h))$ for each $h$ and $s_h$, and let $\etQ_{i:j}^{\pi^{t+1}}$ and $\sQ_{i:j}^{\pi^{t+1}}$ be its composite distributions.
Under event $\gE$ we have
\begin{align*}
\big\lVert \etQ_{h:H+1}^{\pi^{t+1}}(\cdot|s_h)-\widehat\etP_{h:H+1}^{t,\pi^{t+1}}(\cdot|s_h) \big\rVert_1 &\leq \widehat L{}_h^t(s_h) \leq 3 L_h^t(s_h),\\
\big\lVert \sQ_{h:H+1}^{\pi^{t+1}}(\cdot|s_h)-\widehat\sP_{h:H+1}^{t,\pi^{t+1}}(\cdot|s_h) \big\rVert_1 &\leq \widehat L{}_h^t(s_h) \leq 3 L_h^t(s_h).
\end{align*}
\end{lemma}

\begin{proof}
We first show by induction on $h$ that $\widehat L_h^t(s_h)$ upper bounds $\lVert \etQ_{h:H+1}^{\pi^{t+1}}(\cdot|s_h)-\widehat\etP_{h:H+1}^{t,\pi^{t+1}}(\cdot|s_h) \rVert_1$ and $\lVert \sQ_{h:H+1}^{\pi^{t+1}}(\cdot|s_h)-\widehat\sP_{h:H+1}^{t,\pi^{t+1}}(\cdot|s_h) \rVert_1$.
The base case is given by $h=H$. In this case by definition we have
\begin{align*}
\widehat\etP_{H:H+1}^{t,\pi^{t+1}}(\cdot|s_H) &= \widehat\sP_{H:H+1}^{t,\pi^{t+1}}(\cdot|s_H) = \widehat\gP_H^{t,\pi^{t+1}}(\cdot|s_H),\\
\etQ_{H:H+1}^{\pi^{t+1}}(\cdot|s_H) &= \sQ_{H:H+1}^{\pi^{t+1}}(\cdot|s_H) = \gQ_H^{\pi^{t+1}}(\cdot|s_H),\\
\widehat L_H^t(s_H) &= \min \left[ 2, B_H^{t,\pi^{t+1}}(s_H) + 0 \right] = B_H^{t,\pi^{t+1}}(s_H),
\end{align*}
since $B_H^{t,\pi^{t+1}}(s_H)$ already involves a minimum with $2$.
Consequently we can write
\begin{align*}
\normone{ \etQ_{H:H+1}^{\pi^{t+1}}(\cdot|s_H)-\widehat\etP_{H:H+1}^{t,\pi^{t+1}}(\cdot|s_H) } &= \normone{ \gQ_H^{\pi^{t+1}}(\cdot|s_H) - \widehat\gP_H^{t,\pi^{t+1}}(\cdot|s_H) } \\ &\leq B_H^{t,\pi^{t+1}}(s_H) = \widehat L_H^t(s_H),
\end{align*}
where the inequality follows from the fact that $\gQ_H^{\pi^{t+1}}(\cdot|s_H)$ belongs to $\gC^t(s_H,\pi_H^{t+1}(s_H))$.
The proof for $\lVert \sQ_{H:H+1}^{\pi^{t+1}}(\cdot|s_H)-\widehat\sP_{H:H+1}^{t,\pi^{t+1}}(\cdot|s_H) \rVert_1$ is identical.
In the inductive case $h\in\ival{H-1}$ we have
\begin{small}
\begin{align*}
&\normone{ \etQ_{h:H+1}^{\pi^{t+1}}(\cdot|s_h)-\widehat\etP_{h:H+1}^{t,\pi^{t+1}}(\cdot|s_h) } = \hspace*{-8pt} \sum_{s_{h+1:H+1}} \hspace*{-8pt} \left\vert \etQ_{h:H+1}^{\pi^{t+1}}(s_{h+1:H+1}|s_h) - \widehat\etP_{h:H+1}^{t,\pi^{t+1}}(s_{h+1:H+1}|s_h) \right\vert\\
 &\leq_{(a)} \hspace*{-8pt} \sum_{s_{h+1:H+1}} \hspace*{-8pt} \left\vert \gQ_h^{\pi^{t+1}}(s_{h+1}|s_h) \etQ_{h+1:H+1}^{\pi^{t+1}}(s_{h+2:H+1}|s_{h+1}) - \widehat\gP_h^{t,\pi^{t+1}}(s_{h+1}|s_h) \etQ_{h+1:H+1}^{\pi^{t+1}}(s_{h+2:H+1}|s_{h+1}) \right\vert\\
 & \hspace*{12pt} + \hspace*{-8pt} \sum_{s_{h+1:H+1}} \hspace*{-8pt} \left\vert \widehat\gP_h^{t,\pi^{t+1}}(s_{h+1}|s_h) \etQ_{h+1:H+1}^{\pi^{t+1}}(s_{h+2:H+1}|s_{h+1}) - \widehat\gP_h^{t,\pi^{t+1}}(s_{h+1}|s_h) \widehat\etP_{h+1:H+1}^{t,\pi^{t+1}}(s_{h+2:H+1}|s_{h+1}) \right\vert\\ 
 &=_{(b)} \sum_{s_{h+1}} \left\vert \gQ_h^{\pi^{t+1}}(s_{h+1}|s_h) - \widehat\gP_h^{t,\pi^{t+1}}(s_{h+1}|s_h) \right\vert \sum_{s_{h+2:H+1}} \hspace*{-8pt} \etQ_{h+1:H+1}^{\pi^{t+1}}(s_{h+2:H+1}|s_{h+1})\\ 
 & \hspace*{11pt} + \sum_{s_{h+1}} \widehat\gP_h^{t,\pi^{t+1}}(s_{h+1}|s_h) \hspace*{-8pt}  \sum_{s_{h+2:H+1}} \hspace*{-8pt} \left\vert \etQ_{h+1:H+1}^{\pi^{t+1}}(s_{h+2:H+1}|s_{h+1}) - \widehat\etP_{h+1:H+1}^{t,\pi^{t+1}}(s_{h+2:H+1}|s_{h+1}) \right\vert\\
 &=_{(c)} \normone{ \gQ_h^{\pi^{t+1}}(\cdot|s_h) - \widehat\gP_h^{t,\pi^{t+1}}(\cdot|s_h) } \hspace*{-3.5pt} + \hspace*{-1.5pt} \sum_{s_{h+1}} \widehat\gP_h^{t,\pi^{t+1}}(s_{h+1}|s_h) \normone{ \etQ_{h+1:H+1}^{\pi^{t+1}}(\cdot|s_{h+1}) - \etP_{h+1:H+1}^{\pi^{t+1}}(\cdot|s_{h+1}) }\\
 &\leq_{(d)} B_h^{t,\pi^{t+1}}(s_h) + \sum_{s_{h+1}} \widehat\gP_h^{t,\pi^{t+1}}(s_{h+1}|s_h) \widehat L_{h+1}^t(s_{h+1}) \equiv \widehat Z_h^t(s_h).
\end{align*}
\end{small}
In $(a)$ we use the definition of $\etQ_{h:H+1}^{\pi^{t+1}}$ and $\widehat\etP_{h:H+1}^{t,\pi^{t+1}}$ and the triangle inequality, in $(b)$ we use the equality $|ab-ac|=a|b-c|$ for $a\geq 0$, in $(c)$ we use the definition of the L1-norm and the equality $\sum_{s_{h+2:H+1}} \etQ_{h+1:H+1}^{\pi^{t+1}}(s_{h+2:H+1}|s_{h+1})=1$, and in $(d)$ we use the fact that $\gQ_h^{\pi^{t+1}}(\cdot|s_h)$ belongs to $\gC^t(s_h,\pi_h^{t+1}(s_h))$ and the inductive hypothesis for $h+1$.
Since the trivial bound $2$ also applies, we obtain
\[
\normone{\etQ_{h:H+1}^{\pi^{t+1}}(\cdot|s_h)-\widehat\etP_{h:H+1}^{t,\pi^{t+1}}(\cdot|s_h)} \leq \min \left[ 2, \widehat Z_h^t(s_h) \right] = \widehat L_h^t(s_h).
\]
The proof for $\lVert \sQ_{h:H+1}^{\pi^{t+1}}(\cdot|s_h)-\widehat\sP_{h:H+1}^{t,\pi^{t+1}}(\cdot|s_h) \rVert_1$ is very similar:
\begin{small}
\begin{align*}
&\normone{\sQ_{h:H+1}^{\pi^{t+1}}(\cdot|s_h)-\widehat\sP_{h:H+1}^{t,\pi^{t+1}}(\cdot|s_h)} = \sum_{s_{H+1}} \left\vert \sQ_{h:H+1}^{\pi^{t+1}}(s_{H+1}|s_h) - \widehat\sP_{h:H+1}^{t,\pi^{t+1}}(s_{H+1}|s_h) \right\vert\\
 &\leq_{(a)} \sum_{s_{H+1}} \sum_{s_{h+1}} \left\vert \gQ_h^{\pi^{t+1}}(s_{h+1}|s_h)\sQ_{h+1:H+1}^{\pi^{t+1}}(s_{H+1}|s_{h+1}) - \widehat\gP_h^{t,\pi^{t+1}}(s_{h+1}|s_h)\sQ_{h+1:H+1}^{\pi^{t+1}}(s_{H+1}|s_{h+1}) \right\vert\\
 & \hspace*{12pt} + \sum_{s_{H+1}} \sum_{s_{h+1}} \left\vert \widehat\gP_h^{t,\pi^{t+1}}(s_{h+1}|s_h) \sQ_{h+1:H+1}^{\pi^{t+1}}(s_{H+1}|s_{h+1}) - \widehat\gP_h^{t,\pi^{t+1}}(s_{h+1}|s_h) \widehat\sP_{h+1:H+1}^{t,\pi^{t+1}}(s_{H+1}|s_{h+1}) \right\vert\\
 &=_{(b)} \sum_{s_{h+1}} \left\vert \gQ_h^{\pi^{t+1}}(s_{h+1}|s_h) - \widehat\gP_h^{t,\pi^{t+1}}(s_{h+1}|s_h) \right\vert \sum_{s_{H+1}} \sQ_{h+1:H+1}^{\pi^{t+1}}(s_{H+1}|s_{h+1})\\
 & \hspace*{12pt} + \sum_{s_{h+1}} \widehat\gP_h^{t,\pi^{t+1}}(s_{h+1}|s_h) \sum_{s_{H+1}} \left\vert \sQ_{h+1:H+1}^{\pi^{t+1}}(s_{H+1}|s_{h+1}) - \widehat\sP_{h+1:H+1}^{t,\pi^{t+1}}(s_{H+1}|s_{h+1}) \right\vert\\
 &=_{(c)} \normone{ \gQ_h^{\pi^{t+1}}(\cdot|s_h) - \widehat\gP_h^{t,\pi^{t+1}}(\cdot|s_h) } + \sum_{s_{h+1}} \widehat\gP_h^{t,\pi^{t+1}}(\cdot|s_h) \normone{\sQ_{h+1:H+1}^{\pi^{t+1}}(\cdot|s_{h+1})-\sP_{h+1:H+1}^{\pi^{t+1}}(\cdot|s_{h+1})}\\
 &\leq_{(d)} B_h^{t,\pi^{t+1}}(s_h) + \sum_{s_{h+1}} \widehat\gP_h^{t,\pi^{t+1}}(\cdot|s_h) \widehat L_{h+1}^t(s_{h+1}) = \widehat Z_h^t(s_h).
\end{align*}
\end{small}

We next prove by induction on $h$ that $\widehat L_h^t(s) - L_h^t(s) \leq 2L_h^t(s)$.
The base case is $h=H$, in which case $\widehat L_H^t(s)=\min[2,B_H^{t,\pi^{t+1}}(s_H)] = B_H^{t,\pi^{t+1}}(s_H)=L_H^t(s)$, implying $\widehat L_H^t(s) - L_H^t(s) = 0 \leq 2L_H^t(s)$.
The inductive case is $h\in\ival{H-1}$. Due to Lemma~\ref{lemma:alt}, an alternative definition of $L_h^t$ is given by
\[
L_h^t(s) = \sum_{i=h}^H \sum_{s_i} \sP_{h:i}^{\pi^{t+1}}(s_i|s) B_i^{t,\pi^{t+1}}(s_i).
\]
We can now write
\begin{small}
\begin{align*}
&\widehat L_h^t(s) - L_h^t(s) \leq B_h^{t,\pi^{t+1}}(s_h) + \sum_{s'} \widehat\gP_h^{t,\pi^{t+1}}(s'|s) \widehat L_{h+1}^t(s') - B_h^{t,\pi^{t+1}}(s_h) - \sum_{s'} \gP_h^{\pi^{t+1}}(s'|s) L_{h+1}^t(s') \\
 &= \sum_{s'} \left( \widehat\gP_h^{t,\pi^{t+1}}(s'|s) - \gP_h^{\pi^{t+1}}(s'|s) \right) \widehat L_{h+1}^t(s') + \sum_{s'} \gP_h^{\pi^{t+1}}(s'|s) \left( \widehat L_{h+1}^t(s') - L_{h+1}^t(s') \right)\\
 &\leq_{(a)} 2 \normone{ \widehat\gP_h^{t,\pi^{t+1}}(\cdot|s) - \gP_h^{\pi^{t+1}}(\cdot|s) } + 2 \sum_{s'} \gP_h^{\pi^{t+1}}(s'|s) \sum_{i=h+1}^H \sum_{s_i} \sP_{h+1:i}^{\pi^{t+1}}(s_i|s') B_i^{t,\pi^{t+1}}(s_i)\\
 &\leq_{(b)} 2 \sum_{s_i} \sP_{h:h}^{\pi^{t+1}}(s_i|s) B_i^{t,\pi^{t+1}}(s_i) + 2 \sum_{i=h+1}^H \sum_{s_i} \sP_{h:i}^{\pi^{t+1}}(s_i|s) B_i^{t,\pi^{t+1}}(s_i)\\
 &= 2 \sum_{i=h}^H \sum_{s_i} \sP_{h:i}^{\pi^{t+1}}(s_i|s) B_i^{t,\pi^{t+1}}(s_i) = 2L_h^t(s).
\end{align*}
\end{small}
In $(a)$ we use the upper bound $\widehat L_{h+1}^t(s')\leq 2$, the inductive hypothesis and the alternative definition of $L_{h+1}^t(s')$. In $(b)$ we use event $\gE$, the definition $\sP_{h:h}^{\pi^{t+1}}(s_i|s)=\sI(s=s_i)$ and change the order of summation to obtain $\sum_{s'} \gP_h^{\pi^{t+1}}(s'|s) \sP_{h+1:i}^{\pi^{t+1}}(s_i|s') = \sP_{h:i}^{\pi^{t+1}}(s_i|s)$.
Finally we identify the alternative definition of $L_h^t(s)$.
Since $\widehat L_h^t(s) - L_h^t(s) \leq 2L_h^t(s)$, it follows that $\widehat L_h^t(s) = \widehat L_h^t(s) - L_h^t(s) + L_h^t(s) \leq 3L_h^t(s)$. This concludes the proof of the lemma.
\end{proof}

We next prove that BPI-UCRL outputs an $\varepsilon$-optimal policy.
To do so, let $\overline\gP{}_{\hspace*{-1pt}h}^t(\cdot|s,a)=\arg\max_{p\in\gC^t(s,a)} \sum_{s'} p(s') \overline V{}_{\hspace*{-2pt}h+1}^t(s';\gR)$ be the transition kernel that achieves the maximum over $\gC^t(s,a)$ in~\eqref{eq:upperV} for each $h$ and $(s,a)$, and let $\overline\gP{}_{\hspace*{-1pt}h}^{t,\pi}(s'|s)=\overline\gP{}_{\hspace*{-1pt}h}^t(s'|s,\pi_h(s))$ be its transition kernel.

\begin{lemma}\label{lemma:flatsc}
Under event $\gE$, BPI-UCRL returns a policy $\widehat\pi$ such that $V_1^*(s_1;\gR)-V_1^{\widehat\pi}(s_1;\gR)\leq\varepsilon$.
\end{lemma}

\begin{proof}
Since $V_1^*$ is optimal and $\overline V{}_{\hspace*{-2pt}1}^\tau$ is an upper bound on $V_1^*$ under $\gE$, we have $V_1^{\pi^{\tau+1}}(s_1;\gR) \leq V_1^*(s_1;\gR) \leq \overline V{}_{\hspace*{-2pt}1}^\tau(s_1;\gR)$.
Since BPI-UCRL outputs the policy $\widehat\pi=\pi^{\tau+1}$, we can use the alternative definitions of $\overline V{}_{\hspace*{-2pt}h}^\tau$ and $V_h^{\pi^{\tau+1}}$ due to Lemma~\ref{lemma:alt} to bound the difference between $V_1^*$ and $V_1^{\widehat\pi}$ as
\begin{align*}
&\big\vert V_1^*(s_1;\gR) - V_1^{\widehat\pi}(s_1;\gR) \big\vert = \big\vert V_1^*(s_1;\gR) - V_1^{\pi^{\tau+1}}(s_1;\gR) \big\vert \leq \big\vert \overline V{}_{\hspace*{-2pt}1}^\tau(s_1;\gR) - V_1^{\pi^{\tau+1}}(s_1;\gR) \big\vert\\
&\leq \sum_{s_{2:H+1}} \left\vert \overline\etP{}_{1:H+1}^{\tau,\pi^{\tau+1}}(s_{2:H+1}|s_1) - \etP_{1:H+1}^{\pi^{\tau+1}}(s_{2:H+1}|s_1) \right\vert \sum_{i=1}^H \gR(s_i,\pi_i^{\tau+1}(s_i))\\
 &\leq \left[ \normone{ \overline\etP{}_{1:H+1}^{\tau,\pi^{\tau+1}}(\cdot|s_1) - \widehat\etP_{1:H+1}^{t,\pi^{\tau+1}}(\cdot|s_1) } + \normone{ \widehat\etP_{1:H+1}^{\tau,\pi^{\tau+1}}(\cdot|s_1) - \etP_{1:H+1}^{\pi^{\tau+1}}(\cdot|s_1) } \right] H \leq 2\widehat L_1^\tau(s_1) H \leq \varepsilon,
\end{align*}
where we use the alternative definitions of $\overline V{}_{\hspace*{-2pt}h}^\tau$ and $V_h^{\pi^{\tau+1}}$, the triangle inequality, the fact that $\overline\gP{}_h^{\tau,\pi^{\tau+1}}(\cdot|s_h)$ and $\gP_h^{\pi^{\tau+1}}(\cdot|s_h)$ belong to $\gC^\tau(s_h,\pi_h^{\tau+1}(s_h))$ under event $\gE$, Lemma~\ref{lemma:tsbound} and the stopping criterion.
\end{proof}

For each $h$ and $(s,a)$, let $\overline n^t(s,a) = \sum_{h=1}^H\sP_{1:h}^{\pi^t}(s|s_1) \cdot \sI(a=\pi_h^t(s))$ be the expected number of visits of $(s,a)$ in episode $t$, and define pseudo-counts $\overline N{}^t(s,a)=\sum_{\ell=1}^t \overline n^\ell(s,a)$. From Lemma~\ref{lemma:pseudoconv} in Appendix~\ref{app:concentration} (applied to $\bm H =1$) under the event \[\gE^{cnt}=\left\{\forall t \geq 1, N^{t}(s,a) \geq \frac{1}{2}\overline{N}{}^{t}(s,a) - H\log\left(\frac{2SAH}{\delta}\right)\right\},\] 
for all $t$ the confidence bonuses satisfy the inequality
\[
B_h^{t,\pi^{t+1}}(s)  \leq 4 {\sqrt{H}} \sqrt{ \frac {\beta\left(\overline{N}{}^t(s,\pi_h^{t+1}(s)),\frac{\delta}{H}\right)} {\overline{N}{}^t(s,\pi_h^{t+1}(s)) \vee {H}} }.
\]
Assume that the algorithm stops after episode $\tau$. Before stopping, we have $\varepsilon/(2H)\leq \widehat L_1^t(s_1)$ for each $t\in\ival{\tau}$. Under events $\gE$ and $\gE^{cnt}$ we can sum the contributions from each episode to obtain
\begin{align*}
\frac {\tau\varepsilon} {2H} &\leq_{(a)} \sum_{t=1}^\tau \widehat L_1^t(s_1) \leq 3 \sum_{t=1}^\tau L_1^t(s_1) = 3 \sum_{t=1}^\tau \sum_{h=1}^H \sum_{s_h} \sP_{1:h}^{\pi^{t+1}}(s_h|s_1) B_h^{t,\pi^{t+1}}(s_h)\\
 &\leq_{(b)} 12 { \sqrt{H}} \sum_{t=1}^\tau \sum_{h=1}^H \sum_{s_h} \sP_{1:h}^{\pi^{t+1}}(s_h|s_1) \sqrt{ \frac {\beta(\overline N{}^t(s_h,\pi_h^{t+1}(s_h)),\delta/H)} {\overline N{}^t(s_h,\pi_h^{t+1}(s_h))\vee {H}} }\\
 &\leq_{(c)} 12 \sqrt{ H\beta(H\tau,\delta/H) } \sum_{t=1}^\tau \sum_{s,a} \frac { \sum_{s_h} \sum_{h=1}^H \sP_{1:h}^{\pi^{t+1}}(s_h|s_1) \cdot \sI((s,a)=(s_h,\pi_h^{t+1}(s_h))) } {\sqrt {\overline N{}^t(s,a)\vee { H}} } \\
 &=_{(d)} 12 \sqrt{ H\beta(H\tau,\delta/H) } \sum_{t=1}^\tau \sum_{s,a} \frac { \sum_{s_h} \overline n^{t+1}(s_h,a) \cdot \sI(s=s_h) } {\sqrt {\overline N{}^t(s,a)\vee { H}} } \\
 &\leq_{(e)} 12 \sqrt{ H\beta(H\tau,\delta/H) } \sum_{s,a} \sum_{t=1}^\tau \frac {\overline N{}^{t+1}(s,a) - \overline N{}^t(s,a)} {\sqrt {\overline N{}^t(s,a)\vee { H}} }\\
 &\leq_{(f)} 12 \sqrt{ H\beta(H\tau,\delta/H) } \sqrt{H} \sum_{s,a} \sum_{t=1}^\tau \frac {(\overline N{}^{t+1}(s,a) - \overline N{}^t(s,a)) / H} {\sqrt {\overline N{}^t(s,a){ /H \vee 1}} }\\
 &\leq_{(g)} 12 H \sqrt{ \beta(H\tau,\delta/H) SA\tau }.
\end{align*}
In $(a)$ we use Lemma~\ref{lemma:tsbound} and the alternative definition of $L_1^t(s_1)$ due to Lemma~\ref{lemma:alt}. In $(b)$ we convert the counts to pseudo-counts. In $(c)$ we use the monotonicity of $\beta$ to obtain an upper bound $\beta(H\tau,\delta/H)$, and we move the term $1/\sqrt{\overline N{}^t(s,a)\vee 1}$ outside the sum over $s_h$ by summing over $s,a$ and introducing an indicator function $\sI((s,a)=(s_h,\pi_h^{t+1}(s_h)))$. In $(d)$ we identify $\overline n^{t+1}(s_h,a)$. In $(e)$ we simplify the sum over $s_h$ and use the definitions of $\overline N{}^{t+1}(s,a)$ and $\overline N{}^t(s,a)$. In $(f)$ we normalize the pseudo-counts by dividing them with $H$, and in $(g)$ we apply Lemma 19 of~\cite{UCRL10}.

Rearranging the terms of the inequality to isolate $\tau$ on one side gives us
\begin{align*}
\sqrt{\tau} &\leq \frac {24 H^2 \sqrt{ \beta(H\tau,\delta/H) SA } } \varepsilon \quad \Leftrightarrow \quad H\tau \leq \frac {576H \cdot SA H^4 \beta(H\tau,\delta/H)} {\varepsilon^2}.
\end{align*}
Recalling the expression $\beta(n,\delta) =\log (2SA/\delta) + (S-1)\log(e(1 + n/(S-1)))$, the last inequality is equivalent to  
\[H\tau \leq \frac {576H \cdot SA H^4 \left(\log(2SAH/\delta) + (S-1)\log(e(1+H\tau/(S-1)))\right)} {\varepsilon^2}.\]
Using some inversion lemma such as Lemma 15 from \cite{kaufmann2021rf} (which we restate for completeness as Lemma~\ref{lemma:inversion}) yields the bound 
\[
\tau = \gO\left( \frac {SAH^4} {\varepsilon^2} \left( \log \frac {SAH} \delta + S \log \left( \frac {SAH^5} {\varepsilon^2} \log \frac {SAH} \delta \right) \right) \right)\;
\]
that is valid under the event $\gE \cap \gE^{\text{cnt}}$. Moreover, from Lemma~\ref{lem:calibration}, this event holds with probability larger than $1-\delta$, which concludes the proof. 

The proof follows a very similar strcuture compared to that of the original BPI-UCRL algorithm. A subtle difference is that BPI-UCRL is originally analyzed for time-inhomogeneous MDPs. For time-homogeneous transitions, the threshold function $\beta$ in the event $\mathcal{E}$ can be chosen slightly smaller (removing an $H$ inside the log) but the event $\mathcal{E}^{\text{cnt}}$ needs to be defined in a slightly different way compared to the original analysis, with an extra $H$ factor in the correction term involved therein (see Appendix~\ref{app:concentration}). This was actually overlooked in the initial work of \cite{kaufmann2021rf} who claim that an improvement of a factor $H$ can be obtained in the resulting sample complexity bound.

For sparse-reward MDPs, Lemma~\ref{lemma:flatsc} trivially holds for the stopping rule $\widehat L_1^\tau(s_1)=\varepsilon/2$:
\begin{small}
\begin{align*}
\big\vert V_1^*(s_1;\gR) - V_1^{\widehat\pi}(s_1;\gR) \big\vert &\leq \sum_{s_{2:H+1}} \left\vert \overline\etP{}_{1:H+1}^{\tau,\pi^{\tau+1}}(s_{2:H+1}|s_1) - \etP_{1:H+1}^{\pi^{\tau+1}}(s_{2:H+1}|s_1) \right\vert \sum_{i=1}^H \gR(s_i,\pi_i^{\tau+1}(s_i))\\
 &\leq \normone{ \overline\etP{}_{1:H+1}^{\tau,\pi^{\tau+1}}(\cdot|s_1) - \widehat\etP_{1:H+1}^{t,\pi^{\tau+1}}(\cdot|s_1) } + \normone{ \widehat\etP_{1:H+1}^{\tau,\pi^{\tau+1}}(\cdot|s_1) - \etP_{1:H+1}^{\pi^{\tau+1}}(\cdot|s_1) }\\
 &\leq 2\widehat L_1^\tau(s_1) \leq \varepsilon.
\end{align*}
\end{small}
Because of the modified stopping rule, the analysis from the proof of Theorem~\ref{thm:rf} directly yields
\[
\frac {\tau\varepsilon} 2 \leq 12 H \sqrt{ \beta(H\tau,\delta/H) SA\tau } \quad \Leftrightarrow \quad H\tau \leq \frac {576H \cdot SA H^2 \beta(H\tau,\delta/H)} {\varepsilon^2}.
\]
Lemma~\ref{lemma:inversion} now yields a sample complexity of order $\widetilde\gO(SAH^2/\varepsilon^2)$.

\section{The Value Function Upper Bound}\label{app:upper}

In this appendix we show that the value function $\overline{\bm V}{}^t$ is an upper bound on $\bm V^*$.
We prove the statement using two lemmas that bound the transition kernels and rewards, respectively.

\begin{lemma}\label{lemma:diffprob}
Under Assumption~\ref{ass:rew_alt} and event $\gE$, for each $(\sigma,k)$ the transition kernels $\bm\gP^{\pi^*}(\cdot|\sigma,k)$ and $\bm\gP^{\pi^{t+1}}(\cdot|\sigma,k)$ belong to the high-level confidence set $\bm\gC^t(\sigma,k)$.
\end{lemma}

\begin{proof}
To show that $\bm\gP^{\pi^*}(\cdot|\sigma,k)$ belongs to $\bm\gC^t(\sigma,k)$, we bound $\lVert \widehat{\bm\gP}{}^{t,\pi^{t+1}}(\cdot|\sigma,k) - \bm\gP^{\pi^*}(\cdot|\sigma,k) \rVert_1$ as
\begin{align*}
&\normone { \widehat{\bm\gP}^{t,\pi^{t+1}}(\cdot|\sigma,k) - \bm\gP^{\pi^*}(\cdot|\sigma,k) }\\
 &\leq_{(a)} \normone { \widehat{\bm\gP}{}^{t,\pi^{t+1}}(\cdot|\sigma,k) - \bm\gP^{\pi^{t+1}}(\cdot|\sigma,k) } + \normone { \bm\gP^{\pi^{t+1}}(\cdot|\sigma,k) - \bm\gP^{\pi^*}(\cdot|\sigma,k) }\\
 &\leq_{(b)} \normone { \widehat\sP{}_{1:H+1}^{t,\pi^{t+1,k}}(\cdot|g(\sigma)) - \sP_{1:H+1}^{\pi^{t+1,k}}(\cdot|g(\sigma)) } + \normone { \sP_{1:H+1}^{\pi^{t+1,k}}(\cdot|g(\sigma)) - \sP_{1:H+1}^{\pi^{*,k}}(\cdot|g(\sigma)) }\\
 &\leq_{(c)} \widehat L_1^{t,k}(g(\sigma)) + 2 \big\vert V_1^{\pi^{t+1,k}}(g(\sigma);\gR^k) - V_1^{\pi^{*,k}}(g(\sigma);\gR^k) \big\vert\\
 &\leq_{(d)} \widehat L_1^{t,k}(g(\sigma)) + 4H \widehat L_1^{t,k}(g(\sigma)) \leq  5H\widehat L_1^{t,k}(g(\sigma)).
\end{align*}
In $(a)$ we use the triangle inequality, and in $(b)$ we use the definitions of $\widehat{\bm\gP}{}^{t,\pi}$ and $\bm\gP^\pi$. In $(c)$ we apply Lemma~\ref{lemma:tsbound} and use Assumption~\ref{ass:rew_alt}, and in $(d)$ we bound the value difference as in the proof of Lemma~\ref{lemma:flatsc}. Since $\lVert \widehat{\bm\gP}{}^{t,\pi^{t+1}}(\cdot|\sigma,k) - \bm\gP^{\pi^{t+1}}(\cdot|\sigma,k) \rVert_1$ is the first term in step $(a)$, the second claim follows.
\end{proof}

\begin{lemma}\label{lemma:diffrew}
Under Assumption~\ref{ass:rew_alt} and event $\gE$, for each $(\sigma,k)$ it holds that $\big\vert \widehat{\bm\gR}{}^{t,\pi^{t+1}}(\sigma,k) - \bm\gR^{\pi^*}(\sigma,k) \big\vert \leq 3H\widehat L_1^{t,k}(g(\sigma))$ and $\big\vert \widehat{\bm\gR}{}^{t,\pi^{t+1}}(\sigma,k) - \bm\gR^{\pi^{t+1}}(\sigma,k) \big\vert \leq 3H\widehat L_1^{t,k}(g(\sigma))$.
\end{lemma}

\begin{proof}
We can bound the first term as follows:
\begin{small}
\begin{align*}
& \big\vert \widehat{\bm\gR}{}^{t,\pi^{t+1}}(\sigma,k) - \bm\gR^{\pi^*}(\sigma,k) \big\vert \leq_{(a)} \big\vert \widehat{\bm\gR}{}^{t,\pi^{t+1}}(\sigma,k) - \bm\gR^{\pi^{t+1}}(\sigma,k) \big\vert + \big\vert \bm\gR^{\pi^{t+1}}(\sigma,k) - \bm\gR^{\pi^*}(\sigma,k) \big\vert\\
 &=_{(b)} \big\vert \widehat V{}_{\hspace*{-1pt}1}^{t,\pi^{t+1,k}}(g(\sigma);\gY_{\bm g(\sigma)}) - V_1^{\pi^{t+1,k}}(g(\sigma);\gY_{\bm g(\sigma)}) \big\vert + \big\vert V_1^{\pi^{t+1,k}}(g(\sigma);\gY_{\bm g(\sigma)}) - V_1^{\pi^{*,k}}(g(\sigma);\gY_{\bm g(\sigma)}) \big\vert\\
 &\leq_{(c)} \normone { \widehat\etP{}_{1:H+1}^{t,\pi^{t+1,k}}(\cdot|g(\sigma)) - \etP_{1:H+1}^{\pi^{t+1,k}}(\cdot|g(\sigma)) } H + \big\vert V_1^{\pi^{t+1,k}}(g(\sigma);\gR^k) - V_1^{\pi^{*,k}}(g(\sigma);\gR^k) \big\vert\\
 &\leq_{(d)} H\widehat L_1^{t,k}(g(\sigma)) + 2H \widehat L_1^{t,k}(g(\sigma)) \leq 3H\widehat L_1^{t,k}(g(\sigma)).
\end{align*}
\end{small}
In $(a)$ we use the triangle inequality, and in $(b)$ we use the definitions of $\widehat{\bm\gR}{}^{t,\pi}$ and $\bm\gR^\pi$. In $(c)$ we bound the value difference using the composite distributions and use Assumption~\ref{ass:rew_alt}. In $(d)$ we use Lemma~\ref{lemma:tsbound} and bound the value difference as in the proof of Lemma~\ref{lemma:flatsc}. The second statement of the lemma holds since $\big\vert \widehat{\bm\gR}{}^{t,\pi^{t+1}}(\sigma,k) - \bm\gR^{\pi^{t+1}}(\sigma,k) \big\vert$ is the first term of step $(a)$.
\end{proof}

As a consequence of Lemmas~\ref{lemma:diffprob} and~\ref{lemma:diffrew}, $\widehat{\bm\gR}{}^{t,\pi^{t+1}}(\sigma,k) + 3H\widehat L_1^{t,k}(g(\sigma))$ is an upper bound on $\bm\gR^{\pi^*}(\sigma,k)$ for each $(\sigma,k)$, and $\bm\gP^{\pi^*}(\cdot\,|\sigma,\bm\pi_{\bm h}(\sigma))$ belongs to the confidence set $\bm\gC^t(\sigma,\bm\pi_{\bm h}(\sigma))$ for each $\bm\pi$ and $\sigma$.
Since $\overline{\bm V}{}^t$ maximizes over $\bm\gC^t(\sigma,k)$, this implies that $\overline{\bm V}{}^t$ is an upper bound on $\bm V^*$.
Since $\bm V_{\bm h}^*(\sigma)$ is trivially upper bounded by $\bm HH$ for each $\bm h$ and $\sigma$, taking the minimum with $\bm HH$ does not invalidate the bound.

\section{Proof of Theorem~\ref{thm:hrf}}\label{app:hrf}

In this section we prove Theorem~\ref{thm:hrf}, which states an upper bound on the sample complexity of HBPI-UCRL.
We first define an error function $\bm L^t$ analogous to $\overline{\bm L}{}^t$ but for the true SMDP transition kernel of $(\bm\pi^{t+1},\pi^{t+1})$.
The function $\bm L^t$ is recursively defined for each $\sigma$ as $\bm L_{\bm H+1}^t(\sigma)=0$ and
\begin{align*}
\bm L_{\bm h}^t(\sigma) &= \bm B_{\bm h}^t(\sigma) + \sum_{\sigma'} \bm\gP^{\pi^{t+1}}(\sigma'|\sigma,\bm\pi_{\bm h}^{t+1}(\sigma)) \bm L_{\bm h+1}^t(\sigma') \quad \forall\bm h\in\ival{\bm H}.
\end{align*}
We also introduce convenient notation for several transition kernels of the policy $\bm\pi^{t+1}$.
\begin{itemize}
\item $\bm\gP_{\bm h}^t(\sigma'|\sigma)=\bm\gP^{\pi^*}(\sigma'|\sigma,\bm\pi_{\bm h}^{t+1}(\sigma))$ is the true transition kernel of $\bm\pi^{t+1}$ under $\pi^*$ with associated composite distributions $\pmb\etP_{\bm i:\bm j}^t$ and $\pmb\sP_{\bm i:\bm j}^t$.
\item $\widetilde{\bm\gP}{}_{\bm h}^t(\sigma'|\sigma)=\bm\gP^{\pi^{t+1}}(\sigma'|\sigma,\bm\pi_{\bm h}^{t+1}(\sigma))$ is the true transition kernel of $\bm\pi^{t+1}$ under $\pi^{t+1}$ with associated composite distributions $\widetilde{\pmb\etP}{}_{\bm i:\bm j}^t$ and $\widetilde{\pmb\sP}{}_{\bm i:\bm j}^t$.
\item $\widehat{\bm\gP}{}_{\bm h}^t(\sigma'|\sigma)=\widehat{\bm\gP}{}^{t,\pi^{t+1}}(\sigma'|\sigma,\bm\pi_{\bm h}^{t+1}(\sigma))$ is the empirical transition kernel of $\bm\pi^{t+1}$ under $\pi^{t+1}$ with associated composite distributions $\widehat{\pmb\etP}{}_{\bm i:\bm j}^t$ and $\widehat{\pmb\sP}{}_{\bm i:\bm j}^t$.
\end{itemize}

We are now ready to prove a lemma for SMDPs analogous to Lemma~\ref{lemma:tsbound} for MDPs.
\begin{lemma}\label{lemma:hierbound}
Let $\{\bm\gQ_{\bm h}^{\bm\pi^{t+1}}\}_{\bm h\in\ival{\bm H}}$ be any transition kernel for $\bm\pi^{t+1}$ such that $\bm\gQ_{\bm h}^{\bm\pi^{t+1}}(\cdot|\sigma)\in\bm\gC^t(\sigma,\bm\pi^{t+1}(\sigma))$ for $\bm h,\sigma$, and let $\pmb\etQ_{\bm i:\bm j}^{\bm\pi^{t+1}}$ and $\pmb\sQ_{\bm i:\bm j}^{\bm\pi^{t+1}}$ be the induced composite distributions.
Under Assumption~\ref{ass:rew_alt} and event $\gE$ we have
\begin{align*}
\big\lVert \pmb\etQ_{\bm h:\bm H+1}^{\bm\pi^{t+1}}(\cdot|\sigma)-\overline{\pmb\etP}{}_{\bm h:\bm H+1}^t(\cdot|\sigma) \big\rVert_1 &\leq \overline{\bm L}{}_{\bm h}^t(\sigma) \leq 3\bm L_{\bm h}^t(\sigma),\\
\big\lVert \pmb\sQ_{\bm h:\bm H+1}^{\bm\pi^{t+1}}(\cdot|\sigma)-\overline{\pmb\sP}{}_{\bm h:\bm H+1}^t(\cdot|\sigma) \big\rVert_1 &\leq \overline{\bm L}{}_{\bm h}^t(\sigma) \leq 3\bm L_{\bm h}^t(\sigma),
\end{align*}
where $\overline{\pmb\etP}{}_{\bm h:\bm H+1}^t$ and $\,\overline{\pmb\sP}{}_{\bm h:\bm H+1}^t$ are the composite distributions induced by $\overline{\bm\gP}{}_{\hspace*{-2pt}\bm h}^t$.
\end{lemma}

\begin{proof}
We first prove the left inequalities by induction on $\bm h$. The base case is $\bm h=\bm H$. In this case by definition we have
\begin{align*}
\pmb\etQ_{\bm H:\bm H+1}^{\bm\pi^{t+1}}(\cdot|\sigma_{\bm H}) &= \pmb\sQ_{\bm H:\bm H+1}^{\bm\pi^{t+1}}(\cdot|\sigma_{\bm H}) = \bm\gQ_{\bm H}^{\bm\pi^{t+1}}(\cdot|\sigma_{\bm H}),\\
\overline{\pmb\etP}{}_{\bm H:\bm H+1}^t(\cdot|\sigma_{\bm H}) &= \overline{\pmb\sP}{}_{\bm H:\bm H+1}^t(\cdot|\sigma_{\bm H}) = \overline{\bm\gP}{}_{\bm H}^t(\cdot|\sigma_{\bm H}),\\
\overline{\bm L}{}_{\bm H}^t(\sigma_{\bm H}) &= \min \left[ 2, \bm B_{\bm H}^t(\sigma_{\bm H}) + 0 \right].
\end{align*}
Consequently we can write
\begin{align*}
&\normone{ \pmb\etQ_{\bm H:\bm H+1}^{\bm\pi^{t+1}}(\cdot|\sigma_{\bm H}) - \overline{\pmb\etP}{}_{\bm H:\bm H+1}^t(\cdot|\sigma_{\bm H}) } = \normone{ \bm\gQ_{\bm H}^{\bm\pi^{t+1}}(\cdot|\sigma_{\bm H}) - \overline{\bm\gP}{}_{\bm H}^t(\cdot|\sigma_{\bm H}) }\\
 &\leq \normone{ \bm\gQ_{\bm H}^{\bm\pi^{t+1}}(\cdot|\sigma_{\bm H}) - \widehat{\bm\gP}{}_{\bm H}^t(\cdot|\sigma_{\bm H}) } + \normone{ \widehat{\bm\gP}{}_{\bm H}^t(\cdot|\sigma_{\bm H}) - \overline{\bm\gP}{}_{\bm H}^t(\cdot|\sigma_{\bm H}) }\\
 &\leq 5H\widehat L{}_1^{t,\bm\pi_{\bm H}^{t+1}(\sigma_{\bm H})}(g(\sigma_{\bm H})) + 5H\widehat L{}_1^{t,\bm\pi_{\bm H}^{t+1}(\sigma_{\bm H})}(g(\sigma_{\bm H})) = 10H\widehat L{}_1^{t,\bm\pi_{\bm H}^{t+1}(\sigma_{\bm H})}(g(\sigma_{\bm H}) = \bm B_{\bm H}^t(\sigma_{\bm H}),
\end{align*}
where we use the triangle inequality and the facts that $\bm\gQ_{\bm H}^{\bm\pi^{t+1}}(\cdot|\sigma_{\bm H})$ and $\overline{\bm\gP}{}_{\bm H}^t(\cdot|\sigma_{\bm H})$ belong to $\bm\gC^t(\sigma_{\bm H},\bm\pi_{\bm H}^{t+1}(\sigma_{\bm H}))$.
Since the trivial bound $2$ also applies, this shows that $\lVert \pmb\etQ_{\bm H:\bm H+1}^{\bm\pi^{t+1}}(\cdot|\sigma_{\bm H}) - \overline{\pmb\etP}{}_{\bm H:\bm H+1}^t(\cdot|\sigma_{\bm H}) \rVert_1 \leq \min[2,\bm B_{\bm H}^t(\sigma_{\bm H})] = \overline{\bm L}{}_{\bm H}^t(\sigma_{\bm H})$. The proof for $\lVert \pmb\sQ_{\bm H:\bm H+1}^{\bm\pi^{t+1}}(\cdot|\sigma_{\bm H}) - \overline{\pmb\sP}{}_{\bm H:\bm H+1}^t(\cdot|\sigma_{\bm H}) \rVert_1$ is identical.

The inductive case is given by $\bm h\in\ival{\bm H-1}$. In this case we have
\begin{small}
\begin{align*}
&\normone{ \pmb\etQ_{\bm h:\bm H+1}^{\bm\pi^{t+1}}(\cdot|\sigma_{\bm h}) - \overline{\pmb\etP}{}_{\bm h:\bm H+1}^t(\cdot|\sigma_{\bm h}) } = \hspace*{-8pt} \sum_{\sigma_{\bm h+1:\bm H+1}} \hspace*{-8pt} \left\vert \pmb\etQ_{\bm h:\bm H+1}^{\bm\pi^{t+1}}(\sigma_{\bm h+1:\bm H+1}|\sigma_{\bm h}) - \overline{\pmb\etP}{}_{\bm h:\bm H+1}^t(\sigma_{\bm h+1:\bm H+1}|\sigma_{\bm h}) \right\vert\\
 &\leq_{(a)} \hspace*{-9pt} \sum_{\sigma_{\bm h+1:\bm H+1}} \hspace*{-9pt} \left\vert \bm\gQ_{\bm h}^{\bm\pi^{t+1}}(\sigma_{\bm h+1}|\sigma_{\bm h}) \pmb\etQ_{\bm h+1:\bm H+1}^{\bm\pi^{t+1}}(\sigma_{\bm h+2:\bm H+1}|\sigma_{\bm h+1}) - \overline{\bm\gP}{}_{\hspace*{-2pt}\bm h}^t(\sigma_{\bm h+1}|\sigma_{\bm h}) \pmb\etQ_{\bm h+1:\bm H+1}^{\bm\pi^{t+1}}(\sigma_{\bm h+2:\bm H+1}|\sigma_{\bm h+1}) \right\vert\\
 & \hspace*{12pt} + \hspace*{-9pt} \sum_{\sigma_{\bm h+1:\bm H+1}} \hspace*{-9pt} \left\vert \overline{\bm\gP}{}_{\hspace*{-2pt}\bm h}^t(\sigma_{\bm h+1}|\sigma_{\bm h}) \pmb\etQ_{\bm h+1:\bm H+1}^{\bm\pi^{t+1}}(\sigma_{\bm h+2:\bm H+1}|\sigma_{\bm h+1}) - \overline{\bm\gP}{}_{\hspace*{-2pt}\bm h}^t(\sigma_{\bm h+1}|\sigma_{\bm h}) \overline{\pmb\etP}{}_{\bm h+1:\bm H+1}^t(\sigma_{\bm h+2:\bm H+1}|\sigma_{\bm h+1}) \right\vert\\ 
 &=_{(b)} \sum_{\sigma_{\bm h+1}} \left\vert \bm\gQ_{\bm h}^{\bm\pi^{t+1}}(\sigma_{\bm h+1}|\sigma_{\bm h}) - \overline{\bm\gP}{}_{\hspace*{-2pt}\bm h}^t(\sigma_{\bm h+1}|\sigma_{\bm h}) \right\vert \sum_{\sigma_{\bm h+2:\bm H+1}} \hspace*{-9pt} \pmb\etQ_{\bm h+1:\bm H+1}^{\bm\pi^{t+1}}(\sigma_{\bm h+2:\bm H+1}|\sigma_{\bm h+1})\\ 
 & \hspace*{11pt} + \sum_{\sigma_{\bm h+1}} \overline{\bm\gP}{}_{\hspace*{-2pt}\bm h}^t(\sigma_{\bm h+1}|\sigma_{\bm h}) \hspace*{-9pt} \sum_{\sigma_{\bm h+2:\bm H+1}} \hspace*{-9pt} \left\vert \pmb\etQ_{\bm h+1:\bm H+1}^{\bm\pi^{t+1}}(\sigma_{\bm h+2:\bm H+1}|\sigma_{\bm h+1}) - \overline{\pmb\etP}{}_{\bm h+1:\bm H+1}^t(\sigma_{\bm h+2:\bm H+1}|\sigma_{\bm h+1}) \right\vert\\
 &=_{(c)} \normone{ \bm\gQ_{\bm h}^{\bm\pi^{t+1}}(\cdot|\sigma_{\bm h}) - \overline{\bm\gP}{}_{\hspace*{-2pt}\bm h}^t(\cdot|\sigma_{\bm h}) } +
  \sum_{\sigma_{\bm h+1}} \overline{\bm\gP}{}_{\hspace*{-2pt}\bm h}^t(\sigma_{\bm h+1}|\sigma_{\bm h}) \normone{ \pmb\etQ_{\bm h+1:\bm H+1}^{\bm\pi^{t+1}}(\cdot|\sigma_{\bm h+1}) - \overline{\pmb\etP}{}_{\bm h+1:\bm H+1}^t(\cdot|\sigma_{\bm h+1}) }\\
 &\leq_{(d)} \bm B_{\bm h}^t(\sigma_{\bm h}) + \sum_{\sigma_{\bm h+1}} \overline{\bm\gP}{}_{\hspace*{.2pt}\bm h}^t(\sigma_{\bm h+1}|\sigma_{\bm h}) \overline{\bm L}{}_{\bm h+1}^t(\sigma_{\bm h+1}) \equiv \overline{\bm Z}{}_{\bm h}^t(\sigma_{\bm h}).
\end{align*}
\end{small}
In $(a)$ we use the definition of $\pmb\etQ_{\bm h:\bm H+1}^{\bm\pi^{t+1}}$ and $\overline{\pmb\etP}{}_{\bm h:\bm H+1}^t$ and the triangle inequality, in $(b)$ we use the equality $|ab-ac|=a|b-c|$ for $a\geq 0$, in $(c)$ we use the definition of the L1-norm and the equality $\sum_{\sigma_{\bm h+2:\bm H+1}} \pmb\etQ_{\bm h+1:\bm H+1}^{\bm\pi^{t+1}}(\sigma_{\bm h+2:\bm H+1}|\sigma_{\bm h+1})=1$, and in $(d)$ we use the fact that $\bm\gQ_{\bm h}^{\bm\pi^{t+1}}(\cdot|\sigma_{\bm h})$ and $\overline{\bm\gP}{}_{\bm h}^t(\cdot|\sigma_{\bm h})$ belong to $\bm\gC^t(\sigma_{\bm h},\bm\pi_{\bm h}^{t+1}(\sigma_{\bm h}))$ and the inductive hypothesis for $\bm h+1$.
Since the trivial bound $2$ also applies, we obtain
\[
\normone{ \pmb\etQ_{\bm h:\bm H+1}^{\bm\pi^{t+1}}(\cdot|\sigma_{\bm h}) - \overline{\pmb\etP}{}_{\bm h:\bm H+1}^t(\cdot|\sigma_{\bm h}) } \leq \min \left[ 2, \overline{\bm Z}{}_{\bm h}^t(\sigma_{\bm h}) \right] = \overline{\bm L}{}_{\bm h}^t(\sigma_{\bm h}).
\]
The proof for $\lVert \pmb\sQ_{\bm h:\bm H+1}^{\bm\pi^{t+1}}(\cdot|\sigma_{\bm h}) - \overline{\pmb\sP}{}_{\bm h:\bm H+1}^t(\cdot|\sigma_{\bm h}) \rVert_1$ is very similar:
\begin{small}
\begin{align*}
&\normone{ \pmb\sQ_{\bm h:\bm H+1}^{\bm\pi^{t+1}}(\cdot|\sigma_{\bm h}) - \overline{\pmb\sP}{}_{\bm h:\bm H+1}^t(\cdot|\sigma_{\bm h}) } = \sum_{\sigma_{\bm H+1}} \left\vert \pmb\sQ_{\bm h:\bm H+1}^{\bm\pi^{t+1}}(\sigma_{\bm H+1}|\sigma_{\bm h}) - \overline{\pmb\sP}{}_{\bm h:\bm H+1}^t(\sigma_{\bm H+1}|\sigma_{\bm h}) \right\vert\\
 &\leq_{(a)} \sum_{\sigma_{\bm H+1}} \sum_{\sigma_{\bm h+1}} \left\vert \bm\gQ_{\bm h}^{\bm\pi^{t+1}}(\sigma_{\bm h+1}|\sigma_{\bm h}) \pmb\sQ_{\bm h+1:\bm H+1}^{\bm\pi^{t+1}}(\sigma_{\bm H+1}|\sigma_{\bm h+1}) - \overline{\bm\gP}{}_{\hspace*{-2pt}\bm h}^t(\sigma_{\bm h+1}|\sigma_{\bm h}) \pmb\sQ_{\bm h+1:\bm H+1}^{\bm\pi^{t+1}}(\sigma_{\bm H+1}|\sigma_{\bm h+1}) \right\vert\\
 & \hspace*{12pt} + \sum_{\sigma_{\bm H+1}} \sum_{\sigma_{\bm h+1}}  \left\vert \overline{\bm\gP}{}_{\hspace*{-2pt}\bm h}^t(\sigma_{\bm h+1}|\sigma_{\bm h}) \pmb\sQ_{\bm h+1:\bm H+1}^{\bm\pi^{t+1}}(\sigma_{\bm H+1}|\sigma_{\bm h+1}) - \overline{\bm\gP}{}_{\hspace*{-2pt}\bm h}^t(\sigma_{\bm h+1}|\sigma_{\bm h}) \overline{\pmb\sP}{}_{\bm h+1:\bm H+1}^t(\sigma_{\bm H+1}|\sigma_{\bm h+1}) \right\vert\\ 
 &=_{(b)} \sum_{\sigma_{\bm h+1}} \left\vert \bm\gQ_{\bm h}^{\bm\pi^{t+1}}(\sigma_{\bm h+1}|\sigma_{\bm h}) - \overline{\bm\gP}{}_{\hspace*{-2pt}\bm h}^t(\sigma_{\bm h+1}|\sigma_{\bm h}) \right\vert \sum_{\sigma_{\bm H+1}} \pmb\sQ_{\bm h+1:\bm H+1}^{\bm\pi^{t+1}}(\sigma_{\bm H+1}|\sigma_{\bm h+1})\\ 
 & \hspace*{11pt} + \sum_{\sigma_{\bm h+1}} \overline{\bm\gP}{}_{\hspace*{-2pt}\bm h}^t(\sigma_{\bm h+1}|\sigma_{\bm h}) \sum_{\sigma_{\bm H+1}} \left\vert \pmb\sQ_{\bm h+1:\bm H+1}^{\bm\pi^{t+1}}(\sigma_{\bm H+1}|\sigma_{\bm h+1}) - \overline{\pmb\sP}{}_{\bm h+1:\bm H+1}^t(\sigma_{\bm H+1}|\sigma_{\bm h+1}) \right\vert\\
 &=_{(c)} \normone{ \bm\gQ_{\bm h}^{\bm\pi^{t+1}}(\cdot|\sigma_{\bm h}) - \overline{\bm\gP}{}_{\hspace*{-2pt}\bm h}^t(\cdot|\sigma_{\bm h}) } +
  \sum_{\sigma_{\bm h+1}} \overline{\bm\gP}{}_{\hspace*{-2pt}\bm h}^t(\sigma_{\bm h+1}|\sigma_{\bm h}) \normone{ \pmb\sQ_{\bm h+1:\bm H+1}^{\bm\pi^{t+1}}(\cdot|\sigma_{\bm h+1}) - \overline{\pmb\sP}{}_{\bm h+1:\bm H+1}^t(\cdot|\sigma_{\bm h+1}) }\\
 &\leq_{(d)} \bm B_{\bm h}^t(\sigma_{\bm h}) + \sum_{\sigma_{\bm h+1}} \overline{\bm\gP}{}_{\hspace*{.2pt}\bm h}^t(\sigma_{\bm h+1}|\sigma_{\bm h}) \overline{\bm L}_{\bm h+1}^t(\sigma_{\bm h+1}) \equiv \overline{\bm Z}_{\bm h}^t(\sigma_{\bm h}).
\end{align*}
\end{small}

We next show by induction on $\bm h$ that $\overline{\bm L}{}_{\bm h}^t(\sigma) - \bm L_{\bm h}^t(\sigma) \leq 2\bm L_{\bm h}^t(\sigma)$.
The base case is given by $\bm h=\bm H$, in which case $\overline{\bm L}{}_{\bm H}^t(\sigma) = \min[2,\bm B_{\bm H}^t(\sigma)] \leq \bm B_{\bm H}^t(\sigma) = \bm L_{\bm H}^t(\sigma)$ by definition, implying $\overline{\bm L}{}_{\bm H}^t(\sigma) - \bm L_{\bm H}^t(\sigma) \leq 0 \leq 2\bm L_{\bm H}^t(\sigma)$.
In the inductive case $\bm h\in\ival{\bm H-1}$ we can write
\begin{small}
\begin{align*}
&\overline{\bm L}{}_{\bm h}^t(\sigma) - \bm L_{\bm h}^t(\sigma) \leq \bm B_{\bm H}^t(\sigma) + \sum_{\sigma'} \overline{\bm\gP}{}_{\bm h}^t(\sigma'|\sigma) \overline{\bm L}{}_{\bm h+1}^t(\sigma') - \bm B_{\bm H}^t(\sigma) - \sum_{\sigma'} \widetilde{\bm\gP}{}_{\bm h}^t(\sigma'|\sigma) \bm L_{\bm h+1}^t(\sigma') \\
 &= \sum_{\sigma'} \left( \overline{\bm\gP}{}_{\bm h}^t(\sigma'|\sigma) - \widetilde{\bm\gP}{}_{\bm h}^t(\sigma'|\sigma) \right) \overline{\bm L}{}_{\bm h+1}^t(\sigma') + \sum_{\sigma'} \widetilde{\bm\gP}{}_{\bm h}^t(\sigma'|\sigma) \left( \overline{\bm L}{}_{\bm h+1}^t(\sigma') - \bm L_{\bm h+1}^t(\sigma') \right)\\
 &\leq_{(a)} 2 \normone{ \overline{\bm\gP}{}_{\bm h}^t(\cdot|\sigma) - \widehat{\bm\gP}{}_{\bm h}^t(\cdot|\sigma) } \hspace*{-5pt} + 2 \normone{ \widehat{\bm\gP}{}_{\bm h}^t(\cdot|\sigma) - \widetilde{\bm\gP}{}_{\bm h}^t(\cdot|\sigma) } \hspace*{-5pt} + 2 \sum_{\sigma'} \widetilde{\bm\gP}{}_{\bm h}^t(\sigma'|\sigma) \hspace*{-3pt} \sum_{\bm i=\bm h+1}^{\bm H} \sum_{\sigma_{\bm i}} \widetilde{\pmb\sP}{}_{\bm h+1:\bm i}^t(\sigma_{\bm i}|\sigma') \bm B_{\bm i}^t(\sigma_{\bm i})\\
 &\leq_{(b)} 2 \sum_{\sigma_{\bm i}} \widetilde{\pmb\sP}{}_{\bm h:\bm h}^t(\sigma_{\bm i}|\sigma) \bm B_{\bm i}^t(\sigma_{\bm i}) + 2 \sum_{\bm i=\bm h+1}^{\bm H} \sum_{\sigma_{\bm i}} \widetilde{\pmb\sP}{}_{\bm h:\bm i}^t(\sigma_{\bm i}|\sigma) \bm B_{\bm i}^t(\sigma_{\bm i}) = 2 \sum_{\bm i=\bm h}^{\bm H} \sum_{\sigma_{\bm i}} \widetilde{\pmb\sP}{}_{\bm h:\bm i}^t(\sigma_{\bm i}|\sigma) \bm B_{\bm i}^t(\sigma_{\bm i}) = 2\bm L_{\bm h}^t(\sigma).
\end{align*}
\end{small}
In $(a)$ we use the trivial bound $\overline{\bm L}{}_{\bm h+1}^t(\sigma')\leq 2$, the triangle inequality, the inductive hypothesis and the alternative definition of $\bm L_{\bm h+1}^t(\sigma')$ due to Lemma~\ref{lemma:alt}. In $(b)$ we use Lemma~\ref{lemma:diffprob} twice, the definition $\widetilde{\pmb\sP}{}_{\bm h:\bm h}^t(\sigma_{\bm i}|\sigma)=\sI(\sigma=\sigma_{\bm i})$ and change the order of summation to obtain $\sum_{\sigma'} \widetilde{\bm\gP}{}_{\bm h}^t(\sigma'|\sigma) \widetilde{\pmb\sP}{}_{\bm h+1:\bm i}^t(\sigma_{\bm i}|\sigma') = \widetilde{\pmb\sP}{}_{\bm h:\bm i}^t(\sigma_{\bm i}|\sigma)$. Since $\overline{\bm L}{}_{\bm h}^t(\sigma) - \bm L_{\bm h}^t(\sigma) \leq 2\bm L_{\bm h}^t(\sigma)$, it follows that $\overline{\bm L}{}_{\bm h}^t(\sigma) = \overline{\bm L}{}_{\bm h}^t(\sigma) - \bm L_{\bm h}^t(\sigma) + \bm L_{\bm h}^t(\sigma) \leq 3\bm L_{\bm h}^t(\sigma)$. This concludes the proof of the lemma.
\end{proof}

We next prove a lemma that bounds two particular choices of value functions.
\begin{lemma}\label{lemma:altval}
Let $\overline{\bm U}{}^t$ and $\overline{\bm W}{}^t$ be two value functions of $\bm\pi^{t+1}$ recursively defined for each $\sigma$ as $\overline{\bm U}{}_{\hspace*{-2pt}\bm H+1}^t(\sigma)=\overline{\bm W}{}_{\hspace*{-2pt}\bm H+1}^t(\sigma)=0$ and
\begin{align*}
\overline{\bm U}{}_{\hspace*{-2pt}\bm h}^t(\sigma) &= \big\vert \widehat{\bm\gR}{}^{t,\pi^{t+1}}(\sigma,\bm\pi_{\bm h}^{t+1}(\sigma)) - \bm\gR^{\pi^*}(\sigma,\bm\pi_{\bm h}^{t+1}(\sigma)) \big\vert + \sum_{\sigma'} \overline{\bm\gP}{}_{\bm h}^t(\sigma'|\sigma) \overline{\bm U}{}_{\hspace*{-2pt}\bm h+1}^t(\sigma'),\\
\overline{\bm W}{}_{\hspace*{-2pt}\bm h}^t(\sigma) &= \min \Big[ 1, \, 3H\widehat L{}_1^{t,\bm\pi_{\bm h}^{t+1}(\sigma)}(g(\sigma)) + \sum_{\sigma'} \overline{\bm\gP}{}_{\bm h}^t(\sigma'|\sigma) \overline{\bm W}{}_{\hspace*{-2pt}\bm h+1}^t(\sigma') \Big].
\end{align*}
Under Assumption~\ref{ass:rew_alt} and event $\gE$, for each $\bm h$ and $\sigma$ it holds that $\overline{\bm U}{}_{\hspace*{-2pt}\bm h}^t(\sigma) \leq \bm HH\overline{\bm L}{}_{\bm h}^t(\sigma)$ and $\overline{\bm W}{}_{\hspace*{-2pt}\bm h}^t(\sigma) \leq \overline{\bm L}{}_{\bm h}^t(\sigma)$.
\end{lemma}

\begin{proof}
We prove the claim by induction on $\bm h$. In the base case $\bm h=\bm H$ we have
\begin{align*}
\overline{\bm U}{}_{\hspace*{-2pt}\bm H}^t(\sigma) &= \big\vert \widehat{\bm\gR}{}^{t,\pi^{t+1}}(\sigma,\bm\pi_{\bm H}^{t+1}(\sigma)) - \bm\gR^{\pi^*}(\sigma,\bm\pi_{\bm H}^{t+1}(\sigma)) \big\vert + 0 \leq 3H\widehat L_1^{t,\bm\pi_{\bm H}^{t+1}(\sigma)}(g(\sigma)) \leq \bm B_{\bm H}^t(\sigma),\\
\overline{\bm W}{}_{\hspace*{-2pt}\bm H}^t(\sigma) &= \min \left[1, 3H\widehat L_1^{t,\bm\pi_{\bm H}^{t+1}(\sigma)}(g(\sigma)) + 0 \right] \leq \min \left[ 2, B_{\bm H}^t(\sigma) \right] = \overline{\bm L}{}_{\bm H}^t(\sigma),
\end{align*}
where we have used Lemma~\ref{lemma:diffrew}. Since $\overline{\bm U}{}_{\hspace*{-2pt}\bm H}^t(\sigma)$ is trivially upper bounded by $\bm HH$, we have
\[
\overline{\bm U}{}_{\hspace*{-2pt}\bm H}^t(\sigma) \leq \min \left[ \bm HH,\bm B_{\bm H}^t(\sigma) \right] \leq \bm HH \min[2, \bm B_{\bm H}^t(\sigma)] = \bm HH \overline{\bm L}{}_{\bm H}^t(\sigma).
\]
The inductive case is given by $\bm h\in\ival{\bm H-1}$. In this case we have
\begin{align*}
\overline{\bm U}{}_{\hspace*{-2pt}\bm h}^t(\sigma) &= \big\vert \widehat{\bm\gR}{}^{t,\pi^{t+1}}(\sigma,\bm\pi_{\bm h}^{t+1}(\sigma)) - \bm\gR^{\pi^*}(\sigma,\bm\pi_{\bm h}^{t+1}(\sigma)) \big\vert + \sum_{\sigma'} \overline{\bm\gP}{}_{\bm h}^t(\sigma'|\sigma) \overline{\bm U}{}_{\hspace*{-2pt}\bm h+1}^t(\sigma')\\
&\leq \bm B_{\bm h}^t(\sigma) + \bm HH \sum_{\sigma'} \overline{\bm\gP}{}_{\bm h}^t(\sigma'|\sigma) \overline{\bm L}{}_{\bm h+1}^t(\sigma'),\\
\overline{\bm W}{}_{\hspace*{-2pt}\bm h}^t(\sigma) &= \min \left[ 1, 3H\widehat L_1^{t,\bm\pi_{\bm h}^{t+1}(\sigma)}(g(\sigma)) + \sum_{\sigma'} \overline{\bm\gP}{}_{\bm h}^t(\sigma'|\sigma) \overline{\bm W}{}_{\hspace*{-2pt}\bm h+1}^t(\sigma') \right]\\
 &\leq \min \left[ 2, \bm B_{\bm h}^t(\sigma) + \sum_{\sigma'} \overline{\bm\gP}{}_{\bm h}^t(\sigma'|\sigma) \overline{\bm L}{}_{\bm h+1}^t(\sigma') \right] = \overline{\bm L}{}_{\bm h}^t(\sigma),
\end{align*}
where we have used Lemma~\ref{lemma:diffrew} and the inductive hypothesis. Since $\overline{\bm U}{}_{\hspace*{-2pt}\bm h}^t(\sigma)$ is upper bounded by $\bm HH$, we have
\begin{align*}
\overline{\bm U}{}_{\hspace*{-2pt}\bm h}^t(\sigma) &\leq \min \left[ \bm HH,\bm B_{\bm h}^t(\sigma) + \bm HH\sum_{\sigma'} \overline{\bm\gP}{}_{\bm h}^t(\sigma'|\sigma) \overline{\bm L}{}_{\bm h+1}^t(\sigma') \right]\\
 &\leq \bm HH\min \left[ 2,\bm B_{\bm h}^t(\sigma) + \sum_{\sigma'} \overline{\bm\gP}{}_{\bm h}^t(\sigma'|\sigma) \overline{\bm L}{}_{\bm h+1}^t(\sigma') \right] = \bm HH\overline{\bm L}{}_{\bm h}^t(\sigma).
\end{align*}
This concludes the proof of the lemma.
\end{proof}

Since $\big\vert \bm\gR^{\pi^{t+1}}(\sigma_{\bm i},\bm\pi_{\bm i}^{t+1}(\sigma_{\bm i})) - \bm\gR^{\pi^*}(\sigma_{\bm i},\bm\pi_{\bm i}^{t+1}(\sigma_{\bm i}) \big\vert \leq 3H\widehat L_1^{t,\bm\pi_{\bm H}^{t+1}(\sigma)}(g(\sigma)) \leq \bm B_{\bm h}^t(\sigma)$ due to the proof of Lemma~\ref{lemma:diffrew}, replacing $\big\vert \widehat{\bm\gR}{}^{t,\pi^{t+1}}(\sigma_{\bm i},\bm\pi_{\bm i}^{t+1}(\sigma_{\bm i})) - \bm\gR^{\pi^*}(\sigma_{\bm i},\bm\pi_{\bm i}^{t+1}(\sigma_{\bm i}) \big\vert$ with $\big\vert \bm\gR^{\pi^{t+1}}(\sigma_{\bm i},\bm\pi_{\bm i}^{t+1}(\sigma_{\bm i})) - \bm\gR^{\pi^*}(\sigma_{\bm i},\bm\pi_{\bm i}^{t+1}(\sigma_{\bm i}) \big\vert$ in the definition of $\overline{\bm U}{}_{\hspace*{-2pt}\bm h}^t(\sigma)$ yields the same bound.

We next prove an upper bound on the value function $\overline{\bm V}{}^t$.

\begin{lemma}\label{lemma:upperupperV}
Let $\widehat{\bm V}{}^{t,\bm\pi^{t+1},\pi^{t+1}}$ be a value function recursively defined for $\sigma$ as $\widehat{\bm V}{}_{\hspace*{-2pt}\bm H+1}^{t,\bm\pi^{t+1},\pi^{t+1}}(\sigma)=0$ and
\[
\widehat{\bm V}{}_{\hspace*{-2pt}\bm h}^{t,\bm\pi^{t+1},\pi^{t+1}}(\sigma) = \widehat{\bm\gR}{}^{t,\pi^{t+1}}(\sigma,\bm\pi_{\bm h}^{t+1}(\sigma)) + \sum_{\sigma'} \overline{\bm\gP}{}_{\bm h}^t(\sigma'|\sigma) \widehat{\bm V}{}_{\hspace*{-2pt}\bm h+1}^{t,\bm\pi^{t+1},\pi^{t+1}}(\sigma').
\]
For each $\bm h$ and $\sigma$ it holds that $\overline{\bm V}{}_{\hspace*{-2pt}\bm h}^t(\sigma) \leq \bm HH\overline{W}{}_{\hspace*{-2pt}\bm h}^t(\sigma) + \widehat{\bm V}{}_{\hspace*{-2pt}\bm h}^{t,\bm\pi^{t+1},\pi^{t+1}}(\sigma)$.
\end{lemma}

\begin{proof}
By induction on $\bm h$, with the base case given by $\bm h=\bm H$. In this case we have
\begin{align*}
\overline{\bm V}{}_{\hspace*{-2pt}\bm H}^t(\sigma) &= \min \left[ \bm HH, \widehat{\bm\gR}{}^{t,\pi^{t+1}}(\sigma,\bm\pi_{\bm H}^{t+1}(\sigma)) + 3H\widehat L_1^{t,\bm\pi_{\bm H}^{t+1}(\sigma)}(g(\sigma)) \right]\\
 &\leq \min \left[ \bm HH, \bm HH 3H\widehat L_1^{t,\bm\pi_{\bm H}^{t+1}(\sigma)}(g(\sigma)) \right] + \widehat{\bm\gR}{}^{t,\pi^{t+1}}(\sigma,\bm\pi_{\bm H}^{t+1}(\sigma))\\
 &= \bm HH\overline{W}{}_{\hspace*{-2pt}\bm H}^t(\sigma) + \widehat{\bm V}{}_{\hspace*{-2pt}\bm H}^{t,\bm\pi^{t+1},\pi^{t+1}}(\sigma).
\end{align*}
The recursive case is given by $\bm h\in\ival{\bm H-1}$. In this case we have
\begin{align*}
&\overline{\bm V}{}_{\hspace*{-2pt}\bm h}^t(\sigma) = \min \Big[ \bm HH, \widehat{\bm\gR}{}^{t,\pi^{t+1}}(\sigma,\bm\pi_{\bm h}^{t+1}(\sigma)) + 3H\widehat L_1^{t,\bm\pi_{\bm h}^{t+1}(\sigma)}(g(\sigma)) + \sum_{\sigma'} \overline{\bm\gP}{}_{\hspace*{-2pt}\bm h}^t(\sigma'|\sigma) \overline{\bm V}{}_{\hspace*{-2pt}\bm h+1}^t(\sigma') \Big]\\
 &\leq \min \Big[ \bm HH, \widehat{\bm\gR}{}^{t,\pi^{t+1}}(\sigma,\bm\pi_{\bm h}^{t+1}(\sigma)) + 3H\widehat L_1^{t,\bm\pi_{\bm h}^{t+1}(\sigma)}(g(\sigma)) \\
 & \hspace*{60pt} + \sum_{\sigma'} \overline{\bm\gP}{}_{\hspace*{-2pt}\bm h}^t(\sigma'|\sigma) \left( \bm HH\overline{W}{}_{\hspace*{-2pt}\bm h+1}^t(\sigma') + \widehat{\bm V}{}_{\hspace*{-2pt}\bm h+1}^{t,\bm\pi^{t+1},\pi^{t+1}}(\sigma') \right) \Big]\\
 &\leq \min \Big[ \bm HH, \bm HH 3H\widehat L_1^{t,\bm\pi_{\bm h}^{t+1}(\sigma)}(g(\sigma)) + \bm HH\sum_{\sigma'} \overline{\bm\gP}{}_{\hspace*{-2pt}\bm h}^t(\sigma'|\sigma) \overline{W}{}_{\hspace*{-2pt}\bm h+1}^t(\sigma') \Big]\\
 & \hspace*{10pt} + \widehat{\bm\gR}{}^{t,\pi^{t+1}}(\sigma,\bm\pi_{\bm h}^{t+1}(\sigma)) + \sum_{\sigma'} \overline{\bm\gP}{}_{\hspace*{-2pt}\bm h}^t(\sigma'|\sigma) \widehat{\bm V}{}_{\hspace*{-2pt}\bm h+1}^{t,\bm\pi^{t+1},\pi^{t+1}}(\sigma') = \bm HH\overline{W}{}_{\hspace*{-2pt}\bm h}^t(\sigma) + \widehat{\bm V}{}_{\hspace*{-2pt}\bm h}^{t,\bm\pi^{t+1},\pi^{t+1}}(\sigma),
\end{align*}
where we have used the inductive hypothesis.
This concludes the proof.
\end{proof}

We note that Lemma~\ref{lemma:alt} applies to each transition kernel of $\bm\pi^{t+1}$ and associated composite distributions, which implies that we can write the value functions of $\bm\pi^{t+1}$ as
\begin{align*}
\bm V_{\bm h}^{\bm\pi^{t+1},\pi^*}(\sigma_{\bm h}) &= \hspace*{-10pt} \sum_{\sigma_{\bm h+1:\bm H+1}} \hspace*{-10pt} \pmb\etP_{\bm h:\bm H+1}^t(\sigma_{\bm h+1:\bm H+1}|\sigma_{\bm h}) \sum_{\bm i=\bm h}^{\bm H} \bm\gR^{\pi^*}(\sigma_{\bm i},\bm\pi_{\bm i}^{t+1}(\sigma_{\bm i})),\\
\bm V_{\bm h}^{\bm\pi^{t+1},\pi^{t+1}}(\sigma_{\bm h}) &= \hspace*{-10pt} \sum_{\sigma_{\bm h+1:\bm H+1}} \hspace*{-10pt} \widetilde{\pmb\etP}{}_{\bm h:\bm H+1}^t(\sigma_{\bm h+1:\bm H+1}|\sigma_{\bm h}) \sum_{\bm i=\bm h}^{\bm H} \bm\gR^{\pi^{t+1}}(\sigma_{\bm i},\bm\pi_{\bm i}^{t+1}(\sigma_{\bm i})),\\
\widehat{\bm V}{}_{\hspace*{-2pt}\bm h}^{t,\bm\pi^{t+1},\pi^{t+1}}(\sigma_{\bm h}) &= \hspace*{-10pt} \sum_{\sigma_{\bm h+1:\bm H+1}} \hspace*{-10pt} \overline{\pmb\etP}{}_{\bm h:\bm H+1}^t(\sigma_{\bm h+1:\bm H+1}|\sigma_{\bm h}) \sum_{\bm i=\bm h}^{\bm H} \widehat{\bm\gR}{}^{t,\pi^{t+1}}(\sigma_{\bm i},\bm\pi_{\bm i}^{t+1}(\sigma_{\bm i})),\\
\overline{\bm U}{}_{\hspace*{-2pt}\bm h}^t(\sigma_{\bm h}) &= \hspace*{-10pt} \sum_{\sigma_{\bm h+1:\bm H+1}} \hspace*{-10pt} \overline{\pmb\etP}{}_{\bm h:\bm H+1}^t(\sigma_{\bm h+1:\bm H+1}|\sigma_{\bm h}) \hspace*{-2pt} \sum_{\bm i=\bm h}^{\bm H} \big\vert \widehat{\bm\gR}{}^{t,\pi^{t+1}}\hspace*{-2pt}(\sigma_{\bm i},\bm\pi_{\bm i}^{t+1}(\sigma_{\bm i})) \hspace*{-2pt} - \hspace*{-2pt} \bm\gR^{\pi^*}\hspace*{-3pt}(\sigma_{\bm i},\bm\pi_{\bm i}^{t+1}(\sigma_{\bm i})) \big\vert.
\end{align*}

We are now ready to prove the first part of Theorem~\ref{thm:hrf}.

\begin{lemma}\label{lemma:epsilon}
Under Assumption~\ref{ass:rew_alt} and event $\gE$, the policies $\bm\pi^{\tau+1}$ and $\{\pi^{\tau+1,k}\}_{k\in\ival{K}}$ returned by HBPI-UCRL satisfy $\vert \bm V_1^*(\sigma_1)-\bm V_1^{\bm\pi^{\tau+1},\pi^{\tau+1}}(\sigma_1) \vert \leq \varepsilon$.
\end{lemma}

\begin{proof}
We first use the triangle inequality to obtain
\[
\big\vert \bm V_1^*(\sigma_1)-\bm V_1^{\bm\pi^{\tau+1},\pi^{\tau+1}}(\sigma_1) \big\vert \leq \big\vert \bm V_1^*(\sigma_1)-\bm V_1^{\bm\pi^{\tau+1},\pi^*}(\sigma_1) \big\vert + \big\vert \bm V_1^{\bm\pi^{\tau+1},\pi^*}(\sigma_1)-\bm V_1^{\bm\pi^{\tau+1},\pi^{\tau+1}}(\sigma_1) \big\vert.
\]
Since $\bm V_1^{\bm\pi^{\tau+1},\pi^*}(\sigma_1)\leq \bm V_1^*(\sigma_1)\leq \overline{\bm V}{}_{\hspace*{-2pt}1}^\tau(\sigma_1)$, the first term can be bounded as
\begin{align*}
&\big\vert \bm V_1^*(\sigma_1)-\bm V_1^{\bm\pi^{\tau+1},\pi^*}(\sigma_1) \big\vert \leq \overline{\bm V}{}_{\hspace*{-2pt}1}^\tau(\sigma_1)-\bm V_1^{\bm\pi^{\tau+1},\pi^*}(\sigma_1)\\
 &\leq_{(a)} \bm HH\overline{\bm W}{}_1^\tau(\sigma_1) + \big\vert \widehat{\bm V}{}_{\hspace*{-2pt}1}^{\tau,\bm\pi^{\tau+1},\pi^{\tau+1}}(\sigma_1)-\bm V_1^{\bm\pi^{\tau+1},\pi^*}(\sigma_1) \big\vert\\
 &\leq_{(b)} \bm HH\overline{\bm W}{}_1^\tau(\sigma_1) + \hspace*{-6pt} \sum_{\sigma_{2:\bm H+1}} \hspace*{-6pt} \overline{\pmb\etP}{}_{1:\bm H+1}^\tau(\sigma_{2:\bm H+1}|\sigma_1) \sum_{\bm i=1}^{\bm H} \left\vert \widehat{\bm\gR}{}^{\tau,\pi^{\tau+1}}\hspace*{-2pt}(\sigma_{\bm i},\bm\pi_{\bm i}^{\tau+1}(\sigma_{\bm i})) - \bm\gR^{\pi^*}\hspace*{-3pt}(\sigma_{\bm i},\bm\pi_{\bm i}^{\tau+1}(\sigma_{\bm i})) \right\vert \\
 & \hspace*{60pt} + \hspace*{-4pt} \sum_{\sigma_{2:\bm H+1}} \left\vert \overline{\pmb\etP}{}_{1:\bm H+1}^\tau(\sigma_{2:\bm H+1}|\sigma_1) - \pmb\etP_{1:\bm H+1}^\tau(\sigma_{2:\bm H+1}|\sigma_1) \right\vert \sum_{\bm i=1}^{\bm H} \bm\gR^{\pi^*}(\sigma_{\bm i},\bm\pi_{\bm i}^{\tau+1}(\sigma_{\bm i}))\\
 &\leq_{(c)} \bm HH\overline{\bm W}{}_1^\tau(\sigma_1) + \overline{\bm U}{}_1^\tau(\sigma_1) + \normone{ \overline{\pmb\etP}{}_{1:\bm H+1}^\tau(\cdot|\sigma_1) - \pmb\etP_{1:\bm H+1}^\tau(\cdot|\sigma_1) } \bm HH \leq 3\bm HH\overline{\bm L}{}_1^\tau(\sigma_1).
\end{align*}
In $(a)$ we use the upper bound on $\overline{\bm V}{}^\tau$ from Lemma~\ref{lemma:upperupperV}, and in $(b)$ we use the alternative definitions of $\widehat{\bm V}{}_{\hspace*{-2pt}1}^{\tau,\bm\pi^{\tau+1},\pi^{\tau+1}}$ and $\bm V_1^{\bm\pi^{\tau+1},\pi^*}$ and the triangle inequality. In $(c)$ we identify the alternative definition of $\overline{\bm U}{}^\tau$, upper bound the reward sum by $\bm HH$ and apply Lemmas~\ref{lemma:altval} and~\ref{lemma:hierbound}.

The second term can be bounded as
\begin{align*}
&\big\vert V_1^{\bm\pi^{\tau+1},\pi^*}(\sigma_1)-\bm V_1^{\bm\pi^{\tau+1},\pi^{\tau+1}}(\sigma_1) \big\vert\\
 &\leq_{(a)} \hspace*{-5pt} \sum_{\sigma_{2:\bm H+1}} \hspace*{-5pt} \left\vert \widetilde{\pmb\etP}{}_{1:\bm H+1}^\tau(\sigma_{\bm 2:\bm H+1}|\sigma_1) - \overline{\pmb\etP}{}_{1:\bm H+1}^\tau(\sigma_{2:\bm H+1}|\sigma_1) \right\vert \sum_{\bm i=1}^{\bm H} \bm\gR^{\pi^{\tau+1}}(\sigma_{\bm i},\bm\pi_{\bm i}^{\tau+1}(\sigma_{\bm i}))\\
& \hspace*{20pt} + \hspace*{-5pt} \sum_{\sigma_{2:\bm H+1}} \hspace*{-5pt} \overline{\pmb\etP}{}_{1:\bm H+1}^\tau(\sigma_{2:\bm H+1}|\sigma_1) \sum_{\bm i=1}^{\bm H} \left\vert \bm\gR^{\pi^{\tau+1}}(\sigma_{\bm i},\bm\pi_{\bm i}^{\tau+1}(\sigma_{\bm i})) - \bm\gR^{\pi^*}(\sigma_{\bm i},\bm\pi_{\bm i}^{\tau+1}(\sigma_{\bm i})) \right\vert\\
& \hspace*{20pt} + \hspace*{-5pt} \sum_{\sigma_{2:\bm H+1}} \hspace*{-5pt} \left\vert \overline{\pmb\etP}{}_{1:\bm H+1}^\tau(\sigma_{2:\bm H+1}|\sigma_1) - \pmb\etP_{1:\bm H+1}^\tau(\sigma_{2:\bm H+1}|\sigma_1) \right\vert \sum_{\bm i=1}^{\bm H} \bm\gR^{\pi^*}(\sigma_{\bm i},\bm\pi_{\bm i}^{\tau+1}(\sigma_{\bm i}))\\
&\leq_{(b)} \normone{ \widetilde{\pmb\etP}{}_{1:\bm H+1}^\tau(\cdot|\sigma_1) - \overline{\pmb\etP}{}_{1:\bm H+1}^\tau(\cdot|\sigma_1) }\bm HH + \normone{ \overline{\pmb\etP}{}_{1:\bm H+1}^\tau(\cdot|\sigma_1) - \pmb\etP_{1:\bm H+1}^\tau(\cdot|\sigma_1) }\bm HH + \overline{\bm U}{}_1^\tau(\sigma_1)\\
 &\leq 3\bm HH\overline{\bm L}{}_1^\tau(\sigma_1).
\end{align*}
In $(a)$ we use the alternative definitions of $\bm V_1^{\bm\pi^{\tau+1},\pi^*}$ and $\bm V_1^{\bm\pi^{\tau+1},\pi^{\tau+1}}$ and the triangle inequality, and in $(b)$ we use the bound $\bm HH$ on the reward sums, identify the alternative definition of $\overline{\bm U}{}_1^\tau(\sigma_1)$ and apply Lemmas~\ref{lemma:altval} and~\ref{lemma:hierbound}.
Gathering the terms we obtain
\begin{align*}
\big\vert \bm V_1^*(\sigma_1)-\bm V_1^{\bm\pi^{\tau+1},\pi^{\tau+1}}(\sigma_1) \big\vert &\leq (3+3)\bm HH\overline{\bm L}{}_1^\tau(\sigma_1) \leq 6\bm HH\overline{\bm L}{}_1^\tau(\sigma_1) \leq \varepsilon,
\end{align*}
where the last inequality follows from the stopping criterion $\overline{\bm L}{}_1^\tau(\sigma_1)\leq \varepsilon/(6\bm HH)$.
\end{proof}

We next define an appropriate notion of pseudo-counts for the hierarchical setting. In episode $t+1$, since the policies are $\bm\pi^{t+1}$ and $\pi^{t+1,k}$, $k\in\ival{K}$, high-level transitions are governed by the transition kernel $\bm\gP_{\bm h}^{\bm\pi^{t+1},\pi^{t+1}}$. Letting $k{}_{\bf h}^t=\bm\pi_{\bm h}^{t+1}(\sigma_{\bm h})$, the expected number of visits of a state-action $(s,a)$ during episode $t+1$ is
\begin{align*}
\overline n^{t+1}(s,a) &= \sum_{\bm h=1}^{\bm H} \sum_{\sigma_{\bm h}} \pmb\sP_{\bm 1:\bm h}^{\bm\pi^{t+1},\pi^{t+1}}(\sigma_{\bm h}|\sigma_1) \sum_{h=1}^H \sP_{1:h}^{\pi^{t+1,k{}_{\bf h}^t}} (s|g(\sigma_{\bm h})) \cdot \sI(a=\pi^{t+1,k{}_{\bf h}^t}(s)).
\end{align*}
The pseudo-count is the sum over all episodes, i.e.~$\overline N{}^t(s,a)=\sum_{\ell=1}^t \overline n^\ell(s,a)$.


While the algorithm does not stop, it holds that $\varepsilon/(6\bm HH) \leq \overline{\bm L}{}_1^t(\sigma_1)$. Under events $\gE$ and $\gE_{cnt}$, summing the contributions of each episode yields

\begin{small}
\begin{align*}
&\frac {\tau\varepsilon} {6\bm HH} \leq_{(a)} \sum_{t=1}^\tau \overline{\bm L}{}_1^t(\sigma_1) \leq 3\sum_{t=1}^\tau \bm L_1^t(\sigma_1) = 3\sum_{t=1}^\tau \sum_{\bm h=\bm 1}^{\bm H} \sum_{\sigma_{\bm h}} \pmb\sP_{\bm 1:\bm h}^{\bm\pi^{t+1},\pi^{t+1}}(\sigma_{\bm h}|\sigma_1) 10H \overline L_1^{t,\bf\pi_{\bf h}^{t+1}(\sigma_{\bm h})}(g(\sigma_{\bm h}))\\
 &\leq_{(b)} 9\sum_{t=1}^\tau \sum_{\bm h=\bm 1}^{\bm H} \sum_{\sigma_{\bm h}} \pmb\sP_{\bm 1:\bm h}^{\bm\pi^{t+1},\pi^{t+1}}(\sigma_{\bm h}|\sigma_1) 10H L_1^{t,\bf\pi_{\bf h}^{t+1}(\sigma_{\bm h})}(g(\sigma_{\bm h}))\\
 &\leq_{(c)} 90H \sum_{t=1}^\tau \sum_{\bm h=\bm 1}^{\bm H} \sum_{\sigma_{\bm h}} \pmb\sP_{\bm 1:\bm h}^{\bm\pi^{t+1},\pi^{t+1}}(\sigma_{\bm h}|\sigma_1) \sum_{h=1}^H \sum_{s_h} \sP_{1:h}^{\pi^{t+1,\bm\pi_{\bm h}^{t+1}(\sigma_{\bm h})}}(s_h|g(\sigma_{\bm h})) B^t(s_h,\pi_h^{t+1,\bm\pi_{\bm h}^{t+1}(\sigma_{\bm h})}(s_h))\\
 &\leq_{(d)} 360H \sum_{t=1}^\tau \sum_{\bm h=\bm 1}^{\bm H} \sum_{\sigma_{\bm h}} \pmb\sP_{\bm 1:\bm h}^{\bm\pi^{t+1},\pi^{t+1}}(\sigma_{\bm h}|\sigma_1)\\
 & \hspace*{20pt} \times \sum_{h=1}^H \sum_{s_h} \sP_{1:h}^{\pi^{t+1,\bm\pi_{\bm h}^{t+1}(\sigma_{\bm h})}}(s_h|g(\sigma_{\bm h})) \sqrt{ \frac {{ \bm HH}\beta(\overline N{}^t(s_h,\pi_h^{t+1,\bm\pi_{\bm h}^{t+1}(\sigma_{\bm h})}(s_h)),\delta/\bm{H} H)} {\overline N{}^t(s_h,\pi_h^{t+1,\bm\pi_{\bm h}^{t+1}(\sigma_{\bm h})}(s_h)) \vee { \bm HH}} }\\
 &\leq_{(e)} 360H \sqrt{ \bm HH \beta(\bm HH\tau,\delta/\bm{H} H) } \sum_{t=1}^\tau \sum_{s,a} \frac 1 {\sqrt{ \overline N{}^t(s,a) \vee \bm HH}}\\
 & \hspace*{20pt} \times \sum_{s_h} \sum_{\bm h=\bm 1}^{\bm H} \sum_{\sigma_{\bm h}} \pmb\sP_{\bm 1:\bm h}^{\bm\pi^{t+1},\pi^{t+1}}(\sigma_{\bm h}|\sigma_1) \sum_{h=1}^H \sP_{1:h}^{\pi^{t+1,k{}_{\bm h}}}(s_h|g(\sigma_{\bm h})) \cdot \sI \left( (s,a)=(s_h,\pi^{t+1,k{}_{\bm h}}(s_h)) \right)\\
 &=_{(f)} 360H \sqrt{ \bm HH \beta(\bm HH\tau,\delta/\bm{H} H) } \sum_{t=1}^\tau \sum_{s,a} \frac { \sum_{s_h} \overline n^{t+1}(s_h,a) \cdot \sI(s=s_h) } {\sqrt{ \overline N{}^t(s,a) \vee \bm HH}}\\
 &=_{(g)} 360H \sqrt{ \bm HH \beta(\bm HH\tau,\delta/\bm{H} H) } \sum_{s,a} \sum_{t=1}^\tau \frac { \overline N{}^{t+1}(s,a) - \overline N{}^t(s,a) } {\sqrt{ \overline N{}^t(s,a) \vee \bm HH}}\\
 &\leq_{(h)} 360H \sqrt{ \bm HH \beta(\bm HH\tau,\delta/\bm{H} H) } \sqrt{\bm HH} \sum_{s,a} \sum_{t=1}^\tau \frac { (\overline N{}^{t+1}(s,a) - \overline N{}^t(s,a))/(\bm HH) } {\sqrt{ \overline N{}^t(s,a)/{ (\bm HH) \vee 1}}}\\
 &\leq_{(i)} 360\bm HH^2 \sqrt{ \beta(\bm HH\tau,\delta/\bm{H} H) SA\tau }.
\end{align*}
\end{small}
In $(a)$ we use Lemma~\ref{lemma:hierbound}, the alternative definition of $\bm L_1^t(\sigma_1)$ due to Lemma~\ref{lemma:alt} and the definition of $\bm B^t$. In $(b)$ we apply Lemma~\ref{lemma:tsbound}, and in $(c)$ we use the alternative definition of $L_1^{t,\bf\pi_{\bf h}^{t+1}(\sigma_{\bm h})}(g(\sigma_{\bm h}))$ due to Lemma~\ref{lemma:alt}. In $(d)$ we apply Lemma~\ref{lemma:pseudoconv} to convert the counts to pseudo-counts. In $(e)$ we use the monotonicity of $\beta$ to obtain an upper bound $\beta(\bm HH\tau,\delta/\bm HH)$, and we move the term $1/\sqrt{\overline N{}^t(s,a)\vee 1}$ outside the sum over $s_h$ by summing over $s,a$ and introducing an indicator function $\sI((s,a)=(s_h,\pi_h^{t+1,k{}_{\bm h}}(s_h)))$. In $(f)$ we identify $\bar n^{t+1}(s_h,a)$, and in $(g)$ we simplify the sum over $s_h$ and use the definitions of $\overline N{}^{t+1}(s,a)$ and $\overline N{}^t(s,a)$. In $(h)$ we divide the pseudo-counts by $\bm HH$, and in $(i)$ we apply Lemma 19 of~\cite{UCRL10}.

Rearranging the terms of the inequality to isolate $\tau$ on one side gives us
\begin{align*}
\sqrt{\tau} &\leq \frac {2160 \bm H^2H^3 \sqrt{ \beta(\bm HH\tau,\delta/\bm{H} H) SA } } {\varepsilon}\\
 \Leftrightarrow \quad \bm HH\tau &\leq \frac {4665600 \bm HH\cdot SA \bm H^4H^6 \beta(\bm HH\tau,\delta/\bm{H} H)} {\varepsilon^2}.
\end{align*}
From here, using the expression of $\beta$ together with Lemma~\ref{lemma:inversion} applied to $n = \bm{H} H \tau$ yields 
\[
\tau = \gO\left( \frac {SA\bm H^4H^6} {\varepsilon^2} \left( \log \frac {SA\bm{H}H} \delta + S \log \left(\frac {SA \bm H^5H^7} {\varepsilon^2} \log \frac {SA\bm{H}H} \delta \right)\right) \right).
\]
Finally, Lemma~\ref{lem:calibration} implies that the events $\gE$ and $\gE_{cnt}$ hold simultaneously with probability $1-\delta$.

\section{Proof of Theorem~\ref{thm:hrf2}}\label{app:better}

We first remark that since the flat MDP $\gM$ is sparse-reward, for any state sequence $\sigma_{1:\bm H+1}$ we have
\begin{align*}
\sum_{\bm h=1}^{\bm H} &\widehat{\bm\gR}{}^{t,\pi^{t+1}}(\sigma_{\bm h},\bm\pi_{\bm h}^{t+1}(\sigma_{\bm h}))\\
 &= \sum_{\bm h=1}^{\bm H} \widehat V_1^{t,\pi^{t+1,\bm\pi_{\bm h}^{t+1}(\sigma_{\bm h})}}(g(\sigma_{\bm h}) ;\gY_{\bm g(\sigma_{\bm h})})\\
 &= \sum_{\bm h=1}^{\bm H} \sum_{s_{2:H+1}} \hspace*{-4pt} \widehat{\etP}{}_{1:H+1}^{t,\pi^{t+1,\bm\pi_{\bm h}^{t+1}(\sigma_{\bm h})}}(s_{2:H+1}|s_1=g(\sigma_{\bm h})) \sum_{i=1}^H \gY(f(\bm g(\sigma_{\bm h}),s_i),\pi_i^{t+1}(s_i))\\
 &= \hspace*{-8pt} \sum_{\sigma_{2:\bm HH+1}} \hspace*{-4pt} \Gamma_{1:\bm HH+1}^t(\sigma_{2:\bm HH+1}|\sigma_1) \sum_{a_{1:\bm HH}} \prod_{i=1}^{\bm HH} \mu_i^t(a_i|\sigma_i) \sum_{i=1}^{\bm HH} \gY(\sigma_i,a_i) \leq 1,
\end{align*}
where $\Gamma_{1:\bm HH+1}^t$ is an appropriately chosen composite distribution on state sequences, and $\mu_i^t:\Sigma\to\Delta(\gA)$ is an appropriately chosen stochastic policy for each $i\in\ival{\bm HH}$.
As a consequence, this implies that all value functions (both true as well as empirical) are upper bounded by $1$.

We next exploit that the subproblem MDPs are sparse-reward.
We prove that Lemma~\ref{lemma:diffprob} still holds for the new definition of the confidence sets $\bm\gC^t$.
\begin{small}
\begin{align*}
&\normone { \widehat{\bm\gP}^{t,\pi^{t+1}}(\cdot|\sigma,k) - \bm\gP^{\pi^*}(\cdot|\sigma,k) } \\
 &\leq \normone { \widehat{\bm\gP}{}^{t,\pi^{t+1}}(\cdot|\sigma,k) - \bm\gP^{\pi^{t+1}}(\cdot|\sigma,k) } + \normone { \bm\gP^{\pi^{t+1}}(\cdot|\sigma,k) - \bm\gP^{\pi^*}(\cdot|\sigma,k) }\\
 &\leq \widehat L_1^{t,k}(g(\sigma)) + 2\big\vert V_1^{\pi^{t+1,k}}(g(\sigma);\gR^k) - V_1^{\pi^{*,k}}(g(\sigma);\gR^k) \big\vert\\
 &\leq \widehat L_1^{t,k}(g(\sigma)) + 2\big\vert \overline V{}_{\hspace*{-2pt}1}^{t,k}(g(\sigma);\gR^k) - V_1^{\pi^{t+1,k}}(g(\sigma);\gR^k) \big\vert\\
 &\leq \widehat L_1^{t,k}(g(\sigma)) + 2 \sum_{s_{2:H+1}} \left\vert \overline\etP{}_{1:H+1}^{t,\pi^{t+1,k}}(s_{2:H+1}|g(\sigma)) - \etP_{1:H+1}^{\pi^{t+1,k}}(s_{2:H+1}|g(\sigma)) \right\vert \sum_{i=1}^H \gR^k(s_i,\pi_i^{t+1,k}(s_i))\\
 &\leq \widehat L_1^{t,k}(g(\sigma)) + 2 \normone{ \overline\etP{}_{1:H+1}^{t,\pi^{t+1,k}}(\cdot|g(\sigma)) - \widehat\etP_{1:H+1}^{t,\pi^{t+1,k}}(\cdot|g(\sigma)) } + 2 \normone{ \widehat\etP_{1:H+1}^{t,\pi^{t+1,k}}(\cdot|g(\sigma)) - \etP_{1:H+1}^{\pi^{t+1,k}}(\cdot|g(\sigma)) }\\
 &\leq 5 \widehat L_1^{t,k}(g(\sigma)).
\end{align*}
\end{small}
This implies that Lemma~\ref{lemma:hierbound} still holds since $\lVert \bm\gQ_{\bm h}^{\bm\pi^{t+1}}(\cdot|\sigma_{\bm h}) - \overline{\bm\gP}{}_{\hspace*{-2pt}\bm h}^t(\cdot|\sigma_{\bm h}) \rVert_1 \leq \bm B_{\bm h}^t(\sigma_{\bm h})$ for the new definition of $\bm B^t$. Consequently, the modified $\overline{\bm V}{}^t$ is still an upper bound on $\bm V^*$ since $\widehat{\bm\gR}{}^{t,\pi^{t+1}}(\sigma,k)+3\widehat L_1^{t,k}(g(\sigma))$ is an upper bound on $\bm\gR^{\pi^*}(\sigma,k)$ due to Lemma~\ref{lemma:diffrew}. Since $\gM$ is sparse-reward, taking the minimum with $1$ does not invalidate the bound.

\begin{lemma}\label{lemma:altval2}
Under Assumption~\ref{ass:rew_alt} and event $\gE$, for each $\bm h$ and $\sigma$ it holds that $\overline{\bm U}{}_{\hspace*{-2pt}\bm h}^t(\sigma) \leq \overline{\bm L}{}_{\bm h}^t(\sigma)$ and $\overline{\bm W}{}_{\hspace*{-2pt}\bm h}^t(\sigma) \leq \overline{\bm L}{}_{\bm h}^t(\sigma)$ for $\overline{\bm W}^t$ defined as $\overline{\bm W}{}_{\hspace*{-2pt}\bm H+1}^t(\sigma)=0$ and
\begin{align*}
\overline{\bm W}{}_{\hspace*{-2pt}\bm h}^t(\sigma) &= \min \left[ 1, \, 3\widehat L{}_1^{t,\bm\pi_{\bm h}^{t+1}(\sigma)}(g(\sigma)) + \sum_{\sigma'} \overline{\bm\gP}{}_{\bm h}^t(\sigma'|\sigma) \overline{\bm W}{}_{\hspace*{-2pt}\bm h+1}^t(\sigma') \right].
\end{align*}
\end{lemma}

\begin{proof}
We prove the claim by induction on $\bm h$, with the base case given by $\bm h=\bm H$. In this case we have
\begin{align*}
\overline{\bm U}{}_{\hspace*{-2pt}\bm H}^t(\sigma) &= \big\vert \widehat{\bm\gR}{}^{t,\pi^{t+1}}(\sigma,\bm\pi_{\bm H}^{t+1}(\sigma)) - \bm\gR^{\pi^*}(\sigma,\bm\pi_{\bm H}^{t+1}(\sigma)) \big\vert + 0 \leq 3\widehat L_1^{t,\bm\pi_{\bm H}^{t+1}(\sigma)}(g(\sigma)) \leq \bm B_{\bm H}^t(\sigma),\\
\overline{\bm W}{}_{\hspace*{-2pt}\bm H}^t(\sigma) &= \min \left[1, 3\widehat L_1^{t,\bm\pi_{\bm H}^{t+1}(\sigma)}(g(\sigma)) + 0 \right] \leq \min \left[ 2, B_{\bm H}^t(\sigma) \right] = \overline{\bm L}{}_{\bm H}^t(\sigma).
\end{align*}
Since $\overline{\bm U}{}_{\hspace*{-2pt}\bm H}^t(\sigma)$ is upper bounded by $2$ due to $\gM$ being sparse-reward, we obtain
\[
\overline{\bm U}{}_{\hspace*{-2pt}\bm H}^t(\sigma) \leq \min \left[ 2,\bm B_{\bm H}^t(\sigma) \right] \leq \overline{\bm L}{}_{\bm H}^t(\sigma).
\]
The recursive case is given by $\bm h\in\ival{\bm H-1}$. In this case we have
\begin{align*}
\overline{\bm U}{}_{\hspace*{-2pt}\bm h}^t(\sigma) &= \big\vert \widehat{\bm\gR}{}^{t,\pi^{t+1}}(\sigma,\bm\pi_{\bm h}^{t+1}(\sigma)) - \bm\gR^{\pi^*}(\sigma,\bm\pi_{\bm h}^{t+1}(\sigma)) \big\vert + \sum_{\sigma'} \overline{\bm\gP}{}_{\bm h}^t(\sigma'|\sigma) \overline{\bm U}{}_{\hspace*{-2pt}\bm h+1}^t(\sigma')\\
 &\leq \bm B_{\bm h}^t(\sigma) + \sum_{\sigma'} \overline{\bm\gP}{}_{\bm h}^t(\sigma'|\sigma) \overline{\bm L}{}_{\bm h+1}^t(\sigma'),\\
\overline{\bm W}{}_{\hspace*{-2pt}\bm h}^t(\sigma) &= \min \left[ 1, 3\widehat L_1^{t,\bm\pi_{\bm h}^{t+1}(\sigma)}(g(\sigma)) + \sum_{\sigma'} \overline{\bm\gP}{}_{\bm h}^t(\sigma'|\sigma) \overline{\bm W}{}_{\hspace*{-2pt}\bm h+1}^t(\sigma') \right]\\
 &\leq \min \left[ 2, \bm B_{\bm h}^t(\sigma) + \sum_{\sigma'} \overline{\bm\gP}{}_{\bm h}^t(\sigma'|\sigma) \overline{\bm L}{}_{\bm h+1}^t(\sigma') \right] = \overline{\bm L}{}_{\bm h}^t(\sigma),
\end{align*}
where we have used the inductive hypothesis. Since $\overline{\bm U}{}_{\hspace*{-2pt}\bm h}^t(\sigma)$ is upper bounded by $2$ we obtain
\[
\overline{\bm U}{}_{\hspace*{-2pt}\bm h}^t(\sigma) \leq \min \left[ 2,\bm B_{\bm h}^t(\sigma) + \sum_{\sigma'} \overline{\bm\gP}{}_{\bm h}^t(\sigma'|\sigma) \overline{\bm L}{}_{\bm h+1}^t(\sigma') \right] = \overline{\bm L}{}_{\bm h}^t(\sigma).
\]
This concludes the proof of the lemma.
\end{proof}

\begin{lemma}\label{lemma:upperupperV2}
For each $\bm h$ and $\sigma$ it holds that $\overline{\bm V}{}_{\hspace*{-2pt}\bm h}^t(\sigma) \leq \overline{W}{}_{\hspace*{-2pt}\bm h}^t(\sigma) + \widehat{\bm V}{}_{\hspace*{-2pt}\bm h}^{t,\bm\pi^{t+1},\pi^{t+1}}(\sigma)$.
\end{lemma}

\begin{proof}
By induction on $\bm h$, with the base case given by $\bm h=\bm H$. In this case we have
\begin{align*}
\overline{\bm V}{}_{\hspace*{-2pt}\bm H}^t(\sigma) &= \min \left[ 1, \widehat{\bm\gR}{}^{t,\pi^{t+1}}(\sigma,\bm\pi_{\bm H}^{t+1}(\sigma)) + 3\widehat L_1^{t,\bm\pi_{\bm H}^{t+1}(\sigma)}(g(\sigma)) \right]\\
 &\leq \min \left[ 1, 3\widehat L_1^{t,\bm\pi_{\bm H}^{t+1}(\sigma)}(g(\sigma)) \right] + \widehat{\bm\gR}{}^{t,\pi^{t+1}}(\sigma,\bm\pi_{\bm H}^{t+1}(\sigma)) = \overline{W}{}_{\hspace*{-2pt}\bm H}^t(\sigma) + \widehat{\bm V}{}_{\hspace*{-2pt}\bm H}^{t,\bm\pi^{t+1},\pi^{t+1}}(\sigma).
\end{align*}
The recursive case is given by $\bm h\in\ival{\bm H-1}$. In this case we have
\begin{align*}
&\overline{\bm V}{}_{\hspace*{-2pt}\bm h}^t(\sigma) = \min \Big[ 1, \widehat{\bm\gR}{}^{t,\pi^{t+1}}(\sigma,\bm\pi_{\bm h}^{t+1}(\sigma)) + 3\widehat L_1^{t,\bm\pi_{\bm h}^{t+1}(\sigma)}(g(\sigma)) + \sum_{\sigma'} \overline{\bm\gP}{}_{\hspace*{-2pt}\bm h}^t(\sigma'|\sigma) \overline{\bm V}{}_{\hspace*{-2pt}\bm h+1}^t(\sigma') \Big]\\
 &\leq \min \Big[ 1, \widehat{\bm\gR}{}^{t,\pi^{t+1}}(\sigma,\bm\pi_{\bm h}^{t+1}(\sigma)) + 3\widehat L_1^{t,\bm\pi_{\bm h}^{t+1}(\sigma)}(g(\sigma))\\
 & \hspace*{40pt} + \sum_{\sigma'} \overline{\bm\gP}{}_{\hspace*{-2pt}\bm h}^t(\sigma'|\sigma) \left( \overline{W}{}_{\hspace*{-2pt}\bm h+1}^t(\sigma') + \widehat{\bm V}{}_{\hspace*{-2pt}\bm h+1}^{t,\bm\pi^{t+1},\pi^{t+1}}(\sigma') \right) \Big]\\
 &\leq \min \Big[ 1, 3\widehat L_1^{t,\bm\pi_{\bm h}^{t+1}(\sigma)}(g(\sigma)) + \sum_{\sigma'} \overline{\bm\gP}{}_{\hspace*{-2pt}\bm h}^t(\sigma'|\sigma) \overline{W}{}_{\hspace*{-2pt}\bm h+1}^t(\sigma') \Big]\\
 & \hspace*{10pt} + \widehat{\bm\gR}{}^{t,\pi^{t+1}}(\sigma,\bm\pi_{\bm h}^{t+1}(\sigma)) + \sum_{\sigma'} \overline{\bm\gP}{}_{\hspace*{-2pt}\bm h}^t(\sigma'|\sigma) \widehat{\bm V}{}_{\hspace*{-2pt}\bm h+1}^{t,\bm\pi^{t+1},\pi^{t+1}}(\sigma')\\
 &= \overline{W}{}_{\hspace*{-2pt}\bm h}^t(\sigma) + \widehat{\bm V}{}_{\hspace*{-2pt}\bm h}^{t,\bm\pi^{t+1},\pi^{t+1}}(\sigma).
\end{align*}
This concludes the proof.
\end{proof}

We are now ready to prove the first part of the theorem.

\begin{lemma}\label{lemma:epsilon2}
Under Assumption~\ref{ass:rew_alt} and event $\gE$, the policies $\bm\pi^{t+1}$ and $\{\pi^{t+1,k}\}_{k\in\ival{K}}$ returned by the modified version of HBPI-UCRL satisfy $\big\vert \bm V_1^*(\sigma_1)-\bm V_1^{\bm\pi^{t+1},\pi^{t+1}}(\sigma_1) \big\vert \leq \varepsilon$.
\end{lemma}

\begin{proof}
We first use the triangle inequality to obtain
\[
\big\vert \bm V_1^*(\sigma_1)-\bm V_1^{\bm\pi^{t+1},\pi^{t+1}}(\sigma_1) \big\vert \leq \big\vert \bm V_1^*(\sigma_1)-\bm V_1^{\bm\pi^{t+1},\pi^*}(\sigma_1) \big\vert + \big\vert \bm V_1^{\bm\pi^{t+1},\pi^*}(\sigma_1)-\bm V_1^{\bm\pi^{t+1},\pi^{t+1}}(\sigma_1) \big\vert.
\]
Since $\bm V_1^{\bm\pi^{t+1},\pi^*}(\sigma_1)\leq \bm V_1^*(\sigma_1)\leq \overline{\bm V}{}_{\hspace*{-2pt}1}^t(\sigma_1)$, the first term can be bounded as
\begin{align*}
&\big\vert \bm V_1^*(\sigma_1)-\bm V_1^{\bm\pi^{t+1},\pi^*}(\sigma_1) \big\vert \leq \overline{\bm V}{}_{\hspace*{-2pt}1}^t(\sigma_1)-\bm V_1^{\bm\pi^{t+1},\pi^*}(\sigma_1)\\
 &\leq \overline{\bm W}{}_1^t(\sigma_1) + \big\vert \widehat{\bm V}{}_{\hspace*{-2pt}1}^{t,\bm\pi^{t+1},\pi^{t+1}}(\sigma_1)-\bm V_1^{\bm\pi^{t+1},\pi^*}(\sigma_1) \big\vert\\
 &\leq \overline{\bm W}{}_1^t(\sigma_1) + \hspace*{-6pt} \sum_{\sigma_{2:\bm H+1}} \hspace*{-6pt} \overline{\pmb\etP}{}_{1:\bm H+1}^t(\sigma_{2:\bm H+1}|\sigma_1) \sum_{\bm i=1}^{\bm H} \left\vert \widehat{\bm\gR}{}^{t,\pi^{t+1}}\hspace*{-2pt}(\sigma_{\bm i},\bm\pi_{\bm i}^{t+1}(\sigma_{\bm i})) - \bm\gR^{\pi^*}\hspace*{-3pt}(\sigma_{\bm i},\bm\pi_{\bm i}^{t+1}(\sigma_{\bm i})) \right\vert \\
 & \hspace*{60pt} + \hspace*{-4pt} \sum_{\sigma_{2:\bm H+1}} \left\vert \overline{\pmb\etP}{}_{1:\bm H+1}^t(\sigma_{2:\bm H+1}|\sigma_1) - \pmb\etP_{1:\bm H+1}^t(\sigma_{2:\bm H+1}|\sigma_1) \right\vert \sum_{\bm i=1}^{\bm H} \bm\gR^{\pi^*}(\sigma_{\bm i},\bm\pi_{\bm i}^{t+1}(\sigma_{\bm i}))\\
 &\leq \overline{\bm W}{}_1^t(\sigma_1) + \overline{\bm U}{}_1^t(\sigma_1) + \normone{ \overline{\pmb\etP}{}_{1:\bm H+1}^t(\cdot|\sigma_1) - \pmb\etP_{1:\bm H+1}^t(\cdot|\sigma_1) } \leq 3\overline{\bm L}_1^t(\sigma_1),
\end{align*}
where we have used the upper bound $1$ on the reward due to $\gM$ being sparse-reward. The second term can be bounded as
\begin{align*}
&\big\vert V_1^{\bm\pi^{t+1},\pi^*}(\sigma_1)-\bm V_1^{\bm\pi^{t+1},\pi^{t+1}}(\sigma_1) \big\vert\\
 &\leq \hspace*{-5pt} \sum_{\sigma_{2:\bm H+1}} \hspace*{-5pt} \left\vert \widetilde{\pmb\etP}{}_{1:\bm H+1}^t(\sigma_{\bm 2:\bm H+1}|\sigma_1) - \overline{\pmb\etP}{}_{1:\bm H+1}^t(\sigma_{2:\bm H+1}|\sigma_1) \right\vert \sum_{\bm i=1}^{\bm H} \bm\gR^{\pi^{t+1}}(\sigma_{\bm i},\bm\pi_{\bm i}^{t+1}(\sigma_{\bm i}))\\
& \hspace*{20pt} + \hspace*{-5pt} \sum_{\sigma_{2:\bm H+1}} \hspace*{-5pt} \overline{\pmb\etP}{}_{1:\bm H+1}^t(\sigma_{2:\bm H+1}|\sigma_1) \sum_{\bm i=1}^{\bm H} \left\vert \bm\gR^{\pi^{t+1}}(\sigma_{\bm i},\bm\pi_{\bm i}^{t+1}(\sigma_{\bm i})) - \bm\gR^{\pi^*}(\sigma_{\bm i},\bm\pi_{\bm i}^{t+1}(\sigma_{\bm i})) \right\vert\\
& \hspace*{20pt} + \hspace*{-5pt} \sum_{\sigma_{2:\bm H+1}} \hspace*{-5pt} \left\vert \overline{\pmb\etP}{}_{1:\bm H+1}^t(\sigma_{2:\bm H+1}|\sigma_1) - \pmb\etP_{1:\bm H+1}^t(\sigma_{2:\bm H+1}|\sigma_1) \right\vert \sum_{\bm i=1}^{\bm H} \bm\gR^{\pi^*}(\sigma_{\bm i},\bm\pi_{\bm i}^{t+1}(\sigma_{\bm i}))\\
&\leq \normone{ \widetilde{\pmb\etP}{}_{1:\bm H+1}^t(\cdot|\sigma_1) - \overline{\pmb\etP}{}_{1:\bm H+1}^t(\cdot|\sigma_1) } + \normone{ \overline{\pmb\etP}{}_{1:\bm H+1}^t(\cdot|\sigma_1) - \pmb\etP_{1:\bm H+1}^t(\cdot|\sigma_1) } + \overline{\bm U}{}_1^t(\sigma_1) \leq 3\overline{\bm L}{}_1^t(\sigma_1).
\end{align*}
Gathering the terms we obtain
\begin{align*}
\big\vert \bm V_1^*(\sigma_1)-\bm V_1^{\bm\pi^{t+1},\pi^{t+1}}(\sigma_1) \big\vert &\leq (3+3)\overline{\bm L}{}_1^t(\sigma_1) \leq 6\overline{\bm L}{}_1^t(\sigma_1) \leq \varepsilon,
\end{align*}
where we have used the new stopping criterion $\overline{\bm L}{}_1^t(\sigma_1)\leq \varepsilon/6$.
This concludes the proof.
\end{proof}

The remaining analysis is the same as in the proof of Theorem~\ref{thm:hrf}. However, since the stopping criterion is different and since $\bm B^t$ does not contain a factor $H$ anymore, summing the contributions of each episode yields
\[
\frac {\tau\varepsilon} 6 \leq 360\bm HH \sqrt{ \beta(\bm HH\tau,\delta/\bm{H} H) SA\tau }.
\]
Rearranging the terms of the inequality to isolate $\tau$ on one side gives us
\begin{align*}
\sqrt{\tau} &\leq \frac {2160 \bm HH \sqrt{ \beta(\bm HH\tau,\delta/\bm{H} H) SA } } {\varepsilon}\\
 \Leftrightarrow \quad \bm HH\tau &\leq \frac {4665600 \bm HH\cdot SA \bm H^2H^2 \beta(\bm HH\tau,\delta/\bm{H}H)} {\varepsilon^2}.
\end{align*}
From here, using Lemma~\ref{lemma:inversion} gives
\[
\tau \leq \gO\left( \frac {SA\bm H^2H^2} {\varepsilon^2} \left( \log \frac {SA\bm{H}H} \delta + S \log \left(\frac {SA \bm H^3H^3} {\varepsilon^2} \log \frac {SA\bm{H}H} \delta \right)\right) \right).
\]
Finally, Lemma~\ref{lem:calibration} implies that the events $\gE$ and $\gE_{cnt}$ hold simultaneously with probability $1-\delta$.

\section{On High-Probability Events} \label{app:concentration}

We first recall the definition of the counts and introduce the notion of pseudo-counts, in the general hierarchical setting. We have
\begin{eqnarray*}
    N^{t}(s,a) &= &\sum_{\ell = 1}^{t}\sum_{\bm h =1}^{\bm H}\sum_{h=1}^{H} \sI(s_h^{\ell,\bm h} = s, a_h^{\ell,\bm h} = a) \\
    \overline{N}{}^{t}(s,a) &= &\sum_{\ell = 1}^{t}\sum_{\bm h=1}^{\bm H} \sum_{h=1}^{H}\sum_{\sigma_{\bm h}} \pmb\sP_{\bm 1:\bm h}^{\bm\pi^{\ell},\pi^{\ell}}(\sigma_{\bm h}|\sigma_1) \cdot \sP_{1:h}^{\pi^{\ell,\bm\pi_{\bm h}^{\ell}(\sigma_{\bm h})}} (s|g(\sigma_{\bm h})) \cdot \sI(a=\pi^{\ell,\bm\pi_{\bm h}^{\ell}(\sigma_{\bm h})}(s)) \\
    & =& \sum_{\ell = 1}^{t}\sum_{\bm h=1}^{\bm H} \sum_{h=1}^{H}\mathbb{P}\left(\left.s_h^{\ell,\bm h} = s, a_h^{\ell,\bm h} = a \right| \mathcal{F}_{\ell-1}\right)
\end{eqnarray*}
where $\mathcal{F}_{\ell}$ is  the filtration generated by the observation made in the $\ell$ first episodes. Indeed, by definition of the algorithm, the high level policy $\bm \pi^{\ell}$ and low level policies $\pi^{\ell,k}$ used in episode $\ell$ are $\mathcal{F}_{\ell -1}$ measurable. Observe that we are marginalizing out the starting state of each low-level episode, $\sigma_{\bm h}$, in the computation of the conditional probability.

With these definition, and the empirical transition probability $\hat{\mathcal P}^{t}(s' | s,a) = \frac{{N}^{t}(s,a,s')}{N^{t}(s,a)}$, we introduce the events 
\begin{eqnarray*}
\gE & = &\left\{ \forall t\in\mathbb{N}, \forall (s,a), \left\|\hat{\mathcal P}_h^{t}(\cdot | s,a) -  {\mathcal P}_h(\cdot | s,a) \right\|_{1} \leq B_h^{t}(s,a) \right\} \\
\text{and } \gE^{\text{cnt}}  & = &\left\{ \forall t\in\mathbb{N}, \forall (s,a), N_h^{t}(s,a) \geq \frac{1}{2}\overline{N}{}_h^{t}(s,a) - \beta^{\text{cnt}}(\delta) \right\}
\end{eqnarray*}
where $B_h^{t}(s,a) = \sqrt{\frac{2\beta(N_h^{t}(s,a),\delta)}{N_h^{t}(s,a)}} \wedge 2$.

\begin{lemma}\label{lem:calibration} Choosing 
\begin{eqnarray*}
   \beta(n,\delta) & = & \log(2SA/\delta) + (S-1)\log(e(1+n/(S-1))) \\
   \beta^{\text{cnt}}(\delta) & = & {\bm H H} \log\left(2SA\bm{H} H/\delta\right)
\end{eqnarray*}
yields $\mathbb{P}\left(\gE^c\right) \leq \frac{\delta}{2}$ and $\mathbb{P}\left((\gE^{\text{cnt}})^c\right) \leq \frac{\delta}{2}$.
\end{lemma}

\begin{proof} To prove the first statement, we write 
\begin{eqnarray*}
\mathbb{P}\left(\gE^c\right) &{\leq} & \sum_{s,a} \mathbb{P}\left(\exists t \in \mathbb{N} : \|\hat{\mathcal P}^{t}(\cdot | s,a) -  {\mathcal P}(\cdot | s,a) \|_1> B^{t}(s,a) \right) \\
& \overset{(a)}{\leq} & \sum_{s,a} \mathbb{P}\left(\exists t \in \mathbb{N} : \|\hat{\mathcal P}^{t}(\cdot | s,a) -  {\mathcal P}(\cdot | s,a) \|_1> \sqrt{\frac{2\beta(N^{t}(s,a),\delta)}{N^{t}(s,a)}} \right) \\
&  \overset{(b)}{\leq} & \sum_{s,a} \mathbb{P}\left(\exists t \in \mathbb{N} : N^{t}(s,a) > 1,  N^{t}(s,a) \DR{KL}{\widehat\gP^t(\cdot|s,a)}{\gP(\cdot|s,a)} > \beta(N^{t}(s,a),\delta)\right) \\
& \overset{(c)}{\leq} & \sum_{s,a} \mathbb{P}\left(\exists n \geq 1 : n \DR{KL}{\widehat\gP^{(n)}(\cdot|s,a)}{\gP(\cdot|s,a)} > \beta(n,\delta)\right) \\
& \overset{(d}{\leq} & \frac{\delta}{2}
\end{eqnarray*}
where in $(a)$ we use that the L1 norm between any two probability distributions is upper bounded by 2, in $(b)$ we use Pinsker's inequality, in $(c)$ we introduce ${\widehat\gP^{(n)}(\cdot|s,a)}$ the empirical transition probability based on the first $n$ visits of $(s,a)$ (in any step $\bm h$ of any (high level) episode and any step $h$ of a low-level episode) and in (d) we use the expression of $\beta$ and Proposition 1 from~\cite{MDPGapE} which provides time-uniform concentration bound in KL on transition probabilities.

To establish the second statement, for each $\bm h \in [\bm H]$ and each $h \in [H]$, we use Lemma  F.4 in~\cite{dann2017unifying} to prove that, for all $\delta \in (0,1]$, 
\[\mathbb{P}\left(\exists t \in \mathbb{N} : \sum_{\ell = 1}^{t}\sI(s_h^{\ell,\bm h} = s, a_h^{\ell,\bm h} = a) < \frac{1}{2}\sum_{\ell = 1}^{t} \mathbb{P}\left(s_h^{\ell,\bm h} = s, a_h^{\ell,\bm h} = a | \mathcal{F}_{\ell-1}\right) - \log(1/\delta)\right) \leq \delta.\]
As a consequence, we get that 
\begin{small}
\begin{align*}&\mathbb{P}\hspace*{-2pt}\left(\hspace*{-2pt}\exists t \in \mathbb{N} : \hspace*{-4pt} \sum_{\substack{\bm h \in [\bm{H}]\\ h \in [H]}}\sum_{\ell = 1}^{t}\sI(s_h^{\ell,\bm h} = s, a_h^{\ell,\bm h} = a) < \frac{1}{2}\sum_{\substack{\bm h \in [\bm{H}]\\ h \in [H]}}\sum_{\ell = 1}^{t} \mathbb{P}\left(s_h^{\ell,\bm h} = s, a_h^{\ell,\bm h} = a | \mathcal{F}_{\ell-1}\right) - \bm{H} H \log(1/\delta)\right) \\
& \leq  \sum_{\substack{\bm h \in [\bm{H}]\\ h \in [H]}}\mathbb{P}\left(\exists t \in \mathbb{N} : \sum_{\ell = 1}^{t}\sI(s_h^{\ell,\bm h} = s, a_h^{\ell,\bm h} = a) < \frac{1}{2}\sum_{\ell = 1}^{t} \mathbb{P}\left(s_h^{\ell,\bm h} = s, a_h^{\ell,\bm h} = a | \mathcal{F}_{\ell-1}\right) - \log(1/\delta)\right) \\
& \leq \bm{H} H \delta\;.\end{align*}
\end{small}
Using the definition of the counts and pseudo-counts and a union bound over $s$ and $a$ further yields
\[\mathbb{P}\left(\exists t \in \mathbb{N}, \exists s \in \mathcal{S}, a \in \mathcal{A}: N^{t}(s,a)< \frac{1}{2}\overline{N}^{t}(s,a) - \bm{H}H\log(2SA\bm{H} H/\delta)\right) \leq \frac{\delta}{2}\]
hence we have $\mathbb{P}((\mathcal{E}^{\text{cnt}})^c) \leq \delta/2$.
\end{proof}

\paragraph{From counts to pseudo-counts} Below we provide a simple extension of Lemma 7 from~\cite{kaufmann2021rf} to relate the counts and pseudo-counts.

\begin{lemma}\label{lemma:pseudoconv}
On the event $\mathcal{E}^{\text{cnt}}$ (with $\beta^{\text{cnt}}$ as in Lemma~\ref{lem:calibration}) it holds that, for all $t$ and $(s,a)$,
\[B^{t}(s,a) \leq 4\sqrt{\bm H H}\sqrt{\frac{\beta\left(\overline{N}{}_h^{t}(s,a),\frac{\delta}{\bm H H}\right)}{\overline{N}{}_h^{t}(s,a) \vee \bm{H} H}}.\]   
\end{lemma}

\begin{proof} We assume that $\mathcal{E}^{\text{cnt}}$ holds and fix a time step $t$ and a triplet $(h,s,a)$. To ease the notation, we let $n= N_h^{t}(s,a)$ and $\overline{n} = \overline{N}{}_h^{t}(s,a)$. 

We consider a first case in which $\frac{1}{4} \overline{n} > \bm H H\log\left(\frac{2SA\bm{H} H}{\delta}\right)$. On event $\mathcal{E}^{\text{cnt}}$, this implies $n \geq \frac{1}{4}\overline{n}$. Using the monotonicity properties of $\beta$ yields 
\[\frac{\beta(n,\delta)}{n} \leq \frac{\beta\left(\overline{n}/4,\delta)\right)}{\overline{n}/4} \leq 4 \frac{\beta(\overline{n},\delta)}{\overline{n}} = 4 \frac{\beta(\overline{n},\delta)}{\overline{n}\vee \bm{H} H},\]
where the last equality follows from the observation that $\overline{n} \geq \bm{H} H$ as $4\log\left(\frac{2SA\bm{H} H}{\delta}\right) \geq 4\log\left(\frac{2}{\delta}\right) \geq 1$ for $\delta \in (0,1)$. 
Hence 
\[\sqrt{\frac{2\beta(n,\delta)}{n}} \leq  \sqrt{8} \sqrt{\frac{\beta(\overline{n},\delta)}{\overline{n}\vee \bm{H} H}}\; \]
which implies $B_h^{t}(s,a) \leq \sqrt{8}\sqrt{\frac{\beta(\overline{N}_h^{t}(s,a),\delta)}{\overline{N}_h^{t}(s,a) \vee  \bm{H} H}} \leq 4\sqrt{\bm H}\sqrt{\frac{\beta\left(\overline{N}_h^{t}(s,a),\frac{\delta}{\bm H H}\right)}{\overline{N}_h^{t}(s,a) \vee \bm{H} H}}$\;.

In the second case, we assume $\frac{1}{4} \overline{n} \leq \bm H H \log\left(\frac{2SA\bm{H}H}{\delta}\right)$ hence $\overline{n} \leq 4 \bm H H \log\left(\frac{2SAH}{\delta}\right)$. As the right-hand side is also always larger than $\bm{H} H$, we get 
\[\overline{n} \vee \bm{H} H \leq 4 \bm H H\log\left(\frac{2SA\bm{H}H}{\delta}\right)\;\;.\]
Hence we have
\[1 \leq \frac{4\bm H H \log\left(\frac{2SA\bm{H} H}{\delta}\right)}{\overline{n} \vee \bm{H} H} = \frac{4\bm H H\beta\left(0,\frac{\delta}{\bm{H} H}\right)}{\overline{n} \vee \bm{H} H}\leq \frac{4\bm H H\beta\left(\overline{n},\frac{\delta}{\bm H H}\right)}{\overline{n} \vee \bm{H} H} \]
and 
\[2 \leq 4 \sqrt{\bm H H}\sqrt{\frac{ \beta\left(\overline{n},\frac{\delta}{\bm H H}\right)}{\overline{n} \vee 1}}\;.\]
As $B_h^{t}(s,a) \leq 2$, we also have $B_h^{t}(s,a) \leq 4\sqrt{\bm H H}\sqrt{\frac{\beta\left(\overline{N}_h^{t}(s,a),\frac{\delta}{\bm H H}\right)}{\overline{N}{}_h^{t}(s,a) \vee 1}}$, which concludes the proof.
\end{proof}

\section{A Technical Lemma}

For completeness, we provide below a technical lemma extracted from the literature that allows to get an upper bound on $x$ from an inequality upper bounding $x$ by some (quasi) linear function of $\log(x)$.

\begin{lemma}[Lemma 15 of \cite{kaufmann2021rf}]
	\label{lemma:inversion}
	Let $n \geq 1$ and $a, b, c, d > 0$. If $n \Delta^2 \leq a + b \log( c+ dn)$
	then
	\begin{align*}
	n \leq \frac{1}{\Delta^2}\left[ a + b \log\left(c + \frac{d}{\Delta^4} (a + b(\sqrt{c} + \sqrt{d}))^2 \right) \right].
	\end{align*}
\end{lemma}

\section{Experimental Settings and Additional Results}\label{app:exp}
In this appendix, we discuss the experimental setting for the experiments provided in the main paper and we present further results.

\paragraph{Algorithm Setting.}
We perform experiments on a synthetic grid maze as presented in Figure~\ref{fig:ex}, with treasures removed. The environment consists of an $m\times m$ grid of rooms, where each room is a discrete grid of size $n \times n$. Rooms are connected via specific doorways (corresponding to terminal states of subproblems). The total number of states is $(m \cdot n)^2 + 1$, with the goal state $G$ separate from the grid. The action space is a discrete set of 5 possible actions, where the first four actions move the agent in a cardinal direction, and the last action is exclusively used to reach the goal $G$. The transitions are deterministic. The reward function is sparse as the agent only receives a reward of $1$ upon reaching the goal state $G$ in the final room, or when reaching the correct terminal state in the case of subproblems. 

\begin{figure}
\centering
    \includegraphics[height=6cm]{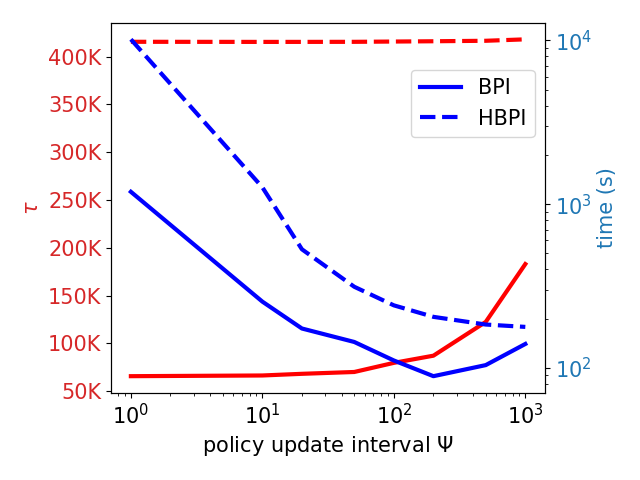}
\caption{Stopping time and execution time as functions of the policy update interval $\Psi$.}
\label{fig:stopexec}
\end{figure}

In each episode, the agent starts in an initial state uniformly sampled within the first room.
Without loss of generality this setting can be modelled using a unique dummy initial state $s_\bot$ by defining the dynamics of each action $a$ as $\gR(s_\bot,a)=0$ and $\gP(s'|s_\bot,a)=\nu(s')$, where $\nu(s')$ is the initial probability of $s'$.
All experiments are performed with three fixed random seeds, a confidence parameter $\delta = 0.1$ and an accuracy $\varepsilon=1$. Since the domain is deterministic, the value of a policy is 1 if it successfully reaches $G$, and $0$ otherwise. Hence an accuracy $\varepsilon=1$ is sufficient to distinguish optimal from non-optimal policies. For HBPI-UCRL, the theoretical confidence bounds derived from Hoeffding bonuses are not very tight. To make the algorithm more practical, we remove the constants in several expressions: 3 in the SMDP reward bonus of $\overline{\bm V}{}^t$, 5 in the SMDP confidence set $\bm\gC^t$ and 10 in the expression for $\bm B^t$. As is customary in the literature~\citep[e.g.][]{kaufmann2021rf}, we use a simplified bound $\beta(n,\delta)=2\log(1/\delta)+\log n$ in both BPI-UCRL and HBPI-UCRL.

Since running the algorithms until stopping is slow, we implement an approximation that updates the policy only every $\Psi$ episodes.
Figure~\ref{fig:stopexec} shows the stopping time (red) and execution time (blue) of BPI (solid) and HBPI (broken) as functions of $\Psi$, in the setting of $3\times 3$ rooms of size $3\times 3$.
For the experiments we choose $\Psi=20$ for BPI and $\Psi=100$ for HBPI since these values offer good tradeoffs between stopping time accuracy and execution time.
All experiments were run on a single core of a standard CPU.
Even with fewer policy updates, some executions of BPI-UCRL and HBPI-UCRL are slow, and can sometimes take multiple days of computation time.

\revision{
\paragraph{Additional Results on Specific Instances.}
To further understand the behavior of HBPI-UCRL, we present additional experiments over different domains, varying both the room size $n$ and the grid dimension $m$. We first report the average reward of the algorithms, to get a sense on how fast the algorithms are playing near optimal policies.}

In Figure~\ref{fig:2x2grids}, we can notice how in small domains such as $2\times 2 = 4$ rooms, BPI-UCRL manages, within a shorter horizon, to reach comparable performance to HBPI-UCRL. This result illustrates a trade-off in the domain complexity as the benefit of a hierarchical structure can not be appreciated in such simple tasks. 

\begin{figure}[h!]
    \subfloat[$2\times 2$ grid $3 \times 3$ room.]{
        \resizebox{0.485\linewidth}{0.4\linewidth}{\includegraphics[]{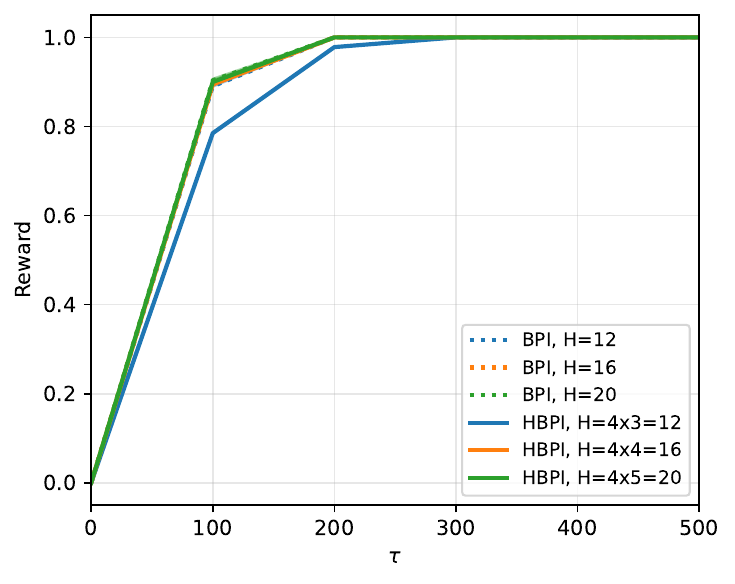}}%
        \label{fig:2x2grid-3x3room}%
    }\hfill
    \subfloat[$2\times 2$ grid $5 \times 5$ room.]{
        \resizebox{0.485\linewidth}{0.4\linewidth}{\includegraphics[]{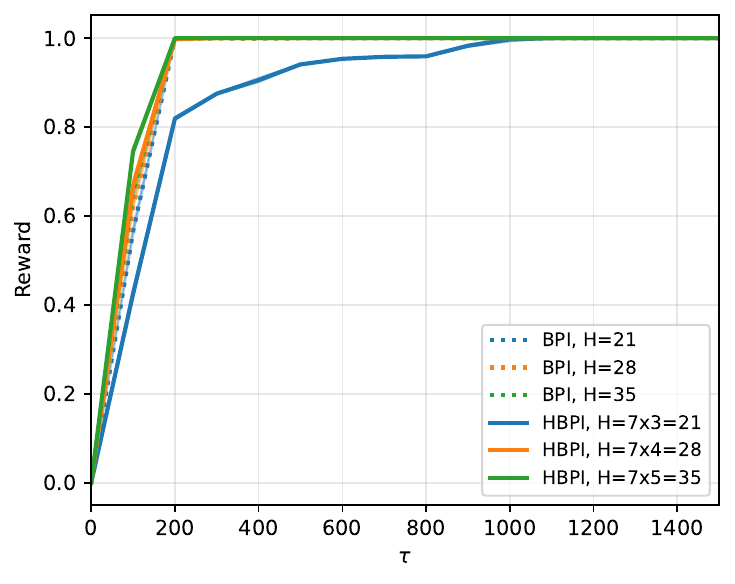}}%
        \label{fig:2x2grid-5x5room}%
    }
    \caption{Comparison of observed reward between BPI-UCRL and HBPI-UCRL (dotted line for BPI and solid line for HBPI 3 runs, mean $\pm$ $95\%$ C.I.).}
    \label{fig:2x2grids}
\end{figure}

Differently, in Figure~\ref{fig:3x3-4x4grids}, as we approach larger domains with $3 \times 3 = 9$ and $4 \times 4 = 16$ rooms of dimension $n = \{3, 5\}$, the advantage of a hierarchical structure is more evident. HBPI-UCRL consistently manages to \revision{discover optimal policies} in fewer time steps with respect to BPI-UCRL, which in more complex settings requires almost double the time to approximate an optimal policy (Figure~\ref{fig:4x4grid-5x5room}).
This behavior is consistent with our analysis as HBPI-UCRL scales better in larger environments as it successfully mitigates the increase in the number of high-level states. 

\begin{figure}[h!]
    \centering
    \subfloat[$3\times 3$ grid $3 \times 3$ room.]{
        \resizebox{0.32\linewidth}{!}{\includegraphics[]{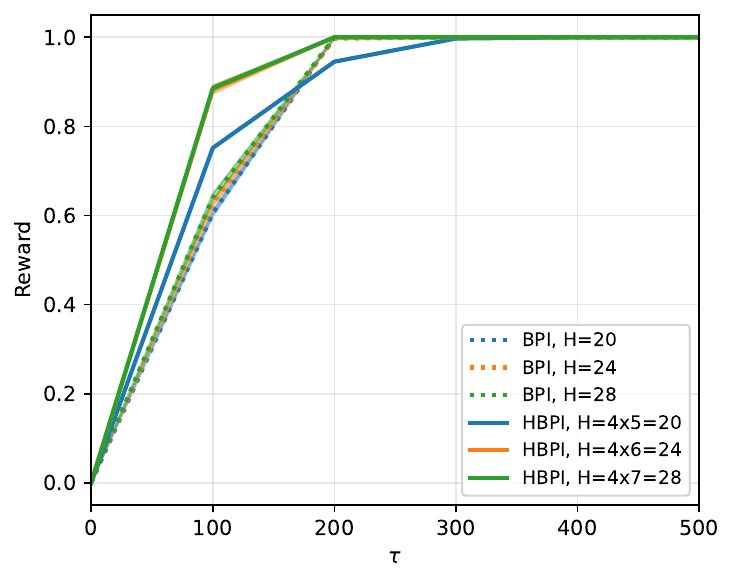}}%
        \label{fig:3x3grid-3x3room}%
    }
    \subfloat[$3\times 3$ grid $5 \times 5$ room.]{
        \resizebox{0.32\linewidth}{!}{\includegraphics[]{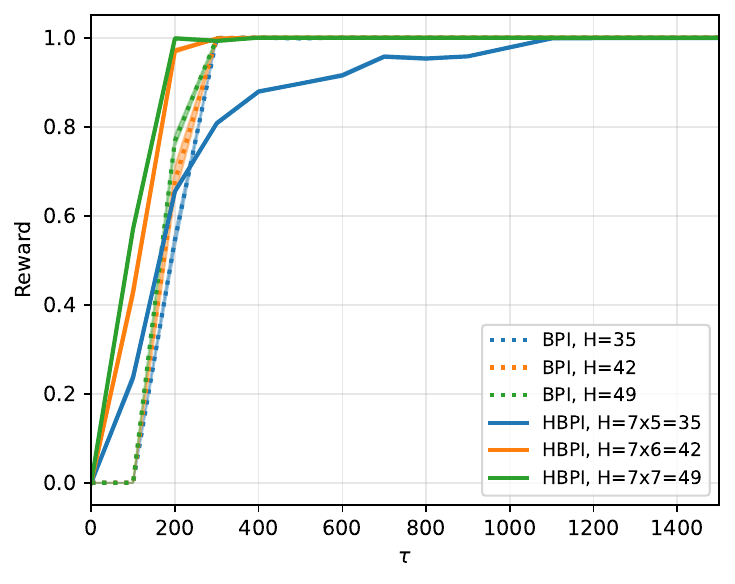}}%
        \label{fig:3x3grid-5x5room}%
    }
    \subfloat[$4\times 4$ grid $5 \times 5$ room.]{
        \resizebox{0.32\linewidth}{!}{\includegraphics[]{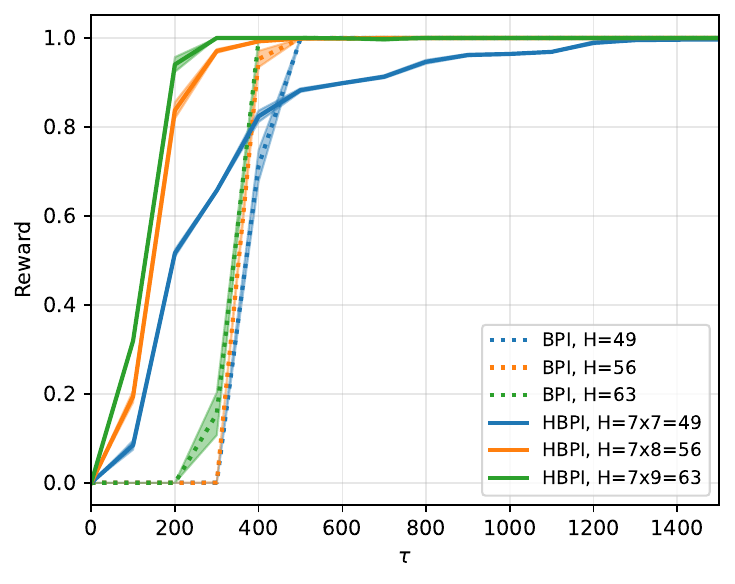}}%
        \label{fig:4x4grid-5x5room}%
    }
    \caption{Comparison of observed reward between BPI-UCRL and HBPI-UCRL (dotted line for BPI and solid line for HBPI 3 runs, mean $\pm$ $95\%$ C.I.).}
    \label{fig:3x3-4x4grids}
\end{figure}

Eventually, pushing the grid size to $5 \times 5 = 25$ rooms, in Figure~\ref{fig:5x5grids} we can consistently see the advantage of adopting a hierarchical structure both in terms of performance, as BPI-UCRL requires more time steps to reach an optimal \revision{policy}, approximating the policy at a much slower rate than HBPI-UCRL. 


\begin{figure}[h]
    \centering
    \subfloat[$5\times 5$ grid $3 \times 3$ room.]{
        \resizebox{0.485\linewidth}{0.45\linewidth}{\includegraphics[]{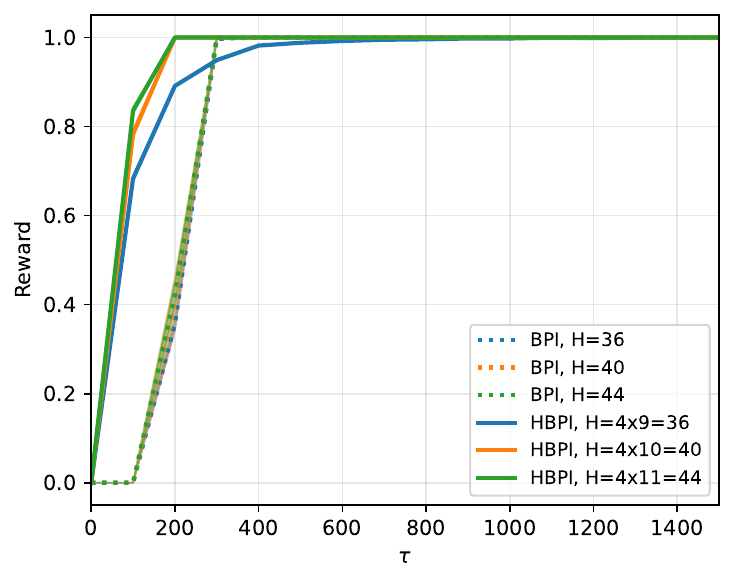}}%
        \label{fig:5x5grid-3x3room}%
    }\hfill
    \subfloat[$5\times 5$ grid $5 \times 5$ room.]{
        \resizebox{0.485\linewidth}{0.45\linewidth}{\includegraphics[]{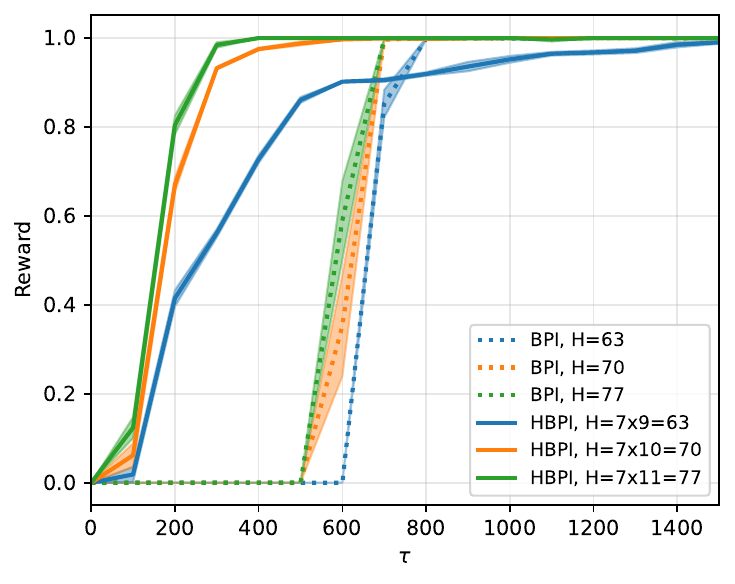}}%
        \label{fig:5x5grid-5x5room}%
    }
    \caption{Comparison of observed reward between BPI-UCRL and HBPI-UCRL (dotted line for BPI and solid line for HBPI, 3 runs, mean $\pm$ $95\%$ C.I.).}
    \label{fig:5x5grids}
\end{figure}

Figure~\ref{fig:heatmap} shows the heatmap of visited states during the first $40{,}000$ timesteps of BPI-UCRL and HBPI-UCRL in the domain with $5\times 5$ rooms of size $3\times 3$. It is interesting to see that the two algorithms explore the space in very different ways.
BPI-UCRL explores all states uniformly until it eventually finds a policy that reaches the goal state in the middle of the top right room, highlighted by the star marker.
In contrast, HBPI-UCRL extensively explores the first room, but after that it does not need to explore all rooms and all states within a room since the rooms share the dynamics. Instead it proceeds more directly to the goal. The asymmetry is likely due to the fact that the subproblem for moving up has a lower index than the subproblem for moving to the right, which causes tiebreaking to prefer moving up first.

\begin{figure}[h]
\begin{center}
    \input{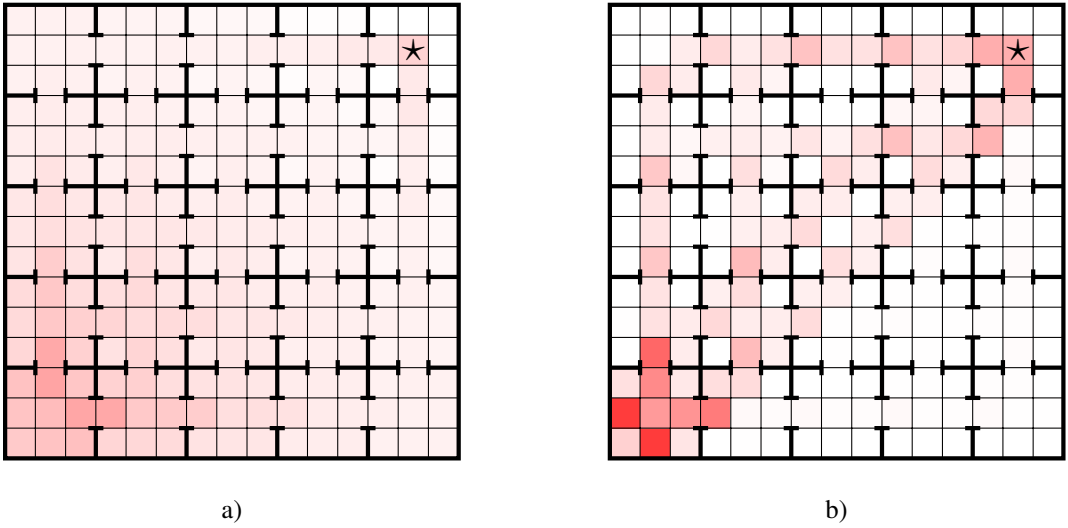}
\end{center}
\caption{Heatmap of visited states in the domain with $5\times 5$ rooms of size $3\times 3$ for a) BPI-UCRL; b) HBPI-UCRL.}
\label{fig:heatmap}
\end{figure}

\begin{figure}[h]
    \centering
    \subfloat[Stopping time of $\bm S$ for fixed $\bm H$.]{
        \resizebox{0.485\linewidth}{0.45\linewidth}{\includegraphics{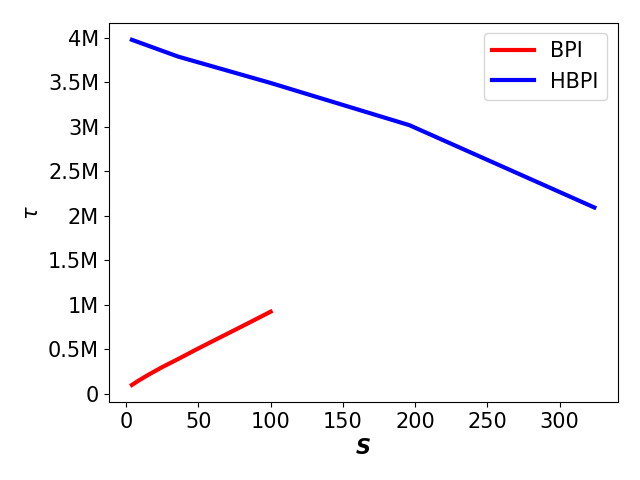}}%
        \label{fig:ablationsS}%
    }\hfill
    \subfloat[Stopping time of $\bm HH$ for fixed $\bm S$.]{
        \resizebox{0.485\linewidth}{0.45\linewidth}{\input{fig/horizon_comparison_5x3}}%
        \label{fig:ablationsH}%
    }
    \caption{Ablations varying $\bm S$ and $\bm HH$ while maintaining everything else constant.}
    \label{fig:ablations}
\end{figure}

\paragraph{Ablations.}

To test the exact dependency of the stopping time on $\bm S$ and $\bm H$, we perform several ablations.
Figure~\ref{fig:ablationsS} shows the stopping time of BPI-UCRL and HBPI-UCRL as functions of $\bm S$ when $\bm H$ is fixed.
As expected, the stopping time of BPI-UCRL is linear in $\bm S$ even when $\bm H$ is fixed.
However, the stopping time of HBPI-UCRL actually decreases as a function of $\bm S$.
We believe that the stopping time of HBPI-UCRL is adversely affected by setting the horizon much too large (as is the case for small values of $\bm S$ in the figure), and we have two possible explanations for this.
The first is that our implementation of HBPI-UCRL interrupts an episode or a subproblem when a terminal state is reached, preventing the algorithm from collecting as much data as the horizon allows.
The second is that the policy update interval $\Psi$ may cause a larger approximation error when the horizon is larger.

Conversely, Figure~\ref{fig:ablationsH} shows the stopping time of BPI-UCRL and HBPI-UCRL as functions of $\bm H$ and $H$ when $\bm S$ is fixed, in the setting of $5\times 5$ rooms of size $3\times 3$.
The stopping time of BPI-UCRL increases monotonically in $\bm HH$, but the stopping time of HBPI-UCRL depends on the values of both $\bm H$ and $H$.
Concretely, the stopping time of HBPI-UCRL increases as a function of $\bm H$ but is initially constant as a function of $H$.
However, when increasing the subproblem horizon $H$, we observe the same effect as in Figure~\ref{fig:ablationsS}, that the stopping time of HBPI-UCRL is adversely affected by setting the horizon much too large.
We conjecture that the stopping time of HBPI-UCRL depends on the {\em effective} subproblem horizon $H$, i.e.~the number of time steps necessary to solve each subproblem in practice.

\begin{figure}[h]
    \centering
    \includegraphics[height=6cm]{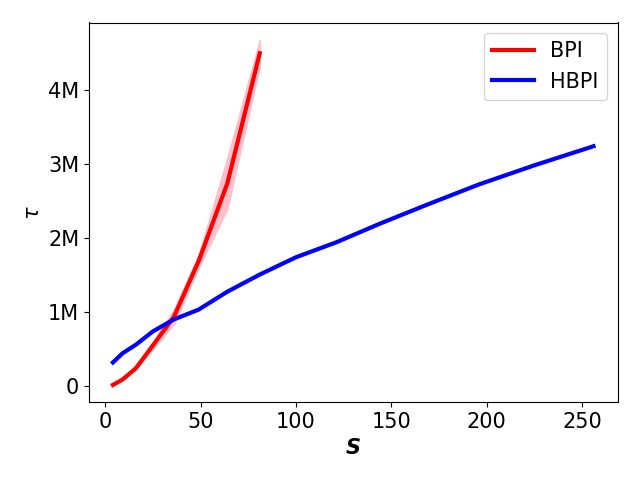}
    \caption{Stopping time $\tau$ of BPI-UCRL and HBPI-UCRL as a function of $\bm S$ in stochastic domains.}
    \label{fig:stochastic_exp}
\end{figure}

\paragraph{Results for Stochastic Domains.}
Even though Assumption~\ref{ass:rew_alt} is not guaranteed to hold for stochastic domains in which the probability of reaching a terminal state is less than $1$, we carry out experiments in a version of the domain in Figure~\ref{fig:ex} with stochastic actions. Concretely, with probability $0.1$ an action has a random effect among all neighboring states. We use the same setting as before with three seeds and varying the number of high-level states $\bm S$. Figure~\ref{fig:stochastic_exp} shows the stopping time of BPI-UCRL and HBPI-UCRL in this setting. We can observe that the stopping time of BPI-UCRL is no longer linear, while the stopping time of HBPI-UCRL still has a sublinear tendency.

\begin{figure}[h]
    \centering
    \input{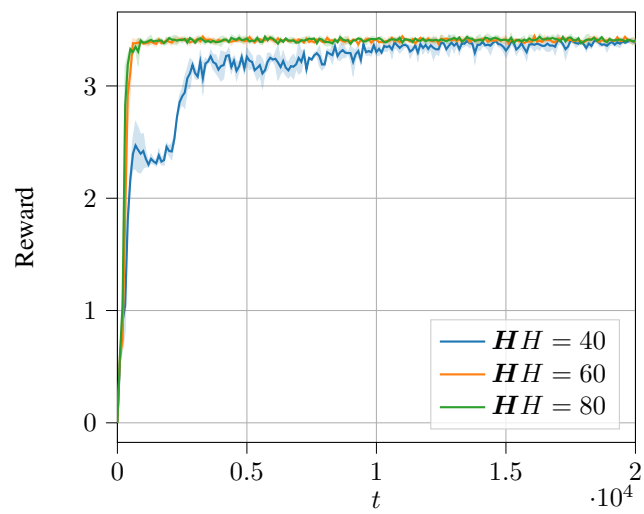}
    \caption{Results for HBPI-UCRL with treasures.}
    \label{fig:treasures}
\end{figure}

\paragraph{Results with Treasures.}
We also carry out experiments with HBPI-UCRL when treasures are present. 
Concretely, we consider a version of the example in Figure~\ref{fig:ex} with $2\times 2$ rooms of size $3\times 3$ and at most one treasure in each room.
This example has almost $8,000$ states and over $300$ million state-action-state triplets.
Since BPI-UCRL has to maintain a count $N^t(s,a,s')$ for each state-action-state triplet, the algorithm suffers extremely slow running time, which prevents it from solving the problem in practice.
In contrast, Figure~\ref{fig:treasures} shows the reward collected by HBPI-UCRL for a fixed subproblem horizon $H=10$ and varying SMDP horizons $\bm H=4,6,8$.



\end{document}